\documentclass{article}

\PassOptionsToPackage{authoryear, round}{natbib}

\usepackage{graphicx}
\usepackage[dvipsnames]{xcolor}

\usepackage{amsmath}
\usepackage{amssymb}
\usepackage{mathtools}
\usepackage{amsthm}

\usepackage{hyperref}                     
\usepackage[capitalise,nameinlink]{cleveref} 

\usepackage[most]{tcolorbox}

\theoremstyle{plain}
\newtheorem{theorem}{Theorem}[section]
\newtheorem{proposition}[theorem]{Proposition}
\newtheorem{lemma}[theorem]{Lemma}

\theoremstyle{definition}

\theoremstyle{remark}

\tcbset{subtlebox/.style={
    colback=gray!8, colframe=gray!8,
    arc=0mm, boxrule=0pt,
    top=4pt, bottom=4pt, left=4pt, right=4pt,
    before skip=10pt, after skip=10pt,
    breakable
}}
\tcolorboxenvironment{theorem}{subtlebox}
\tcolorboxenvironment{proposition}{subtlebox}
\tcolorboxenvironment{lemma}{subtlebox}

\definecolor{DarkNavy}{HTML}{183B56}
\hypersetup{
    pdftitle={Outlier-Robust Multi-Output Gaussian Processes},
    colorlinks=true,
    allcolors=DarkNavy
}

\usepackage{bm}
\usepackage{enumitem}
\usepackage{multirow}
\usepackage{pifont}
\newcommand{\cmark}{\ding{51}}  
\newcommand{\xmark}{\ding{55}}  

\usepackage[main, final]{neurips_2026}

\usepackage[utf8]{inputenc} 
\usepackage[T1]{fontenc}    
\usepackage{subcaption} 
\usepackage{url}            
\usepackage{booktabs}       
\usepackage{amsfonts}       
\usepackage{nicefrac}       
\usepackage{microtype}      
\usepackage{xcolor}         

\usepackage{enumitem}

\usepackage{wrapfig}

\title{Outlier-Robust Multi-Output Gaussian Processes}

\author{%
  Joshua Rooijakkers \\
  Imperial College London
  \And
  Leiv R{\o}nneberg \\
  University of Oslo \\
  \And
  Fran\c{c}ois-Xavier Briol \\
  University College London \\
  \AND
  Jeremias Knoblauch \\
  University College London \\
  \And
  Matias Altamirano \\
  University of Oxford \\
}

\begin{document}

\maketitle

\begin{abstract}
  Multi-output Gaussian process (MOGP) regression allows modelling dependencies among multiple correlated response variables. Similarly to standard Gaussian processes (GPs), MOGPs are sensitive to model misspecification and outliers, which can distort predictions within individual outputs. 
  In the multi-output case, this situation can be further exacerbated as anomalous observations in one response can adversely influence predictions for other correlated outputs. To handle this situation, we propose R-MOGP---an MOGP that is provably robust to outliers in any individual response. Our approach preserves conjugacy, achieving robustness without sacrificing computational efficiency, and integrates naturally with sparse variational inference for scalability. Empirically, we show that R-MOGP achieves performance comparable to existing state-of-the-art robust MOGPs at a fraction of their computational costs. 
\end{abstract}

\section{Introduction} \label{sec:introduction}

Gaussian processes (GPs; \citealp{williams2006gaussian}) provide flexible Bayesian models for latent functions which are widely used in statistics and machine learning.
While GPs are usually formulated as having a single response variable (and hence a single output), many real-world applications require simultaneous predictions of multiple interacting response variables, leading to the natural extension to \textit{multi-output Gaussian processes} (MOGPs; \citealp{bonilla2007multi, alvarez2012kernels}). By modelling these correlated variables jointly, MOGPs are well-suited for a wide variety of tasks, such as multi-task regression \citep{alvarez2012kernels}, causal inference \citep{ alaa2017bayesian,dimitriou2024data}, spatio-temporal modelling \citep{genton2015cross}, numerical integration \citep{Xi2018MultiOutput,Sun2021} and multi-task  optimisation \citep{swersky2013multi, maddox2021bayesian}.
Similarly to standard GPs, MOGPs offer an exact Bayesian closed-form posterior under the assumption of a Gaussian likelihood \citep{alvarez2012kernels}. 
This assumption is computationally convenient, but also makes MOGPs vulnerable to outliers, model misspecification,
\begin{wrapfigure}[12]{r}{0.5\textwidth} 
    \vspace{-0.31cm}
  \centering
  \includegraphics[width=\linewidth]{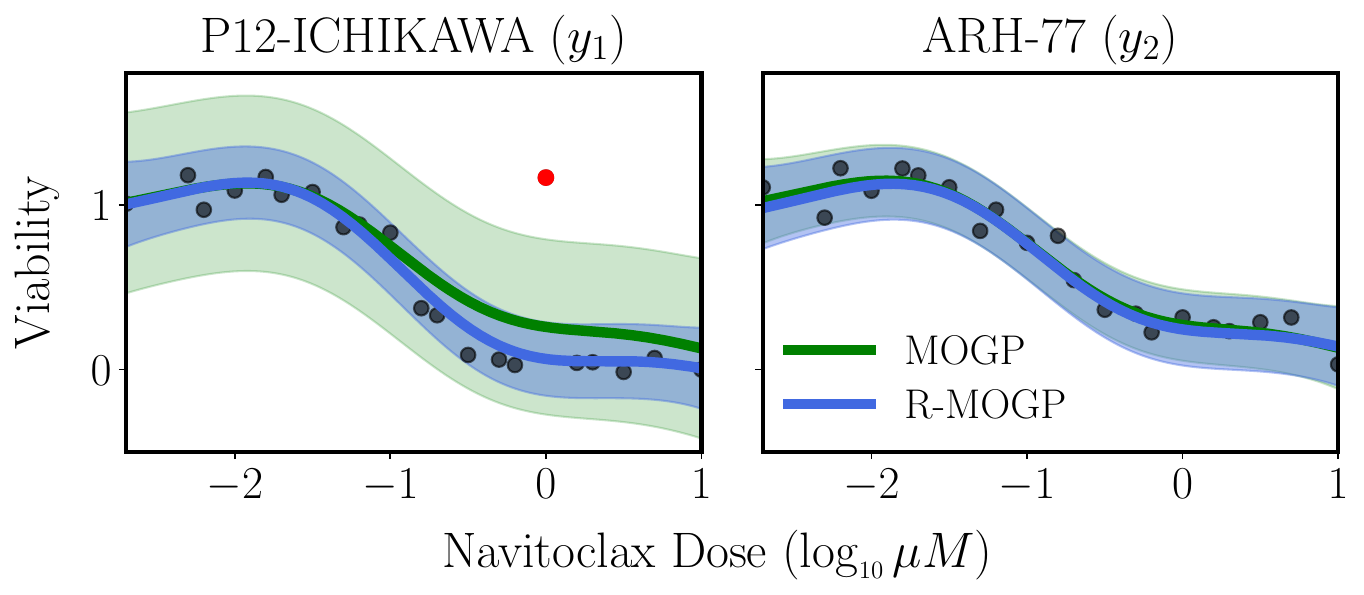}
  \caption{{Navitoclax dose-response with outlier.} We show the viability of two cancer cell lines exposed to varying doses of the drug Navitoclax.}
  \label{fig:cisplatin_MOGP}
  \vspace{-0.4cm}
\end{wrapfigure}
or non-Gaussian noise \citep{Wynne2020}. 
As a result, predictions and uncertainty estimates may become unreliable. In the multi-output setting, however, this issue is even more pronounced: contaminated observations in one output can not only distort predictions for that output but also propagate errors to correlated outputs. 
This is a common problem in various applications, including in cancer research, where MOGPs are used to model how the viability of different correlated cancer cell samples respond to varying drug doses \citep{ronneberg2023dose}. Such experiments often produce extremely noisy measurements due to the heterogeneity of cell growth and technical issues with the experimental assays, such as evaporation or contamination---which in turn distorts MOGP predictions. 
Figure \ref{fig:cisplatin_MOGP} illustrates this for the drug Navitoclax, where such an outlier (marked by the red dot) causes the MOGP predictions to lie above the observed dose-response values, with increased uncertainty.
If unaccounted for, this can bias summary statistics commonly used to link genetic markers to drug efficacy, and ultimately lead to incorrect treatments and poor health outcomes.

To address these challenges, prior work has replaced Gaussian noise with heavier-tailed alternatives such as the Student-$t$, Laplace, and mixture distributions \citep{jylanki2011robust, huang2024robust, kuss2006gaussian, lu2023robust}. Alternatively, \citet{yu2007robust} introduce a Student-$t$ process prior over the latent functions. While these approaches offer greater robustness to outliers, they also break conjugacy, and therefore require approximate inference and additional non-convex optimisation routines with overall computational costs that far exceed their conjugate counterparts. 
This is already problematic for standard GPs, but even more consequential for MOGPs: even conjugate inference in MOGPs scales cubically in both the number of samples \textit{and} the number of outputs. 

In this paper, we introduce R-MOGP, which is \textit{fully conjugate} for regression, avoiding the need for additional computationally expensive optimisation techniques. Unlike standard conjugate MOGPs, R-MOGP is also robust: we theoretically show that the method is provably robust against unbounded corruptions in any individual response, both within a single output and across correlated outputs. We further propose novel parameter-optimisation procedures that improve robustness by borrowing statistical strength across outputs. R-MOGPs are flexible, and accommodate \textit{heterotopic data}---a common scenario in multi-output problems where we do not observe data for some of the output channels. They also naturally integrate with the sparse variational framework to further improve scalability.
Empirically, we demonstrate that R-MOGPs achieve performance comparable to existing robust MOGPs at a fraction of their computational cost.
Our contributions are summarised as follows:
\begin{enumerate}[leftmargin=8mm]
    \item We introduce a generalised, robust form of MOGPs and derive the associated closed forms for the posterior and predictive distributions.
    \item We theoretically show that, unlike MOGPs, R-MOGP is robust to arbitrarily extreme outliers.
    \item We show that R-MOGP is competitive with existing state-of-the-art robust MOGPs while requiring substantially less computation.
\end{enumerate}

\section{Background} \label{sec:background}

\paragraph{GP Regression}

Let $\mathcal{D} = \{X, \mathbf{y}\}$ be a dataset of $N$ observations, where $X = [\mathbf{x}_1, \dots, \mathbf{x}_{\scriptscriptstyle N}]^\top$ and $\mathbf{y} = [y_1, \dots, y_N]^\top$, with $\mathbf{x}_i \in \mathcal{X} \subseteq \mathbb{R}^d$ and $y_i \in \mathbb{R}$ for all $i \in \{1,\ldots,N\}$. GP regression \citep{Stein1990,williams2006gaussian,Garnett2021} models these observations as noisy evaluations of a latent function $f: \mathcal{X} \rightarrow \mathbb{R}$ with additive noise $\varepsilon_i \in \mathbb{R}$ via $y_i = f(\mathbf{x}_i) + \varepsilon_i$ for all $i \in \{1,\ldots,N\}$.
Our prior beliefs about the unknown function $f$ are expressed through a GP prior $f \sim \mathcal{GP}(m, \kappa)$ for $m: \mathcal{X} \rightarrow \mathbb{R}$ and $\kappa: \mathcal{X} \times \mathcal{X} \rightarrow \mathbb{R}$ denoting \textit{scalar-valued} mean and covariance (or kernel) functions, respectively. 
Here, $m$ and $\kappa$ capture our prior beliefs about properties such as smoothness and sparsity of $f$, and typically depend on hyperparameters $\bm{\theta} \in \Theta$ which are suppressed from notation for legibility. 
Under the GP prior, the corresponding latent function vector $\mathbf{f} = [f(\mathbf{x}_1), \dots, f(\mathbf{x}_N)]^\top$ follows an $N$-dimensional Gaussian distribution with density $p(\mathbf{f}|X) = \mathcal{N}(\mathbf{f}; \mathbf{m}, K)$, where $\mathbf{m} = [m(\mathbf{x}_1), \dots, m(\mathbf{x}_{\scriptscriptstyle N})]^\top$ and $K_{ij} = \kappa(\mathbf{x}_i, \mathbf{x}_j)$.

If $\bm{\varepsilon} = [\varepsilon_1, \dots, \varepsilon_N]^\top \in \mathbb{R}^N$ is assumed to be Gaussian so that $\bm{\varepsilon} \sim \mathcal{N}(0, \sigma^2 I_N)$, where $I_N$ is the $N$-dimensional identity matrix, the posterior predictive density for $f_\star = f(\mathbf{x}_\star)$ at $\mathbf{x}_\star \in \mathcal{X}$ has the closed form $p(f_\star|\mathbf{x}_\star, \mathbf{y}, X) = \mathcal{N}(f_\star; \mu_\star^{\scriptscriptstyle \text{GP}}, \Sigma^{\scriptscriptstyle \text{GP}}_\star)$, where
\begin{equation}
\begin{aligned}
    \mu_\star^{\scriptscriptstyle \text{GP}} &= m_\star + k_\star^\top (K + \sigma^2\textcolor{Green}{I_{\scriptscriptstyle N}})^{-1}(\mathbf{y} - \textcolor{Green}{\mathbf{m}}),\\
    \Sigma^{\scriptscriptstyle \text{GP}}_\star &= k_{\star\star} - k_\star^\top (K + \sigma^2\textcolor{Green}{I_{\scriptscriptstyle N}})^{-1}k_\star,
\end{aligned}
\label{eq:gp_predictive}
\end{equation}
where $k_\star = [\kappa(\mathbf{x}_\star, \mathbf{x}_1), \dots, \kappa(\mathbf{x}_\star, \mathbf{x}_N)]^\top$, $k_{\star\star} = \kappa(\mathbf{x}_\star, \mathbf{x}_\star)$ and $m_\star = m(\mathbf{x}_\star)$. The green colour is later used to highlight  differences between GPs and RCGPs.

\paragraph{Robust and Conjugate GPs}

The closed-form predictive in \eqref{eq:gp_predictive} relies on the assumption that the observation noise $\bm{\varepsilon}$ is Gaussian.
If this assumption is accurate, a Bayesian approach constitutes a way of learning from data which is in some sense optimal \citep{zellner1988optimal}.
Critically however, this optimality breaks down in the presence of model misspecification, such as in the case of heavy-tailed noise or outliers when working with GPs. 
To address this issue, recent work has explored generalised Bayesian (GB) inference  \citep{grunwald2012safe, bissiri2016general, ghosh2016robust, miller2019robust, knoblauch2022optimization, dellaporta2022robust,fong2023martingale, wild2023rigorous, mclatchie2025predictively, mclatchie_predictive_2025, shenprediction2025}.
The generalised Bayesian posterior is given by
\begin{align}
    p_{\scriptscriptstyle\text{GB}}(\mathbf{f}|\mathbf{y}, X) \propto p(\mathbf{f}|X) \exp (- N \mathcal{L}(\mathbf{f}, \mathbf{y}, X)),\label{eq:gb_update}
\end{align}
where $\mathcal{L}$ denotes an empirical loss function measuring the disagreement between the statistical model and the observed data, and different loss functions are chosen to provide robustness or computational efficiency (see e.g. \citealp{knoblauch2018doubly, futami2018variational, boustati2020generalised, matsubara2022robust, matsubara2024generalized, altamirano2023robust, duran2024outlier}). 

Recent work \citep{altamirano2024robustGP, sinaga2024computation, laplante2025robust, ezzerg2025robust} has applied generalised Bayesian inference to GPs by using a \textit{weighted Fisher divergence} as the loss function \citep{hyvarinen2005estimation, barp2019minimum, xu2022generalized}, given by
\begin{equation}
\begin{aligned}
    \mathcal{L}_\mathbf{w}(\mathbf{f}, \mathbf{y}, X) = \mathbb{E}_{\tilde{\mathbf{x}} \sim p_{0, x}}\bigl[\mathbb{E}_{\tilde{{y}}\sim p_0(\cdot|\tilde{\mathbf{x}})}\big[\|w(\tilde{\mathbf{x}}, \tilde{{y}}) (s_{\text{model}}(\tilde{\mathbf{x}}, \tilde{{y}}) - s_{\text{data}}(\tilde{\mathbf{x}}, \tilde{{y}}))\|^2_2 \big]\bigr]
\end{aligned}
\label{eq:DSM_divergence}
\end{equation}
where $s_{\text{model}}$ and $s_\text{data}$ are the scores of the respective posited and true conditional likelihoods $p_\theta(y|\mathbf{x})$ and $p_0(y|\mathbf{x})$, i.e. $s_{\text{model}} = \nabla_y \log p_\theta(y|\mathbf{x})$. Importantly, $w: \mathcal{X}\times \mathbb{R} \rightarrow \mathbb{R}\backslash\{0\}$ is a continuously differentiable weight function that downweights extreme observations. When positing a Gaussian model, \citet{altamirano2024robustGP} show that the GB posterior with the above loss function is conjugate, and its closed-form predictive is given by $p_w(f_\star|x_\star, \mathbf{x}, \mathbf{y}) = \mathcal{N}(f_\star; \mu_\star^{\scriptscriptstyle \text{RCGP}}, \Sigma^{\scriptscriptstyle \text{RCGP}}_\star)$, where
\begin{equation}
\begin{aligned}
    \mu_\star^{\scriptscriptstyle \text{RCGP}} &= m_\star + k_\star^\top (K + \sigma^2\textcolor{RoyalBlue}{\mathbf{J}_{\mathbf{w}}})^{-1}(\mathbf{y} - \textcolor{RoyalBlue}{\mathbf{m}_\mathbf{w}}),\\
    \Sigma^{\scriptscriptstyle \text{RCGP}}_\star &= k_{\star\star} - k_\star^\top (K + \sigma^2\textcolor{RoyalBlue}{\mathbf{J}_{\mathbf{w}}})^{-1}k_\star,
\end{aligned}
\label{eq:rcgp_predictive}
\end{equation}
where the matrix $\textcolor{RoyalBlue}{\mathbf{J}_{\mathbf{w}}} =\text{diag}(\sigma^2 \mathbf{w}^{-2}/2) \in \mathbb{R}^{N \times N}$, the vector $\textcolor{RoyalBlue}{\mathbf{m}_\mathbf{w}} = \mathbf{m} + \sigma^2 \nabla_y \log (\mathbf{w}^2) \in \mathbb{R}^N$, and $\mathbf{w} = [w(\mathbf{x}_1, y_1), \dots, w(\mathbf{x}_N, y_N)]^\top \in \mathbb{R}^N$. 
While this closely resembles the results of standard GPs in \eqref{eq:gp_predictive}, there are two key differences highlighted in \textcolor{RoyalBlue}{\textbf{blue}}.
Note further that the above is a strict generalisation of GPs: the specific choice $w(\mathbf{x}, y) = \sigma/\sqrt{2}$ recovers standard GPs.

\paragraph{Multi-Output GPs}

In many applications, we are not interested in modelling a scalar-valued function, but a vector-valued latent function $\mathbf{f}: \mathcal{X} \rightarrow \mathbb{R}^T$, where each component $f_t: \mathcal{X} \rightarrow \mathbb{R}$ corresponds to one of $T$ outputs such that $\mathbf{f} = [f_1, \dots, f_T]^\top$. MOGPs extend standard GPs to this setting by also capturing the correlations across output dimensions, also known as co-kriging \citep{alvarez2012kernels}. From here on out, we focus exclusively on MOGPs, and all notation and definitions should be interpreted in this context. Let $Y = [\mathbf{y}_1, \dots, \mathbf{y}_{\scriptscriptstyle N}]^\top$ and $\mathcal{E} = [\bm{\varepsilon}_{1}, \dots, \bm{\varepsilon}_{\scriptscriptstyle N}]^\top$ denote $N \times T$ matrices where $\mathbf{y}_i = [y_{i, 1}, \dots, y_{i, T}]^\top$ and $\bm{\varepsilon}_i = [\varepsilon_{i, 1}, \dots, \varepsilon_{i, T}]^\top$.

We place a GP prior on the latent function $\mathbf{f} \sim \mathcal{GP}(\mathbf{m}, \mathcal{K})$, where, in contrast to scalar GPs, the mean $\mathbf{m}: \mathcal{X} \rightarrow \mathbb{R}^{T}$ is now \textit{vector-valued}, and the kernel function $\mathcal{K}: \mathcal{X} \times \mathcal{X} \rightarrow \mathbb{R}^{T\times T}$ is \textit{matrix-valued} \citep{Micchelli2004}. In line with this, the $T$ elements in $\mathbf{m} = [m_1, \dots, m_T]^\top$ denote the prior means for the $T$ outputs.
Additionally, the kernel $\mathcal{K}(\mathbf{x}, \mathbf{x}')$ is a positive semidefinite matrix-valued function, whose entries $[\mathcal{K}(\mathbf{x}, \mathbf{x}')]_{t,t'}$ encode the covariance between the outputs $f_t(\mathbf{x})$ and $f_{t'}(\mathbf{x'})$ and capture both the smoothness within the same output with respect to inputs $\mathbf{x}$ and $\mathbf{x}'$ and the correlations across outputs $t$ and $t'$. Common approaches include the intrinsic coregionalisation model (ICM; \citealp{goovaerts1997geostatistics}), the linear model of coregionalisation \citep{alvarez2013linear}, process convolutions \citep{etde_5214736} and spectral mixtures \citep{parra2017spectral}. For example, in the ICM, the kernel is assumed to be separable, such that it can be written as the product $[\mathcal{K}(\mathbf{x}, \mathbf{x}')]_{t,t'} = B_{t,t'}\kappa(\mathbf{x}, \mathbf{x}')$ with a coregionalisation matrix $B \in \mathbb{R}^{T \times T}$, which represents the covariance matrix of the latent functions across outputs, and a scalar-valued kernel $\kappa(\mathbf{x}, \mathbf{x}')$ representing properties of individual outputs.

When assuming that observations follow a Gaussian distribution and are independent across outputs such that $\text{vec}(\mathcal{E}) \sim \mathcal{N}(0, \mathbf{\Sigma})$, where $\mathbf{\Sigma} = \Sigma \otimes I_N \in \mathbb{R}^{NT \times NT}$ with $\Sigma = \mathrm{diag}(\sigma_1^2, \dots, \sigma_T^2) \in \mathbb{R}^{T \times T}$ denoting the variances in each output, the MOGP posterior and predictive are available in closed form. Here, $\textup{vec}(\cdot)$ denotes vectorisation of a matrix by stacking its columns, and $\otimes$ denotes the Kronecker product. We define $\mathbf{y}_{\mathrm{vec}} = \mathrm{vec}(Y) \in \mathbb{R}^{NT}$ and $\mathbf{m}_{\mathrm{vec}} = \mathrm{vec}(M)  \in \mathbb{R}^{NT}$, where $M = [\mathbf{m}_1, \dots, \mathbf{m}_{\scriptscriptstyle{N}}]^\top$ with $\mathbf{m}_i = \mathbf{m}(\mathbf{x}_i)$. In addition, the $NT \times NT$ block-structured kernel matrix is $\mathbf{K} = \left[\mathcal{K}_{tt'}(X,X)\right]_{t,t' = 1}^T$, where each $\mathcal{K}_{tt'}(X, X)$ is the Gram matrix between inputs for outputs $t$ and $t'$. Finally, let $F = [\mathbf{f}_1, \dots, \mathbf{f}_{\scriptscriptstyle{N}}]^\top$ be the $N \times T$ matrix of latent function evaluations, with $\mathbf{f}_i = \mathbf{f}(\mathbf{x}_i) = [f_1(\mathbf{x}_i), \dots, f_T(\mathbf{x}_i)]^\top$ and $\mathbf{f}_{\mathrm{vec}} = \mathrm{vec}(F)$. Then, the closed-form MOGP predictive density for $f_\star = [f_1(\mathbf{x}_\star), \dots, f_T(\mathbf{x}_\star)]^\top$ is given by $p(f_\star | \mathbf{x}_\star, Y, X) = \mathcal{N}(f_\star ; \bm{\mu}_\star^{\scriptscriptstyle \text{MOGP}}, \bm{\Sigma}_\star^{\scriptscriptstyle \text{MOGP}})$, where
\begin{equation}
\begin{aligned}
    \bm{\mu}_\star^{\scriptscriptstyle \text{MOGP}} &= \mathbf{m}_\star + \mathbf{k}_\star^\top(\mathbf{K} + \mathbf{\Sigma} \,\textcolor{Green}{I_{\scriptscriptstyle NT}})^{-1}(\mathbf{y}_{\text{vec}} - \textcolor{Green}{\mathbf{m}_{\text{vec}}}),\\
    \bm{\Sigma}_\star^{\scriptscriptstyle \text{MOGP}} &= \mathbf{k}_{\star\star} - \mathbf{k}_\star^\top(\mathbf{K} + \mathbf{\Sigma} \,\textcolor{Green}{I_{\scriptscriptstyle NT}})^{-1}\mathbf{k}_\star,
\end{aligned}
\label{eq:mogp_predictive}
\end{equation}
where $\mathbf{k}_\star = [\mathcal{K}(\mathbf{x}_\star, \mathbf{x}_1), \dots, \mathcal{K}(\mathbf{x}_\star, \mathbf{x}_{\scriptscriptstyle N})]^\top \in \mathbb{R}^{NT \times T}$, $\mathbf{k}_{\star\star} = \mathcal{K}(\mathbf{x}_\star, \mathbf{x}_\star) \in \mathbb{R}^{T \times T}$ and $\mathbf{m}_\star = \mathbf{m}(\mathbf{x}_\star) \in \mathbb{R}^T$. We note that it is straightforward to apply the expressions above in the case of heterotopic data---where only $T_i < T$ outputs are observed for input $\mathbf{x}_i$, by simply restricting $\mathbf{y}_\textup{vec}, \mathbf{m}_\textup{vec}$ and the corresponding rows and columns of $\mathbf{K}$ and $\mathbf{\Sigma}$ to the observed entries.
In general, computing the MOGP predictive in  \eqref{eq:mogp_predictive} requires inverting an $NT \times NT$ matrix, leading to a worst-case computational complexity of $\mathcal{O}(N^3T^3)$, which quickly becomes expensive as either the number of samples $N$ or number of outputs $T$ grows. 

Despite their large computational cost, MOGPs are very useful, particularly for predicting multiple correlated response variables.
In this setting, they share information across outputs---which independent scalar GPs cannot do. 
However, they also assume Gaussian errors. 
While this assumption enables conjugacy and closed-form posteriors, it also constitutes an even more pronounced problem than in the scalar case: outliers can distort not only the predictions corresponding to their {own} output, but also those of \textit{other} correlated outputs.
While several robust multi-output GP algorithms have been proposed \citep{jylanki2011robust, yu2007robust}, they do not result in exact conjugate updates. As a result, they further inflate the already substantive computational budget of MOGPs.

\section{Methodology} \label{sec:methodology}

We again consider the weighted Fisher divergence, but instead of the scalar observations in \eqref{eq:DSM_divergence}, we now focus on the case of vector-valued observations, as also studied in \citet{barp2019minimum, altamirano2023robust}. This divergence measures the discrepancy between vector-valued scores using a weight function $W:\mathcal{X} \times \mathbb{R}^T \rightarrow \mathbb{R}^{T \times T}$ returning sufficiently regular positive semi-definite matrices, and is given by
\begin{equation}
\begin{aligned}
\mathcal{L}_\mathbf{W}(F, Y, X) = \mathbb{E}_{\tilde{\mathbf{x}}\sim p_{0,x}}\bigl[\mathbb{E}_{\tilde{\mathbf{y}}\sim p_0{(\cdot | \tilde{\mathbf{x}})}}[\|W(\tilde{\mathbf{x}}, \tilde{\mathbf{y}})^\top& \big( s_{\text{model}}(\tilde{\mathbf{x}}, \tilde{\mathbf{y}}) - s_{\text{data}}(\tilde{\mathbf{x}}, \tilde{\mathbf{y}})\big)\|^2_2]\bigr],
    \label{eq:multivariate_DSM}
\end{aligned}
\end{equation}
where $\tilde{\mathbf{x}} \in \mathbb{R}^d$ and $\tilde{\mathbf{y}} \in \mathbb{R}^T$ are random variables and $p_{0,x}$ and $p_0$ now represent corresponding densities. Additionally, $s_{\text{model}}$  and $s_{\text{data}}$ denote $T$-sized score vectors of the model and the true data, where $s_{\text{data}}(\mathbf{x}, \mathbf{y}) = \nabla_{\mathbf{y}}\log p_{0}(\mathbf{y}|\mathbf{x})$ and $s_{\text{model}}(\mathbf{x}, \mathbf{y}) = \nabla_{\mathbf{y}}\log p_{\theta}(\mathbf{y}|\mathbf{x})$ for true and posited model $p_0$ and $p_{\theta}$, respectively. Importantly, $W(\mathbf{x}, \mathbf{y})$ is now a $T \times T$ matrix that depends on observations across outputs. Mild assumptions on $W$, such as pointwise invertibility and continuous differentiability, are sufficient for conjugacy \citep{liu2022estimating}, but for practical purposes we choose a diagonal form such that $W_i := W(\mathbf{x}_i, \mathbf{y}_i) = \text{diag}(w_1(\mathbf{x}_i, \mathbf{y}_i), \dots, w_T(\mathbf{x}_i, \mathbf{y}_i))$, where $w_t:\mathcal{X} \times \mathbb{R}^T \rightarrow \mathbb{R}\backslash\{0\}$ is a scalar-valued continuously-differentiable weight function as used in RCGPs. This structure makes our weights interpretable---each observation $y_{i,t}$ is weighted by $w_{i,t} := w_t(\mathbf{x}_i, \mathbf{y}_i)$---and  facilitates robust hyperparameter optimisation. 
Using the above to construct a generalised posterior as in \eqref{eq:gb_update} then results in the  R-MOGP predictive.
\begin{proposition}\label{prop:MORCGP}
    Let $\textup{vec}(\mathcal{E}) \sim \mathcal{N}(0, \mathbf{\Sigma})$ where $\mathbf{\Sigma} = \Sigma \otimes I_N$ for a $T$-dimensional diagonal $\Sigma$, and $W_i = \textup{diag}(w_1(\mathbf{x}_i, \mathbf{y}_i), \dots, w_T(\mathbf{x}_i, \mathbf{y}_i))$, where $y_{i,t} \mapsto w_t(\mathbf{x}_i, \mathbf{y}_i)$ is continuously differentiable. Then, the R-MOGP predictive density is given by $p_{\scriptscriptstyle W}(f_\star | \mathbf{x}_\star, Y, X) = \mathcal{N}(f_\star ; \bm{\mu}_\star^{\scriptscriptstyle \textup{RMOGP}}, \bm{\Sigma}_\star^{\scriptscriptstyle \textup{RMOGP}})$, where
    \begin{align*}
        \bm{\mu}_\star^{\scriptscriptstyle \textup{RMOGP}} &= \mathbf{m}_\star + \mathbf{k}_\star^\top(\mathbf{K} + \mathbf{\Sigma} \,\textcolor{RoyalBlue}{\mathbf{J}_{\scriptscriptstyle \mathbf{W}}})^{-1}(\mathbf{y}_{\textup{vec}} - \textcolor{RoyalBlue}{\mathbf{m}_{\mathbf{W},\textup{vec}}}),\\
        \bm{\Sigma}_\star^{\scriptscriptstyle \textup{RMOGP}} &= \mathbf{k}_{\star\star} - \mathbf{k}_\star^\top(\mathbf{K} + \mathbf{\Sigma} \,\textcolor{RoyalBlue}{\mathbf{J}_{\scriptscriptstyle \mathbf{W}}})^{-1}\mathbf{k}_\star,
    \end{align*}
    where $\textcolor{RoyalBlue}{\mathbf{J}_{\scriptscriptstyle \mathbf{W}}} = \frac{1}{2} \mathbf{\Sigma} \mathbf{W}^{-2} \in \mathbb{R}^{NT \times NT}$, $\textcolor{RoyalBlue}{\mathbf{m}_{\mathbf{W},\textup{vec}}} = \mathbf{m}_\textup{vec} + \mathbf{\Sigma} ( \nabla_\mathbf{y} \cdot \log(\mathbf{W}^2)) \in \mathbb{R}^{NT}$,  $\mathbf{W} = \textup{diag}(\mathbf{W}_1, \dots, \mathbf{W}_{T})\in \mathbb{R}^{NT \times NT}$, $\mathbf{W}_t = \textup{diag}(w_t(\mathbf{x}_1, \mathbf{y}_1), \dots, w_t(\mathbf{x}_{\scriptscriptstyle N} , \mathbf{y}_{\scriptscriptstyle N})) \in \mathbb{R}^{N \times N}$, and $\nabla_{\mathbf{y}} \cdot \log (\mathbf{W}^2) = [\nabla_{\mathbf{y}} \cdot \log(\mathbf{W}_1^2)^\top, \dots, \nabla_{\mathbf{y}} \cdot \log(\mathbf{W}_{\scriptscriptstyle T}^2)^\top]^\top \in \mathbb{R}^{NT}$.
\end{proposition}

The proof is in Appendix~\ref{app:proof_MORCGP}. In the above result, we have abused notation, and written $\log (\mathbf{W})$ to denote a diagonal matrix obtained from $\mathbf{W}$ by taking entry-wise logarithms along its diagonal. Additionally, $\nabla_{\mathbf{y}} \cdot M$ denotes the divergence of a matrix $M$ (see Appendix~\ref{app:notation}).
The resulting expressions are similar to the standard MOGP case presented in \eqref{eq:mogp_predictive}, with the key terms of difference highlighted in \textcolor{RoyalBlue}{\textbf{blue}} and the dimensionality of each quantity has been specified for clarity. 
Importantly, these terms  introduce weights that allow information to be shared to borrow strength across channels when detecting outliers and reducing their influence on posterior inferences. 
As in the scalar case, which can be recovered by taking $T=1$, $w_t(\mathbf{x}_i, \mathbf{y}_i) = \sigma_t/\sqrt{2}$ recovers the standard MOGP predictive, and using $w_t(\mathbf{x}_i, \mathbf{y}_i)=w_t(\mathbf{x}_i)$ corresponds to an MOGP with heteroscedastic noise \citep{gammelli2022generalized, lee2023multi}. 
Evaluating the closed-form posterior of Proposition~\ref{prop:MORCGP} imposes a computational burden of order $\mathcal{O}(N^3T^3)$---the same cost imposed by standard MOGP.
Unlike standard MOGP, the R-MOGP can be made provably robust through the introduction of the weight term. 

\subsection{Weight Function Choice}
\label{sec:weight-choice}

In this paper, we focus on robustness to sparse outliers: more precisely, we consider the setting where only some entries in the observation vector at a particular input point are outliers.
We restrict attention to this case since it is the type of outlier most commonly observed in real-world data (e.g. \citealp{vstefelova2021robust, su2024robust}). 
Nevertheless our framework is flexible: through a simple change of weighting function, it can be geared towards other types of data pollution---such as outliers that occur jointly across all channels, or outliers in a subset of channels.
While many functional forms of $w_t(\mathbf{x}, \mathbf{y})$ are sufficient to achieve the type of robustness we desire, we adopt the inverse multi-quadratic kernel.
We do so as prior work  has shown that it provides a good balance for performance in both misspecified and well-specified settings \citep{matsubara2022robust, altamirano2024robustGP, laplante2025robust}.
This kernel incorporates a centering function $\gamma_t: \mathcal{X} \times \mathbb{R}^T \rightarrow \mathbb{R}$, a decay rate function $c_t: \mathcal{X} \times \mathbb{R}^T \rightarrow \mathbb{R}$, and a constant $\beta_t>0$, and is given by
\begin{align}\label{eq:weight}
  w_t(\mathbf{x},\mathbf{y})
  &= \beta_t\left(1+\left(\frac{y_t-\gamma_t(\mathbf{x},\mathbf{y})}{c_t(\mathbf{x},\mathbf{y})}\right)^2\right)^{-\frac{1}{2}}.
\end{align}
This weight is widely used in robust methods \citep{barp2019minimum, matsubara2024generalized, liu2024robustness} and downweights outlying observations. 
Specifically, the mapping $y_{t} \mapsto w_t(\mathbf{x}, \mathbf{y})$ is heavy-tailed, symmetric around $\gamma_t(\mathbf{x}, \mathbf{y})$, and decreasing as $|y_{t}-\gamma_t(\mathbf{x}, \mathbf{y})|$ increases, where $c_t(\mathbf{x}, \mathbf{y})$ determines this rate of decrease. 
Finally, the constant $\beta_t$ sets the maximum of the weight and corresponds to the GB learning rate \citep{wu2023comparison}.
Although these hyperparameters can in principle be estimated from data, \citet{altamirano2024robustGP} argue that this is suboptimal and proposed several heuristic choices that performed reasonably well in practice.
For example, they recommend $\gamma_t(\mathbf{x}, \mathbf{y}) = m_t(\mathbf{x})$. While often effective, this makes the method sensitive to prior mean specification---a well-known limitation of RCGPs \citep{ament2024robust, laplante2025robust}.

Unlike in the scalar setting discussed in this prior work, the multi-output framework provides an elegant way for us to reduce this sensitivity by borrowing statistical strength across outputs.
Intuitively speaking, it allows us to use outputs from other channels to make an educated guess about what an outlier looks like.
In technical terms, this means that we define the centering function for output channel $t$ as the conditional expectation given the observations in the other output channels.
For 
$\mathbf{y}_{i,-t} =  [{y}_{i,1},\ldots,{y}_{i,t-1},{y}_{i,t+1},\ldots,{y}_{i,T}]^\top \in \mathbb{R}^{T-1}$, this results in $\gamma_t(\mathbf{x}_i, \mathbf{y}_i) = \mu_{t|-t}(\mathbf{x}_i, \mathbf{y}_i) = \mathbb{E}_{\tilde{y}_t \sim p(\cdot | \mathbf{y}_{i,-t})}[\tilde{y}_t]$. 
Conveniently, the conditional expectation has a closed form which is given by
\begin{align*}
\mu_{t|-t}(\mathbf{x}_i, \mathbf{y}_i) &= m_t(\mathbf{x}_i) + ({C_{-t,t}^{(i)}})^\top ({C_{-t,-t}^{(i)}})^{-1}(\mathbf{y}_{i, -t} - \mathbf{m}_{i, -t}).
\end{align*}
Here ${C_{-t,-t}^{(i)}} = [\mathcal{K}(\mathbf{x}_i, \mathbf{x}_i) + \Sigma]_{-t, -t}$ is the $T-1$ dimensional covariance matrix of $\mathbf{y}_{i,-t}$,  $C_{-t,t}^{(i)} = [\mathcal{K}(\mathbf{x}_i, \mathbf{x}_i)]_{-t, t}$ is the $T-1$ dimensional covariance vector between $y_{i,t}$ and $\mathbf{y}_{i,-t}$, and  $\mathbf{m}_{i, -t}$ collects the mean functions corresponding to $\mathbf{y}_{i,-t}$. 
Under the ICM, these expressions simplify since the covariance matrix of $\mathbf{y}_i$ additively decomposes as $C^{(i)} = B + \Sigma$. The closed-form conditional expectation, however, also holds for more general multi-output kernels; full details are provided in Appendix~\ref{app:conditional_expectation}.

Having fixed $\gamma_t$, we use it to adaptively scale $c_t(\mathbf{x}, \mathbf{y})$ by choosing it as the $(1-\epsilon_t)$-th quantile of $\{|y_{i,t} - \gamma_t(\mathbf{x}_i, \mathbf{y}_i)|\}_{i=1}^{N}$, where  $\epsilon_t\in(0,1)$ denotes the  expected proportion of outliers in output channel $t$. 
Finally, $\beta_t$ allows us to incorporate application-specific information about channel $t$. 
For example, one might increase $\beta_t$ when it is a more reliable channel, as is often the case in multi-fidelity modelling and causal inference \citep{dimitriou2024data}.
As a default, we set $\beta_t = \sigma_t / \sqrt{2}$ to recover the standard MOGP in the absence of outliers.

\subsection{Robustness to Outliers} 
With weights as chosen in the previous section, we analyse the outlier robustness of R-MOGPs \citep{huber2011robust}. 
To this end, for some observation with index $m \in \{1, \dots, N\}$ and output $s \in \{1, \dots, T\}$ in  $\mathcal{D} = \{(\mathbf{x}_i, \mathbf{y}_{i})\}^N_{i=1}$, we replace  $y_{m,s}$ with an arbitrarily large contaminated value $y_{m,s}^c$. 
This results in the contaminated dataset $\mathcal{D}_{m,s}^c = (\mathcal{D} \ \backslash \ \{(\mathbf{x}_m, \mathbf{y}_{m})\}) \cup \{(\mathbf{x}_m, \mathbf{y}_{m,s}^c)\}$, where $\mathbf{y}_{m,s}^c = [y_{m,1}, \dots, y_{m,s}^c, \dots, y_{m,T}]^\top$. 
We quantify the effect of this contamination on the marginal posterior over output $t \in \{1, \dots, T\}$ by measuring the divergence between the contaminated and uncontaminated posteriors. As a function of $|y_{m,s}^c - y_{m,s}|$, this divergence is referred to as the \textit{posterior influence function} (PIF; \citealp{ghosh2016robust, matsubara2022robust,altamirano2023robust}). 
For the special case of Gaussian posteriors, several works advocate using the Kullback–Leibler (KL) divergence \citep{altamirano2024robustGP, duran2024outlier, laplante2025robust}.
With this in mind, we assess robustness to outliers by studying 
\begin{align*}
    \text{PIF}(y_{m,s}^c, f_t, \mathcal{D}) := \text{KL}\big(p_{\scriptscriptstyle W}(f_t\mid\mathcal{D}) \,\big\|\, p_{\scriptscriptstyle W}(f_t\mid \mathcal{D}_{m,s}^c)\big).
\end{align*}
As we show in our next result, the R-MOGP posterior is robust insofar as the effect of the contamination in output channel $s$ on the marginal posterior over channel $t$ is bounded, even for channels $t \neq s$, and even as the size of the contamination diverges so that $|y_{m,s}^c - y_{m,s}| \to \infty$. 
For standard MOGPs, as long as the covariance structure between output channels $s$ and $t$ is not diagonal, this behaviour is \textit{not} replicable:
instead, a single contaminated observation can now increase the PIF arbitrarily in any output  $t$ in the MOGP (see Appendix~\ref{app:proof_robustness}). 
More generally, the MOGP can only provide the desired type of robustness against contaminations in output channel $s$ in exchange for completely cutting the feedback between outputs $s$ and $t$ by imposing a block-diagonal matrix structure in which $s$ and $t$ fall into different blocks.
\begin{proposition} \label{prop:robustness}
    Suppose $\mathbf{f} \sim \mathcal{GP}(\mathbf{m}, \mathcal{K})$, $\textup{vec}(\mathcal{E}) \sim \mathcal{N}(0, \mathbf{\Sigma})$, and let $w_{t}$ be defined as in \eqref{eq:weight}, $W_i = \textup{diag}(w_1(\mathbf{x}_i, \mathbf{y}_i), \dots, w_T(\mathbf{x}_i, \mathbf{y}_i))$. Then, there exists a  constant $C>0$ independent of $y_{m,s}^c$ for which
    \begin{align*}
        \max_{1\leq t \leq T}\sup_{y_{m,s}^c \in \mathbb{R}}
        \textup{PIF}(y_{m,s}^c, f_t, \mathcal{D}) \leq C.
    \end{align*}
\end{proposition}

The proof is in Appendix~\ref{app:proof_robustness}. In words, Proposition \ref{prop:robustness} says that the R-MOGP posterior is robust in \textit{all} channels $t$ against outliers of possibly infinite magnitude in any given channel $s$. 

\subsection{Parameter Optimisation} \label{sec:parameter_optimisation}

For R-MOGP to be practically useful, we need a way to optimise the noise variances $\Sigma$ for each output  and the kernel hyperparameters $\bm{\theta}$. 
In standard MOGPs, these are typically optimised through the marginal likelihood \citep{alvarez2012kernels}, but this is ill-posed for generalised Bayes posteriors  since they do not correspond to conditional probability updates \citep{jewson2022general}. 
In the context of RCGPs, \citet{altamirano2024robustGP} therefore derived a closed form objective based on leave-one-out cross-validation (LOO-CV).
Computing this objective scales as
$\mathcal{O}(N^3T^3)$  (see Appendix~\ref{app:LOO-CV}), which is the same computational cost as optimisation of the marginal likelihood in standard MOGPs \citep{sundararajan1999predictive, altamirano2024robustGP}. 
Unfortunately,  optimising this objective in the presence of outliers is problematic:  it treats all observations equally---whether they are outliers or not---and thus tends to overfit to outliers.
Indeed, RCGPs \citep{altamirano2024robustGP} exhibit poor predictive uncertainty calibration after hyperparameter optimisation.
To this end, we propose to re-use the robustness-inducing weights within the hyperparameter optimisation routine via the weighted LOO-CV (w-LOO-CV) objective
\begin{align*}
    \varphi_w(\Sigma, \bm{\theta}) = \sum^N_{i=1}\sum^{T}_{t=1} \left(\frac{w_{i,t}}{\beta_t}\right)^2 \log p_w(y_{i,t}|Y_{-(i,t)}, \Sigma, \bm{\theta}),
\end{align*}
where $Y_{-(i,t)}$ denotes all elements in the output matrix $Y$ except for the $(i,t)$-th element. We rescale the weights ${w}_{i,t}$  by $\beta_t = \sigma_t/\sqrt{2}$ to confine them to the unit interval, and to prevent weights with high variance from dominating the objective. 
In addition, we square the weights as we found it to lead  to better numerical stability and improved performance. 
Implementation details are provided in Appendix~\ref{app:optimisation_details}, and a further discussion on the proposed objective can be found in Appendix~\ref{app:effect_robustness}.

\subsection{Sparse Variational R-MOGP}

A standard strategy to improve scalability of MOGPs with large $N$ or $T$ is to adopt a sparse variational approximation with $Q$ inducing points, following the inducing point framework of \cite{titsias2009variational, alvarez2010efficient}. We now show that the R-MOGP naturally admits such an approximation. To do so, we exploit the interpretation of MOGPs as linear combinations of $R$ latent GPs \citep{alvarez2012kernels} (see Appendix~\ref{app:sparse-variational}). For each latent process we introduce $Q$ inducing variables, yielding the \textit{sparse variational R-MOGP} (R-SVMOGP).

\begin{proposition}
Consider an ICM with $R$ latent functions and coregionalisation matrix $B = AA^\top$, $A \in \mathbb{R}^{T\times R}$. Let $Z_r \in \mathbb{R}^{Q\times d}$ be inducing locations for latent function $r$ and define inducing variables $U \in \mathbb{R}^{Q\times R}$. The R-SVMOGP predictive distribution is $p_W(f_\star | \mathbf{x}_\star, Y, X) = \mathcal{N}(f_\star; \bm{\mu}_\star^S, \bm{\Sigma}_\star^S)$, where
\begin{align*}
    \bm{\mu}_\star^S &= \mathbf{m}_\star + \mathbf{k}_{U\star}^\top \mathbf{P}_{U}^{-1}\mathbf{K}_{FU}^\top \bm{\nu} \\
    \bm{\Sigma}_\star^S &= \mathbf{k}_{\star\star} - \mathbf{k}_{U\star}^\top \mathbf{K}_{UU}^{-1}\mathbf{k}_{U\star} + \mathbf{k}_{U\star}^\top \mathbf{P}_{U}^{-1}\mathbf{k}_{U\star},
\end{align*}
where $\bm{\nu} = \mathbf{\Sigma}^{-1}\mathbf{J}_{\scriptscriptstyle\mathbf{W}}^{-1} (\mathbf{y}_{\text{vec}} - \mathbf{m}_{{\scriptscriptstyle\mathbf{W}},\text{vec}})$, $\mathbf{P}_U = \mathbf{K}_{UU} + \mathbf{K}_{FU}^\top \mathbf{\Sigma}^{-1} \mathbf{J}_{\scriptscriptstyle\mathbf{W}}^{-1}\mathbf{K}_{FU}$, $\mathbf{K}_{UU} = \textup{diag}(\kappa(Z_1, Z_1), \dots, \kappa(Z_R, Z_R)) \in \mathbb{R}^{QR \times QR}$ is a diagonal block matrix, and $\mathbf{K}_{FU} \in \mathbb{R}^{NT \times QR}$ is a block-matrix such that the $t,r$-th block of $\mathbf{K}_{FU}$ equals $a_{t,r} \kappa(X, Z_r)$ for $t \in \{1, \dots, T\}$ and $r \in \{1, \dots, R\}$. In addition, $\mathbf{k}_{U\star} \in \mathbb{R}^{QR \times T}$ is the block-matrix such that $ [\mathbf{k}_{U\star}]_{r,t} = a_{t,r} \kappa(Z_r, \mathbf{x}_\star)$. The corresponding ELBO is given as:
\begin{align*}
    \textup{ELBO}(Z, \Sigma, \bm{\theta}) = - \frac{1}{2} &\biggl(\textup{Tr}\left(\mathbf{\Sigma}^{-1/2}\mathbf{J}_{\scriptscriptstyle\mathbf{W}}^{-1/2} (\mathbf{K} - \mathbf{K}_{FU} \mathbf{K}_{UU}^{-1} \mathbf{K}_{FU}^\top) \mathbf{J}_{\scriptscriptstyle\mathbf{W}}^{-1/2}\mathbf{\Sigma}^{-1/2}\right) + C(\Sigma, \bm{\theta})\biggr) \\
    &+ \frac{1}{2} \bm{\nu}^\top \mathbf{K}_{FU} \mathbf{P}_U^{-1}\mathbf{K}_{FU}^\top \bm{\nu} + \frac{1}{2} \log \left(\frac{\textup{det}(\mathbf{K}_{UU})}{\textup{det}(\mathbf{P}_U)}\right).
\end{align*}
where $C(\Sigma, \bm{\theta})$ is a function that depends only on $\Sigma$ and $\bm{\theta}$.

\end{proposition}

The proof is in Appendix~\ref{app:sparse-variational} and extends naturally to other multi-output kernels \citep{moreno2018heterogeneous, nguyen2014collaborative}. The resulting computational complexity reduces to $\mathcal{O}(NTRQ^2)$ (as compared to $\mathcal{O}(N^3 T^3)$ for R-MOGPs), improving scalability when the number of inducing points $Q$ is small enough.

\section{Experiments}\label{sec:experiments}

The code is available at \url{https://github.com/joshuarooijakkers/rmogp}.
We focus on the ICM with a squared exponential kernel $\kappa(\mathbf{x},\mathbf{x}') = \exp(-\|\mathbf{x}-\mathbf{x}'\|_2^2/\ell^2)$ for some $\ell>0$. We evaluate our experiments based on the root mean squared error (RMSE), as well as the negative log predictive density (NLPD). We start with a few synthetic examples illustrating the robustness of our method under different forms of misspecification, then move to applications to real-world data.

\paragraph{Synthetic Multi-task Imputation}

MOGPs are often used when some outputs are only partially observed across the domain, such as in causal inference \citep{dimitriou2024data} and multi-fidelity modelling \citep{Xi2018MultiOutput}. For instance, missing signals from failing sensors can be reconstructed by exploiting correlations with signals from other sensors \citep{osborne2008towards}. However, outliers can not only distort the MOGP's predictions within their own output, but also propagate these errors to correlated outputs where imputation is performed. 
Figure~\ref{fig:toy_example} illustrates this with a simulated ICM for $T=2$ and a large correlation ($\approx 0.88$) among  outputs. The MOGP’s predictions are strongly affected by outliers and  yield corrupted imputations, whilst the R-MOGP provides robustness and more reliable imputation. In Appendix~\ref{app:conditional_expectation}, we further show that this behaviour also extends to non-ICM multi-output kernels.

\begin{figure}[h]
  \centering
  \includegraphics[width=\columnwidth]{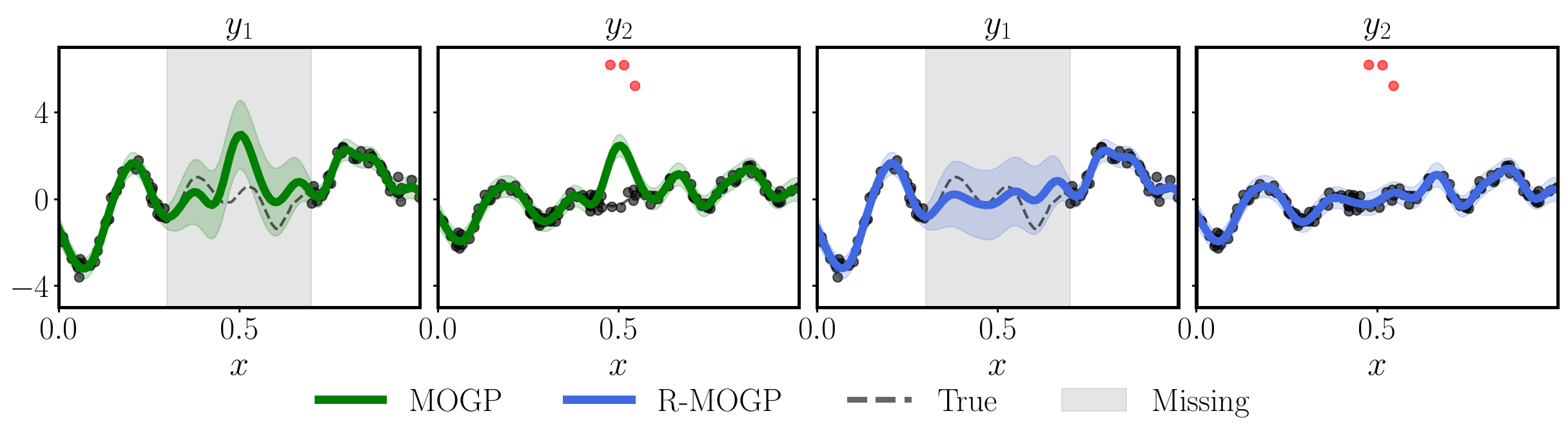}
  \caption{{Synthetic imputation experiment with focused outliers.} We consider a regression problem ($T=2$) with structured missingness in $y_1$ and outliers in $y_2$. The \textcolor{Green}{{MOGP}} predictives are sensitive to outliers, with errors propagating across outputs, whereas \textcolor{RoyalBlue}{{R-MOGP}} is robust in both outputs.}
  \label{fig:toy_example}
\end{figure}

\paragraph{Synthetic Comparisons with Outliers}
We perform a similar experiment for $T=3$ and  10\% outliers of uniform distribution in the first output $t=1$ (see Appendix~\ref{app:synthetic_MOGP_table}), with  results in Table \ref{tab:synthetic_imputation}. 
As part of the comparison, we also consider an R-MOGP with the naive default RCGP weight $w_{\scriptscriptstyle\textup{RCGP}}$ of \citet{altamirano2024robustGP} to show why careful choice of the weight function is crucial for the R-MOGP. 
Additionally, we compare against the performance on a dataset without outliers.
The results show that in the well-specified regime, RMSE and NLPD are comparable across all methods, but that MOGPs experience a clear performance drop in the presence of outliers. 
While R-MOGP with the naive weight $w_{\scriptscriptstyle\textup{RCGP}}$ already improves robustness to outliers, an R-MOGP with the more appropriate weight $w_{\scriptscriptstyle\textup{RMOGP}}$ further outperforms it (see also Appendix~\ref{app:choice_weight_function}). 
In summary, R-MOGPs perform strongly under model misspecification, but are also competitive in the well-specified setting.

\begin{table}[h]%
\caption{Average RMSE and NLPD and standard deviation over 20 seeds under well-specified and outlier-contaminated settings.} 
\label{tab:synthetic_imputation}
\centering
{\begin{tabular}{clcc}
\hline
\textbf{Outliers} & \textbf{Method} & \textbf{RMSE} & \textbf{NLPD} \\ \hline
\multirow{3}{*}{\cmark} & MOGP & $0.20 \pm 0.03$ & $-0.37 \pm 0.04$ \\
 & R-MOGP ($w_{\textup{RCGP}}$) & $0.18 \pm 0.04$ & $-0.44 \pm 0.07$ \\
 & R-MOGP ($w_{\textup{RMOGP}}$) & \textbf{\boldmath$0.11 \pm 0.01$} & \textbf{\boldmath$-0.67 \pm 0.08$} \\ \hline
\multirow{3}{*}{\xmark} & MOGP & $0.093 \pm 0.004$ & $-0.94 \pm 0.04$ \\
 & R-MOGP ($w_{\textup{RCGP}}$) & $0.093 \pm 0.004$ & $-0.94 \pm 0.04$ \\
 & R-MOGP ($w_{\textup{RMOGP}}$) & $0.093 \pm 0.004$ & $-0.95 \pm 0.04$ \\ \hline 
\end{tabular}}
\end{table}

\paragraph{Energy Efficiency Prediction}

\begin{wrapfigure}{r}{0.45\textwidth}
    \vspace{-0.3cm}
  \centering
  \includegraphics[width=\linewidth]{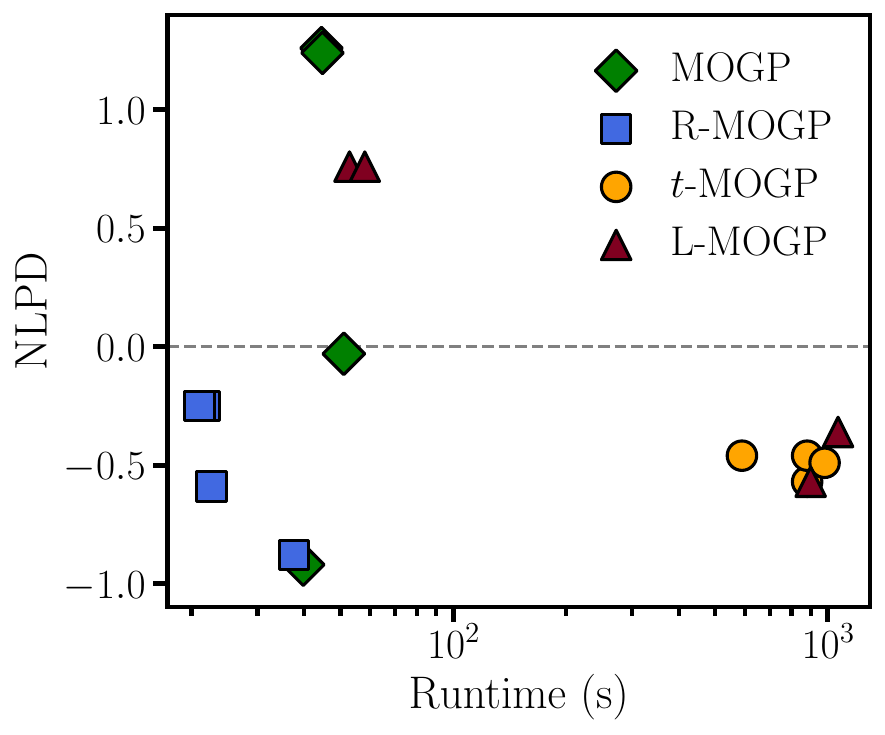}
  \caption{Trade-off between average total runtime and NLPD across several outlier scenarios and methods.
  }
  \vspace{-0.5cm}
  \label{fig:trade-off_NLPD_runtime}
\end{wrapfigure}
Next, we compare our approach to existing state-of-the-art robust MOGPs; an MOGP with a Student-$t$ ($t$-MOGP; \citealp{jylanki2011robust, huang2024robust}) and Laplace likelihood (L-MOGP; \citealp{kuss2006gaussian}). 
As it addresses a different challenge, we do not compare against \citet{yu2007robust}: this method changed the \textit{prior} over the latent function to be a Student-$t$ process---rather than addressing misspecification in the  \textit{noise process} or \textit{likelihood}.
For our comparisons, we rely on the \textsc{Energy Efficiency} dataset from the \href{https://archive.ics.uci.edu/}{UCI repository}, which relates energy efficiency of buildings to architectural characteristics. 
The goal is to predict $T=2$ target variables---heating and cooling loads---using $d=8$ covariates (e.g., surface area, overall height, relative compactness) and $N=576$ training points.
Performances across four different types of outlier scenarios are reported in Table~\ref{tab:energy_efficiency} and runtimes per function evaluation are shown in Figure~\ref{fig:UCI_runtime_fe}. \Cref{fig:trade-off_NLPD_runtime} additionally shows the trade-off between total runtime and NLPD for each method and scenario. Details of our computing setup are provided in Appendix~\ref{app:additional_experiments}.
In inferential terms, R-MOGP outperforms the MOGP across four different types of outlier scenarios, and remains competitive in the absence of outliers. 
Overall, R-MOGPs and $t$-MOGPs fare comparably well in the presence of outliers, substantially outperforming L-MOGP. In the absence of outliers, R-MOGP clearly outperforms both methods.
Further, and as Figure~\ref{fig:UCI_runtime_fe} highlights, R-MOGP's closed form posteriors require a fraction of the computational cost of $t$-MOGPs and L-MOGPs---which require computationally intensive variational approximations. In \Cref{fig:trade-off_NLPD_runtime}, we further find that R-MOGP achieves total runtimes that are orders of magnitude lower than those of the variational methods, as it also requires fewer iterations to converge. We found that L-MOGP converges early under both uniform and asymmetric outlier settings, at the expense of poorer predictive performance, which explains its lower total runtime.
Lastly, as expected from the $\mathcal{O}(N^3 T^3)$ complexity of both methods, computation times between R-MOGPs and  MOGPs are of the same order of magnitude.

\begin{table*}[ht]
  \centering
\begin{minipage}{0.65\textwidth}
    \centering
    \caption{Average and standard deviation of RMSE and NLPD on the \textsc{Energy Efficiency} dataset under different outlier scenarios (10 seeds per scenario). Best results are highlighted in bold.
    }
    \label{tab:energy_efficiency}
    \begin{tabular}{llcc}
    \hline
    \textbf{Outliers} & \textbf{Method} & \textbf{RMSE} & \textbf{NLPD} \\ \hline
     & MOGP & {\boldmath $0.12 \pm 0.01$} & {\boldmath $-0.92 \pm 0.08$} \\
    None & R-MOGP (ours) & {\boldmath $0.12 \pm 0.01$} & {\boldmath $-0.88 \pm 0.07$} \\
     & $t$-MOGP & $0.14 \pm 0.01$ & $-0.57 \pm 0.05$ \\
     & L-MOGP & $0.14 \pm 0.01$ & $-0.57 \pm 0.05$ \\ \hline
     & MOGP & $2.27 \pm 0.29$ & $1.26 \pm 0.06$ \\
    Uniform & R-MOGP (ours) & $0.17 \pm 0.02$ & $-0.25 \pm 0.05$ \\
     & $t$-MOGP & {\boldmath $0.14 \pm 0.01$} & {\boldmath $-0.46 \pm 0.03$} \\
     & L-MOGP & $0.31 \pm 0.02$ & $0.76 \pm 0.02$ \\ \hline
     & MOGP & $2.18 \pm 0.18$ & $1.24 \pm 0.06$ \\
    Asymmetric & R-MOGP (ours) & $0.16 \pm 0.02$ & $-0.31 \pm 0.07$ \\
     & $t$-MOGP & {\boldmath $0.14 \pm 0.01$} & {\boldmath $-0.46 \pm 0.02$} \\
     & L-MOGP & $0.30 \pm 0.02$ & $0.76 \pm 0.01$ \\ \hline
     & MOGP & $0.19 \pm 0.02$ & $-0.03 \pm 0.10$ \\
    Student-$t$ & R-MOGP (ours) & {\boldmath $0.14 \pm 0.01$} & {\boldmath $-0.59 \pm 0.04$} \\
     & $t$-MOGP & {\boldmath $0.14 \pm 0.01$} & $-0.49 \pm 0.03$ \\
     & L-MOGP & $0.18 \pm 0.06$ & $-0.26 \pm 0.24$ \\ \hline
    \end{tabular}
    \end{minipage}
    \hfill
    \begin{minipage}{0.26\textwidth}
    \centering
    \includegraphics[width=\textwidth]{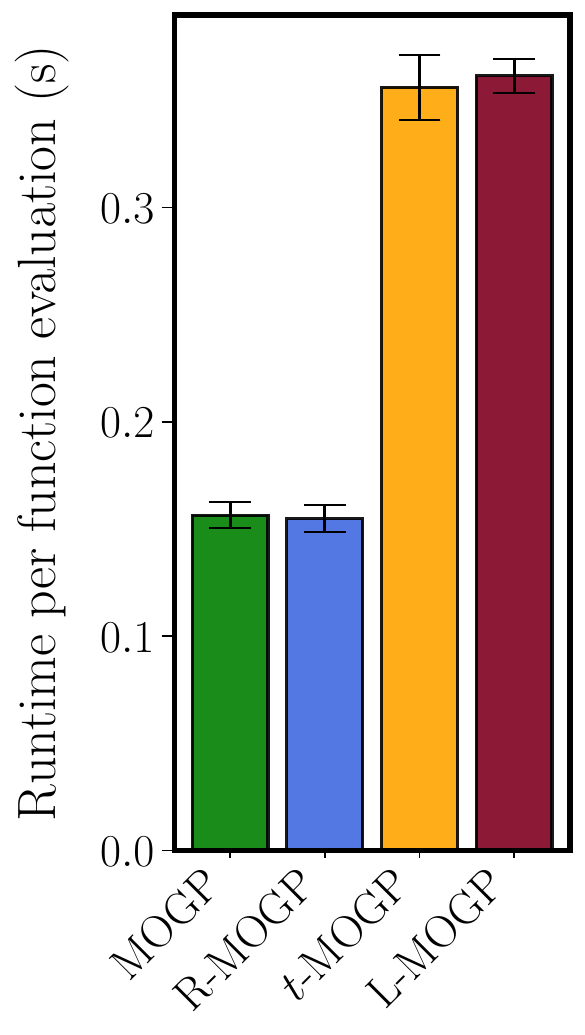}
    \captionof{figure}{Average runtime per function evaluation on the \textsc{Energy Efficiency} dataset across 40 runs.}
    \label{fig:UCI_runtime_fe}
\end{minipage}
\end{table*}
  
\paragraph{Navitoclax Dose-Response Modelling}

Next, we apply our method to a setting where outliers arise naturally rather than synthetically: dose-response analysis in cancer research. In this context, MOGPs are commonly used to model the response of different types of cancer to various drugs \citep{wheeler_bayesian_2019, moran_bayesian_2021, ronneberg2021bayesynergy, ronneberg2023dose,ronneberg2024scalable,gutierrez2024multi}. 
Each output function in an MOGP corresponds to cancer cells isolated from different patients, and the latent functions describe how cell \textit{viability}---the proportion of cells that remain alive since the start of treatment---changes as drug dose increases \citep{gutierrez2024multi, ronneberg2023dose}. 
Modelling this latent function helps pharmaceutical and medical practitioners to identify effective treatments: a steep decline in viability with increasing dose suggests that the drug has a strong effect, whereas a flat response suggests little to no effect. 
Importantly, summary statistics  derived from these curves (most notably area under the curve, or half maximal inhibitory concentration $\text{IC}_{50}$) are explicitly linked to chemical properties of the drugs and genetic markers that enable patient-specific precision medicine.
Yet, outliers and noisy measurements frequently arise in this setting, and stem from a mixture of biological and technical sources such as heterogeneity in cell populations, batch effects or pipetting errors \citep{malo2006statistical}. 
The resulting contaminations can severely distort MOGPs and their summary statistics \citep{wang2020statistical}, which can lead to worse decisions in clinical and drug discovery applications. 

In Figure~\ref{fig:cisplatin_MOGP}, the data point marked in red indicates such a naturally occurring outlier: for unknown reasons, the  viability is significantly inflated, even though similar doses yield viability values close to zero. 
The figure is the result of a drug response analysis of the experimental drug \textit{Navitoclax} based on data from the \href{https://www.cancerrxgene.org/}{Genomics of Drug Sensitivity in Cancer Database} \citep{yang2012genomics}. 
We compare the fit produced by a standard MOGP with our proposed approach on a subset of $T=2$ cell lines with $N=21$ dose measurements given in $\log_{10}(\mu\text{M})$.
For the prior mean, we use the constant function $m_t(\mathbf{x}) = 0.5$ to neutrally balance our prior beliefs halfway between cells being alive ($1.0$) or dead ($0.0$).
We observe that the MOGP’s fit in the first output is visibly distorted by the outlier, with both a biased predictive mean and a large uncertainty.
In contrast, the R-MOGP reduces the influence of the outlier: the produced curve fits the observed points better.
Additionally, its smaller degree of uncertainty aligns with that of the second cell line which is unburdened by outliers.

We quantitatively assess these results by examining the downstream implications of the different predictive distributions. In dose-response modelling, the $\text{IC}_{50}$---defined as the $\log_{10}(\text{dosage})$ at which 50\% of cells remain viable---is a key quantity for precision medicine applications \citep{gutierrez2024multi}. For $y_1$, the estimated $\text{IC}_{50}$ values for the MOGP, R-MOGP, $t$-MOGP, and L-MOGP are $-0.65$, $-0.82$, $-0.86$, and $-0.86$, respectively, with corresponding 95\% credible intervals of $[-0.75,-0.50]$, $[-0.88,-0.71]$, $[-0.91,-0.82]$, and $[-0.91,-0.80]$. The standard MOGP clearly yields both a shifted $\text{IC}_{50}$ estimate and substantially wider uncertainty, whereas the robust models produce more concentrated and consistent estimates. As a result, the presence of a single outlier can propagate to downstream clinical quantities when not explicitly accounted for. We provide further quantitative analysis in Appendix~\ref{app:cancer_research}.

\begin{wrapfigure}{r}{0.5\textwidth} 
  \vspace{-0.4cm}
  \centering
  \includegraphics[width=\linewidth]{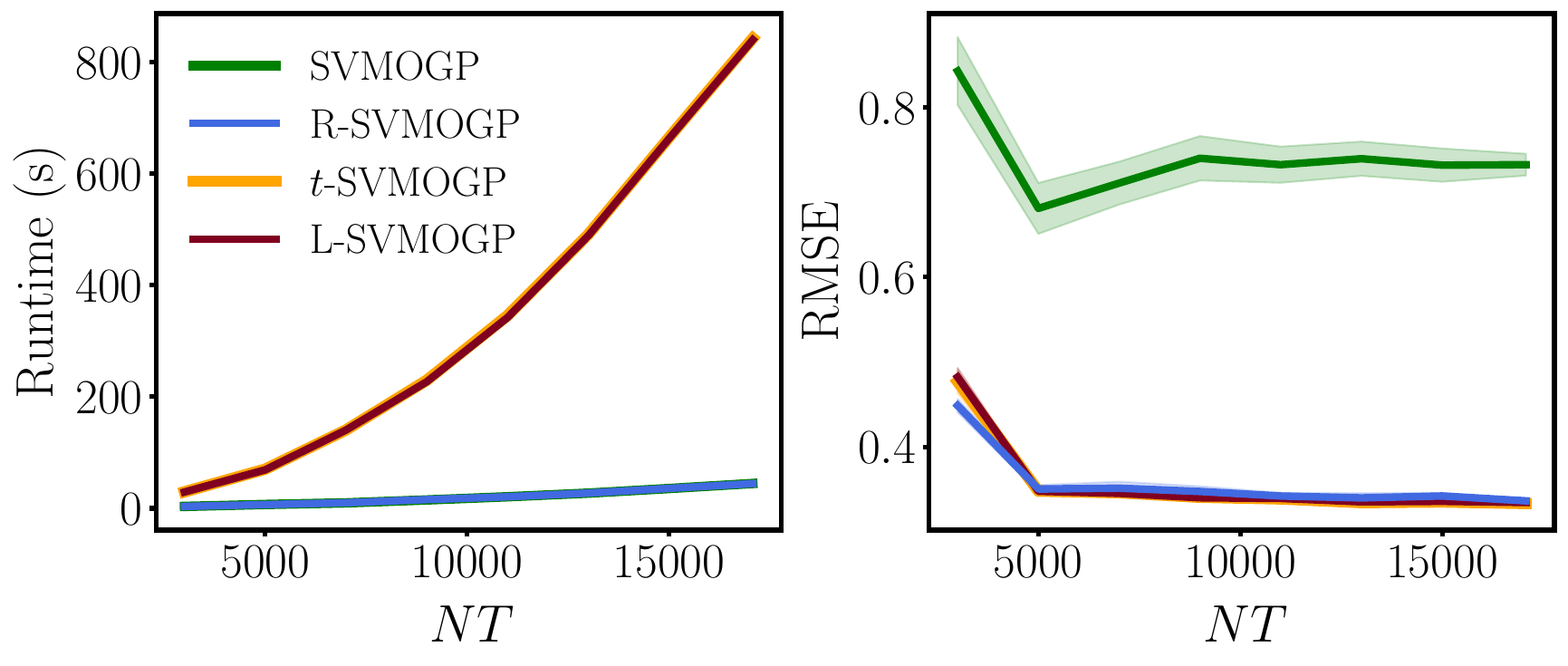}
  \caption{Average runtime in seconds (left) and RMSE (right) across 10 runs on \textsc{Electricity Load} forecasting.
  }
  \label{fig:sparse_variational}
\end{wrapfigure}
\paragraph{Electricity Load Forecasting} Having established the performance of R-MOGP on relatively small-scale datasets, we now consider a larger-scale task to evaluate R-SVMOGP. Specifically, we study daily electricity load forecasting, a setting where MOGPs are widely used in practice \citep{ghasempour2023electric, ghasempour2023short}. 
We analyse one-dimensional time series ($d=1$) of country-level aggregated electricity loads for $T=5$ Northern European countries from January 2010 to September 2020. The \textsc{Electricity Load} dataset is obtained from Open Power System Data \citep{wiese2019open}. The training set contains $N=3{,}415$ observations per country, resulting in a total of $NT=17{,}075$ observations. We compare R-SVMOGP against standard SVMOGP and the sparse variational approaches of $t$-MOGP ($t$-SVMOGP) and L-MOGP (L-SVMOGP). To assess robustness, we contaminate one country with $\epsilon_1=0.05$ focused outliers that mimic the effect of a rare national political event (see Appendix~\ref{app:electricity}). \Cref{fig:sparse_variational} reports runtime and test performance as the training set size increases, using $Q=\sqrt{N}$ inducing points per output. Although R-SVMOGP, $t$-SVMOGP, and L-SVMOGP all achieve significant improvements in accuracy, only R-SVMOGP retains the same computational scaling as standard SVMOGP. In contrast, $t$-SVMOGP and L-SVMOGP scale less favourably because their non-conjugate formulations require numerical integration to evaluate the ELBO.

\section{Conclusion}

MOGPs provide a probabilistic framework for modelling multiple correlated response variables by sharing information across outputs, but are sensitive to outliers and model misspecification if modelled with Gaussian errors. 
While several strategies exist in the literature to address this lack of robustness, they do so at the expense of exact conjugate updates.
To overcome this limitation, we propose R-MOGP: an MOGP that jointly models correlated response variables, retains analytical tractability, and is provably robust to outliers of arbitrarily large magnitude in any individual output. We further resolve the challenges that arise from a naive extension of RCGPs to the multi-output setting by borrowing statistical strength across outputs and introducing a novel hyperparameter optimisation procedure.
Empirically, we show that R-MOGPs provide performances comparable to existing robust MOGPs at a fraction of the latter's computational cost. 
While we explored sparse outliers, R-MOGPs can be tuned for various other types of outliers by adjusting their weight function.
For example, they could be adapted to the case where entire output channels are treated as outliers, or where outliers occur across all channels at the same point.
Lastly, R-MOGPs could be used in various applications of MOGPs, including Bayesian optimisation \citep{swersky2013multi}, multi-fidelity modelling \citep{kennedy2000predicting}, and active learning frameworks \citep{li2022safe}.

\begin{ack}
We thank Dick van Dijk and William Laplante for valuable discussions on early drafts of this paper. JR was supported by the Dutch Cultural Fund Grant (Cultuurfondsbeurs). LR was supported by the European Union’s Horizon Europe research and innovation programme under the Marie Skłodowska-Curie grant agreement No. 101126636, and  Integreat --- Norwegian Centre for knowledge-driven
machine learning, Research Council of Norway project number 332645. FXB and JK were supported by the EPSRC grant [EP/Y011805/1]. FXB was further supported by the EPSRC grant [EP/Y022300/1]. MA was supported by the Bloomberg Data Science PhD fellowship.
\end{ack}

\bibliographystyle{plainnat} 
\bibliography{references}    


\appendix

\onecolumn
\vspace*{-18pt}
\section*{\Large \centering \\ Supplementary Materials
}
\vspace{8pt}
We introduce notation in Appendix~\ref{app:notation}. Appendix~\ref{app:proofs_theoretical_results} presents the proofs of our theoretical results. Finally, Appendix~\ref{app:additional_experiments} contains details of our numerical experiments and additional results that complement the main text.

\section{Notation} \label{app:notation}

In this section, we recap the notation used in this paper:

\begin{itemize}[itemsep=0pt, topsep=0pt, leftmargin=*, label=\scriptsize$\bullet$]
    \item Let $\nabla_y = [\partial/\partial y_1, \dots, \partial/\partial y_{\scriptscriptstyle m}]^\top$, and suppose $f: \mathbb{R}^m \rightarrow \mathbb{R}^m$ and $g:\mathbb{R}^m \rightarrow \mathbb{R}^{m \times m}$ are two sufficiently regular functions. The divergence operator is defined as:
        \begin{align*}
            (\nabla_y \cdot f)(\mathbf{y}) = \sum^m_{i=1} \frac{\partial f_i(\mathbf{y})}{\partial y_i}, \qquad (\nabla_y \cdot g)_j(\mathbf{y}) = \sum^m_{i=1}\frac{\partial g_{ij}(\mathbf{y})}{\partial y_i}, \quad \forall j \in \{1, \dots, m\},
        \end{align*}
    where $f_i$ and $g_{ij}$ denote the $i$-th component of $f$ and the $(i,j)$-th component of $g$, respectively.
    \item We denote by $\tilde{\mathbf{y}} \sim \mathcal{N}(\mu, \Sigma)$ that $\tilde{\mathbf{y}} \in \mathbb{R}^{m}$ follows an $m$-variate Gaussian distribution with mean $\mu \in \mathbb{R}^m$ and covariance matrix $\Sigma \in \mathbb{R}^{m \times m}$. 
    Furthermore, $p(\tilde{\mathbf{y}}) = \mathcal{N}(\tilde{\mathbf{y}}; \mu, \Sigma)$ is the corresponding probability density function.
    \item Let $\mathbf{v} = [v_1, \dots, v_{n}]^\top \in \mathbb{R}^n$, and $\mathbf{M}_1, \dots, \mathbf{M}_m \in \mathbb{R}^{n \times n}$. The diagonal operator is defined as:
    \begin{align*}
        \textup{diag}(\mathbf{v}) &= 
        \begin{bmatrix}
            v_1 & 0 & \cdots & 0 \\
            0 & v_2 & \cdots & 0 \\
            \vdots & \vdots & \ddots & \vdots \\
            0 & 0 & \cdots & v_{n}
        \end{bmatrix} \in \mathbb{R}^{n \times n},\\
        \textup{diag}(\mathbf{M}_1, \dots, \mathbf{M}_m) &= 
        \begin{bmatrix}
            \mathbf{M}_1 & \cdots & \mathbf{0}\\
            \vdots & \ddots & \vdots\\
            \mathbf{0} & \cdots & \mathbf{M}_m
        \end{bmatrix} \in \mathbb{R}^{mn \times mn}.
    \end{align*}
    \item Let $\mathbf{v} = [v_1, \dots, v_{n}]^\top \in \mathbb{R}^n$. Then, $\mathbf{v}^2 = [v_1^2, \dots, v_{n}^2]^\top$, and $\log(\mathbf{v}) = [\log(v_1), \dots, \log(v_{n})]^\top$. Additionally, we abuse notation so that $\log(\textup{diag}(\mathbf{v})) = \textup{diag}(\log(\mathbf{v}))$.
    \item Let $A \in \mathbb{R}^{m \times n}$ and $B \in \mathbb{R}^{p \times q}$, where $a_{ij}$ is the $(i,j)$-th entry of $A$. The Kronecker product $A \otimes B$ is the block matrix defined by:
    \begin{align*}
        A \otimes B = 
        \begin{bmatrix}
            a_{11}B & \cdots & a_{1n}B \\
            \vdots & \ddots & \vdots \\
            a_{m1}B & \cdots & a_{mn}B
        \end{bmatrix}
        \in \mathbb{R}^{mp \times nq}.
    \end{align*}
    \item Let $A \in \mathbb{R}^{m \times n}$, where $a_{ij}$ is the $(i,j)$-th entry of $A$. We define the vectorisation of $A$ column-wise as
    \begin{align*}
        \textrm{vec}(A) = [a_{11}, \dots, a_{m1}, \cdots, a_{1n}, \dots, a_{mn}]^\top \in \mathbb{R}^{mn}
    \end{align*}
    \item Let $A \in \mathbb{R}^{m \times n}$ and $\mathbf{v} \in \mathbb{R}^m$. We define the indexing operator
        $[A]_{i,j} = a_{ij}$ and $[\mathbf{v}]_i = v_i$.
    Furthermore, for integer indices $i_1 \le i_2$ and $j_1 \le j_2$, we define the slicing operators as:
    \begin{align*}
        [A]_{i_1:i_2,\, j_1:j_2} &= 
        \begin{bmatrix}
            a_{i_1 j_1} & \cdots & a_{i_1 j_2} \\
            \vdots & \ddots & \vdots \\
            a_{i_2 j_1} & \cdots & a_{i_2 j_2}
        \end{bmatrix}
        \in \mathbb{R}^{(i_2 - i_1 + 1) \times (j_2 - j_1 + 1)}, \\ [\mathbf{v}]_{i_1:i_2}& = [v_{i_1}, \dots, v_{i_2}]^\top \in \mathbb{R}^{(i_2 - i_1 + 1)}.
    \end{align*}
    That is, $[\cdot]_{i_1:i_2,\, j_1:j_2}$ extracts the submatrix of rows $i_1$ through $i_2$ and columns $j_1$ through $j_2$, while $[\cdot]_{i_1:i_2}$ extracts the corresponding subvector.
\end{itemize}

Finally, we highlight that by expanding the term $\nabla_y \cdot \log (\mathbf{W}^2)$ as part of $\mathbf{m}_{\mathbf{W},\textup{vec}}$ in Proposition~\ref{prop:MORCGP}, we get
\begin{align*}
    \nabla_y \cdot \log (\mathbf{W}^2)& = [\nabla_y \cdot \log(\mathbf{W}_1^2)^\top, \dots, \nabla_y \cdot \log(\mathbf{W}_{\scriptscriptstyle T}^2)^\top]^\top\\
    &= 2\Bigl[\frac{\partial \log (w_1(\mathbf{x}_1, \mathbf{y}_1))}{\partial y_{1,1}}, \dots, \frac{\partial \log (w_1(\mathbf{x}_{\scriptscriptstyle N}, \mathbf{y}_{\scriptscriptstyle N}))}{\partial y_{{\scriptscriptstyle N},1}}, \cdots,\\
    &\qquad \qquad\frac{\partial \log (w_{\scriptscriptstyle T}(\mathbf{x}_{\scriptscriptstyle 1}, \mathbf{y}_{\scriptscriptstyle 1}))}{\partial y_{ \scriptscriptstyle 1,T}}, \dots, \frac{\partial \log (w_{\scriptscriptstyle T}(\mathbf{x}_{\scriptscriptstyle N}, \mathbf{y}_{\scriptscriptstyle N}))}{\partial y_{ \scriptscriptstyle N,T}}\Bigr]^\top \in \mathbb{R}^{NT}
\end{align*}

\section{Proofs of Theoretical Results} \label{app:proofs_theoretical_results}

\subsection{Proof of Proposition~\ref{prop:MORCGP}} \label{app:proof_MORCGP}

Recall the weighted Fisher divergence introduced in \Cref{sec:methodology}:
\begin{align}
D_W(\mathbf{f}) = \mathbb{E}_{\tilde{\mathbf{x}}\sim p_{0,x}}
\Big[ 
    \mathbb{E}_{\tilde{\mathbf{y}}\sim p_0(\cdot \mid \tilde{\mathbf{x}})}
    \Big[\| W(\tilde{\mathbf{x}}, \tilde{\mathbf{y}})^\top 
        \big( s_{\text{model}}(\tilde{\mathbf{x}}, \tilde{\mathbf{y}}) - s_{\text{data}}(\tilde{\mathbf{x}}, \tilde{\mathbf{y}}) \big) \|_2^2 \Big]
\Big],
\label{eq:multivariate_weighted_Fisher}
\end{align}
where $s_{\textup{model}} = \nabla_{\mathbf{y}}\log p_\theta(\mathbf{y}|\mathbf{x})$ and $s_{\textup{data}} = \nabla_{\mathbf{y}}\log p_0(\mathbf{y}|\mathbf{x})$ are score functions of the posited model and the true data, respectively, and $W$ is a matrix-valued weight function that is point-wise invertible and continuously differentiable. As we show in the following, this divergence induces conjugacy through a quadratic loss function and naturally leads to the derivation of our proposed R-MOGP posterior and predictive distributions.

\paragraph{Quadratic Loss Function} \label{app:quadratic_loss}

To begin, we have the following quadratic loss function below.

\begin{lemma}
Under the weighted score matching divergence defined in \Cref{eq:multivariate_weighted_Fisher}, and a Gaussian model $p_{\theta}(\mathbf{y}|\mathbf{x}) = \mathcal{N}(\mathbf{y};\mathbf{f}(\mathbf{x}), \Sigma)$, the generalised Bayes loss function in \Cref{eq:gb_update} takes a quadratic form in $\mathbf{f}$:
\begin{align}
    \mathcal{L}_W(F, Y, X) = \frac{1}{N}\sum^{N}_{i=1} \Big[\mathbf{f}_i^\top A_i \mathbf{f}_i + \mathbf{f}_i^\top \mathbf{b}_i + c_i\Big],
    \label{eq:quadratic_loss}
\end{align}
where 
\begin{align*}
    A_i = \Sigma^{-1}W_iW_i^\top \Sigma^{-1}, \quad
    \mathbf{b}_i = -2(\Sigma^{-1} W_iW_i^\top \Sigma^{-1} \mathbf{y}_i - \nabla_\mathbf{y} \cdot W_iW_i^\top \Sigma^{-1}),
    \end{align*}
    \begin{align*}
    c_i = \mathbf{y}_i^\top \Sigma^{-1} W_iW_i^\top \Sigma^{-1} \mathbf{y}_i - \nabla_\mathbf{y} \cdot (W_iW_i^\top \Sigma^{-1} \mathbf{y}_i),
\end{align*}
for $W_i := W(\mathbf{x}_i, \mathbf{y}_i)$ are terms that do not depend on $\mathbf{f}_i$. They do, however, depend on $\mathbf{x}_i$ and $\mathbf{y}_i$, either directly or via $W_i$.
\end{lemma}

\begin{proof} 
Direct evaluation of \eqref{eq:multivariate_weighted_Fisher} requires access to the true data score $s_{\text{data}}$, which is generally unknown. Fortunately, under mild smoothness and boundary conditions, integration by parts allows us to reformulate the expression up to an additive constant independent of $\mathbf{f}$, denoted by $\overset{+C}{=}$. This weighted version of score matching was first derived in \citet[Theorem 2]{barp2019minimum}, who refer to it as “diffusion score matching.” Using this result, we obtain
\begin{align*}
    D_{W}(\mathbf{f}) \overset{+C}{=} \mathbb{E}_{\tilde{\mathbf{x}} \sim p_{0,x}}\Big[\mathbb{E}_{\tilde{\mathbf{y}}\sim p_0{(\cdot | \tilde{\mathbf{x}})}}\Big[&\big\|W(\tilde{\mathbf{x}}, \tilde{\mathbf{y}})^\top s_{\text{model}}(\tilde{\mathbf{x}}, \tilde{\mathbf{y}})\big\|^2_2\\
    &+ 2 \nabla_\mathbf{y} \cdot \{W(\tilde{\mathbf{x}}, \tilde{\mathbf{y}})W(\tilde{\mathbf{x}}, \tilde{\mathbf{y}})^\top s_{\text{model}}(\tilde{\mathbf{x}}, \tilde{\mathbf{y}})\}\Big]\Big].
\end{align*}
Crucially, this term does not depend on $s_{\textup{data}}$ any more, and only features $p_{0,x}$ through the expectation. This leads to the natural estimator 
\begin{align*}
    \mathcal{L}_W(F, Y, X) = \frac{1}{N} \sum^N_{i=1} \bigg[&\big\|W(\mathbf{x}_i, \mathbf{y}_i)^\top s_{\text{model}}(\mathbf{x}_i, \mathbf{y}_i)\big\|^2_2\\
    &+ 2 \nabla_\mathbf{y} \cdot \Big(W(\mathbf{x}_i, \mathbf{y}_i)W(\mathbf{x}_i, \mathbf{y}_i)^\top s_{\text{model}}(\mathbf{x}_i, \mathbf{y}_i)\Big) \bigg].
\end{align*}
For a Gaussian model $\varepsilon_i \sim \mathcal{N}(0, \Sigma)$, the score vector is given by 
\begin{align*}
    s_{\text{model}}(\mathbf{x}_i, \mathbf{y}_i) &= \nabla_\mathbf{y} \log \bigg(\frac{1}{(2\pi)^{T/2}|\Sigma|^{1/2}}\exp\Big\{-\frac{1}{2}(\mathbf{y}_i-\mathbf{f}(\mathbf{x}_i))^\top\Sigma^{-1}(\mathbf{y}_i-\mathbf{f}(\mathbf{x}_i)\Big\}\bigg) \\
    &= \nabla_\mathbf{y} \big( -\frac{1}{2}(\mathbf{y}_i-\mathbf{f}(\mathbf{x}_i))^\top\Sigma^{-1}(\mathbf{y}_i-\mathbf{f}(\mathbf{x}_i)\big) \\
    &= -\frac{1}{2} \cdot 2 \Sigma^{-1} (\mathbf{y}_i - \mathbf{f}(\mathbf{x}_i)) = \Sigma^{-1}(\mathbf{f}(\mathbf{x}_i) - \mathbf{y}_i) = \Sigma^{-1}(\mathbf{f}_i - \mathbf{y}_i),
\end{align*}
where we used that the gradient of a constant vanishes and the factor of 2 comes from differentiating the quadratic form. The loss can then be rewritten as
\begin{align*}
     \mathcal{L}_W(F, Y, X) = \frac{1}{N} \sum^N_{i=1} \bigg[\underbrace{\|W_i^\top \Sigma^{-1} \left(\mathbf{f}_i - \mathbf{y}_i\right)\|^2_2}_{(\star)} + 2 \underbrace{\nabla \cdot \left(W_iW_i^\top \Sigma^{-1} \left(\mathbf{f}_i - \mathbf{y}_i\right)\right)}_{(\star\star)} \bigg].
\end{align*}
We define $Q_i := W_iW_i^\top \Sigma^{-1}$. We can express ($\star$) as
\begin{align*}
(\star) &= \left(\mathbf{f}_i - \mathbf{y}_i\right)^\top \Sigma^{-1} W_i W_i^\top \Sigma^{-1} \left(\mathbf{f}_i - \mathbf{y}_i\right)\\ 
   &= \mathbf{f}_i^\top \Sigma^{-1} W_i W_i^\top \Sigma^{-1} \mathbf{f}_i - 2 \mathbf{f}_i^\top \Sigma^{-1} W_i W_i^\top \Sigma^{-1} \mathbf{y}_i + \mathbf{y}_i^\top \Sigma^{-1} W_i W_i^\top \Sigma^{-1} \mathbf{y}_i \\
   &= \mathbf{f}_i^\top \Sigma^{-1} Q_i \mathbf{f}_i - 2 \mathbf{f}_i^\top \Sigma^{-1} Q_i \mathbf{y}_i + \mathbf{y}_i^\top \Sigma^{-1} Q_i \mathbf{y}_i.
\end{align*}
by expanding the quadratic form. Similarly, for ($\star\star$)
\begin{align*}
(\star\star) &= \nabla_\mathbf{y} \cdot \left(W_i W_i^\top \Sigma^{-1} \mathbf{f}_i\right) - \nabla_\mathbf{y} \cdot \left(W_i W_i^\top \Sigma^{-1} \mathbf{y}_i\right) \\
   &= \nabla_\mathbf{y} \cdot \left(Q_i \mathbf{f}_i\right) - \nabla_\mathbf{y} \cdot \left(Q_i \mathbf{y}_i\right) \\
   &= \mathbf{f}_i^\top \left(\nabla_\mathbf{y} \cdot Q_i\right) - \nabla_\mathbf{y} \cdot \left(Q_i \mathbf{y}_i\right).
\end{align*}
where we used the product rule for the divergence operator.
Substituting these expression in the loss, we get
\begin{align*}
     &\mathcal{L}_W(F, Y, X) \\
     &= \frac{1}{N} \sum^N_{i=1} \bigg[\mathbf{f}_i^\top \Sigma^{-1} Q_i \mathbf{f}_i - 2 \mathbf{f}_i^\top \Sigma^{-1} Q_i \mathbf{y}_i + 2 \mathbf{f}_i^\top \left(\nabla_\mathbf{y} \cdot Q_i\right) + \mathbf{y}_i^\top \Sigma^{-1} Q_i \mathbf{y}_i - \nabla_\mathbf{y} \cdot \left(Q_i \mathbf{y}_i\right) \bigg]\\
     &= \frac{1}{N} \sum^N_{i=1} \bigg[\mathbf{f}_i^\top \Sigma^{-1} Q_i \mathbf{f}_i - 2 \mathbf{f}_i^\top \left(\Sigma^{-1} Q_i \mathbf{y}_i - \nabla_\mathbf{y} \cdot Q_i\right) + \mathbf{y}_i^\top \Sigma^{-1} Q_i \mathbf{y}_i - \nabla_\mathbf{y} \cdot \left(Q_i \mathbf{y}_i\right) \bigg].
\end{align*}
This loss function is quadratic in $\mathbf{f}_i$ where $A_i = \Sigma^{-1}Q_i$ and $\mathbf{b}_i = -2(\Sigma^{-1} Q_i \mathbf{y}_i - \nabla_\mathbf{y} \cdot Q_i)$, and $c_i = \mathbf{y}_i^\top \Sigma^{-1} Q_i \mathbf{y}_i - \nabla_\mathbf{y} \cdot \left(Q_i \mathbf{y}_i\right)$ as in \Cref{eq:quadratic_loss}. 
\end{proof}
We note that this result generalises \citet[Appendix A.1]{altamirano2024robustGP}, who show that the loss function for scalar RCGPs is quadratic in $\mathbf{f}$ using the univariate weighted score-matching divergence. The corresponding RCGP loss function can be recovered as a special case by setting $T = 1$, $\mathbf{\Sigma} = \sigma^2$ and $W_i = w(\mathbf{x}_i, y_i) \in \mathbb{R}$.

\paragraph{Posterior Distribution}

We have the following closed-form solution for the posterior distribution.

\begin{lemma} \label{lem:posterior}
Let $\textup{vec}(\mathcal{E}) \sim \mathcal{N}(0, \mathbf{\Sigma})$ where $\mathbf{\Sigma} = \Sigma \otimes I_N$ for a $T$-dimensional diagonal $\Sigma$, and $W_i = \textup{diag}(w_1(\mathbf{x}_i, \mathbf{y}_i), \dots, w_T(\mathbf{x}_i, \mathbf{y}_i))$, where $y_{i,t} \mapsto w_t(\mathbf{x}_i, \mathbf{y}_i)$ is continuously differentiable. Then, the R-MOGP posterior distribution over $\mathbf{f}_\textup{vec}$ is given by
\begin{align*}
    p_{\scriptscriptstyle W}(\mathbf{f}_{\textup{vec}} | \mathbf{y}_{\textup{vec}}, \mathbf{x}_{\textup{vec}}) &= \mathcal{N}(\mathbf{f}_{\textup{vec}} ; \bm{\mu}_{\textup{vec}}^{\scriptscriptstyle \textup{RMOGP}}, \bm{\Sigma}^{\scriptscriptstyle \textup{RMOGP}}),\\
    \bm{\mu}_{\textup{vec}}^{\scriptscriptstyle \textup{RMOGP}} &= \mathbf{m}_{\textup{vec}} + \mathbf{K}(\mathbf{K} + \mathbf{\Sigma} \,{\mathbf{J}_{\scriptscriptstyle \mathbf{W}}})^{-1}(\mathbf{y}_{\textup{vec}} - {\mathbf{m}_{{\scriptscriptstyle\mathbf{W}},\textup{vec}})},\\
    \bm{\Sigma}^{\scriptscriptstyle \textup{RMOGP}} &= \mathbf{K}(\mathbf{K} + \mathbf{\Sigma} \,{\mathbf{J}_{\scriptscriptstyle \mathbf{W}}})^{-1}\mathbf{\Sigma} \, {\mathbf{J}_{\scriptscriptstyle \mathbf{W}}},
\end{align*}
where $\mathbf{J}_{\scriptscriptstyle\mathbf{W}} = \frac{1}{2} \mathbf{\Sigma} \mathbf{W}^{-2}$ and $\ \mathbf{m}_{{\scriptscriptstyle\mathbf{W}},\textup{vec}} = \mathbf{m}_\textup{vec} + \mathbf{\Sigma} \left( \nabla_\mathbf{y} \cdot \log(\mathbf{W}^2)\right)$, where $\mathbf{W} = \textup{diag}(\mathbf{W}_1, \dots, \mathbf{W}_{\scriptscriptstyle T})$ and $\nabla_\mathbf{y} \cdot \log (\mathbf{W}^2) = [\nabla_\mathbf{y} \cdot \log(\mathbf{W}_1^2)^\top, \dots, \nabla_\mathbf{y} \cdot \log(\mathbf{W}_{\scriptscriptstyle T}^2)^\top]^\top$.
\end{lemma}
\begin{proof}
The diagonal structure $W_i = \text{diag}(w_1(\mathbf{x}_i, \mathbf{y}_i), \dots, w_T(\mathbf{x}_i, \mathbf{y}_i))$ implies that $\Sigma^{-1}Q_i = \Sigma^{-2}W_i^2$, since $\Sigma$ is also diagonal. Additionally, we can write $\nabla_\mathbf{y} \cdot Q_i = \nabla_\mathbf{y} \cdot (W_iW_i^\top \Sigma^{-1}) = \Sigma^{-1} (\nabla_\mathbf{y} \cdot W_i^2)$ since $\Sigma^{-1}$ is independent of $\mathbf{y}$. In addition, from $\nabla_\mathbf{y} \cdot \log(W_i^2) = W_i^{-2} \nabla_\mathbf{y} \cdot W_i^2$ it follows that $\nabla_\mathbf{y} \cdot W_i^2 = \Sigma^{-1} W_i^2 \Sigma (\nabla_\mathbf{y} \cdot \log (W_i^2))$. Combining these expressions, we obtain $\nabla_\mathbf{y} \cdot Q_i = \Sigma^{-2} W_i^2 \Sigma \nabla_\mathbf{y} \log (W_i^2)$. We can then simplify the loss function as
\begin{align*}
     \mathcal{L}_W(F, Y, X) &\overset{+C}{=} \frac{1}{N} \sum^N_{i=1} \biggl[\mathbf{f}_i^\top \Sigma^{-2}W_i^2 \mathbf{f}_i - 2 \mathbf{f}_i^\top \Sigma^{-2} W_i^2 \Bigl(\mathbf{y}_i - \Sigma \nabla_\mathbf{y} \cdot \log (W_i^2)\Bigr)\biggr]\\
     &= \frac{1}{2N} \Bigl[2\mathbf{f}_{\text{vec}}^\top \mathbf{\Sigma}^{-2}\mathbf{W}^{2}\mathbf{f}_{\text{vec}} - 4 \mathbf{f}_{\text{vec}}^\top \mathbf{\Sigma}^{-2}\mathbf{W}^{2} \Bigl(\mathbf{y}_{\text{vec}} - \mathbf{\Sigma}\nabla_\mathbf{y}\cdot \log (\mathbf{W}^2)\Bigr)\Bigr]\\
     &= \frac{1}{2N} \Bigl[\mathbf{f}_{\text{vec}}^\top \mathbf{\Sigma}^{-1}\mathbf{J}_{\scriptscriptstyle\mathbf{W}}^{-1}\mathbf{f}_{\text{vec}} - 2 \mathbf{f}_{\text{vec}}^\top \mathbf{\Sigma}^{-1}\mathbf{J}_{\scriptscriptstyle\mathbf{W}}^{-1} (\mathbf{y}_{\text{vec}} - \mathbf{m}_{{\scriptscriptstyle\mathbf{W}},\text{vec}} + \mathbf{m}_{\text{vec}}) \Bigr],
\end{align*}
where $\mathbf{W} = \text{diag}(\mathbf{W}_1, \dots, \mathbf{W}_{\scriptscriptstyle T})$, $\mathbf{m}_{{\scriptscriptstyle\mathbf{W}},\text{vec}} = \mathbf{m}_{\text{vec}} + \mathbf{\Sigma}\nabla_\mathbf{y} \cdot \log (\mathbf{W}^2)$, and $\mathbf{J}_{\scriptscriptstyle\mathbf{W}} = \frac{1}{2}\mathbf{\Sigma}\mathbf{W}^{-2}$. In the second equality, since $\Sigma$ and $W_i$ are diagonal, all matrix products are interpreted as element-wise (Hadamard) multiplications. We can then derive the R-MOGP posterior by substituting the above loss function and the MOGP prior into the GB update, given in \Cref{eq:gb_update}:
\begin{align*}
    p_{\scriptscriptstyle W}(\mathbf{f}_\textup{vec} | Y, X) &\propto p(\mathbf{f}_{\text{vec}}|X) \exp (- N \mathcal{L}_{W}(F, Y, X))\\
    &\propto \exp \biggl\{\!-\frac{1}{2}\left(\mathbf{f}_{\text{vec}} - \mathbf{m}_{\text{vec}}\right)^\top\mathbf{K}^{-1}\left(\mathbf{f}_{\text{vec}} - \mathbf{m}_{\text{vec}}\right)\!\biggr\}\\
    &\qquad \times \exp\biggl\{\!-\frac{1}{2}\Bigl(\mathbf{f}_{\text{vec}}^\top \mathbf{\Sigma}^{-1}\mathbf{J}_{\scriptscriptstyle\mathbf{W}}^{-1}\mathbf{f}_{\text{vec}} - 2 \mathbf{f}_{\text{vec}}^\top \mathbf{\Sigma}^{-1}\mathbf{J}_{\scriptscriptstyle\mathbf{W}}^{-1} \left(\mathbf{y}_{\text{vec}} - \mathbf{m}_{{\scriptscriptstyle\mathbf{W}},\text{vec}} + \mathbf{m}_{\text{vec}}\right)\Bigr)\!\biggr\}\\
    &= \exp \biggl\{-\frac{1}{2}\biggl[\left(\mathbf{f}_{\text{vec}} - \mathbf{m}_{\text{vec}}\right)^\top\mathbf{K}^{-1}\left(\mathbf{f}_{\text{vec}} - \mathbf{m}_{\text{vec}}\right) + \mathbf{f}_{\text{vec}}^\top \mathbf{\Sigma}^{-1}\mathbf{J}_{\scriptscriptstyle\mathbf{W}}^{-1}\mathbf{f}_{\text{vec}}\\
    &\qquad \qquad\qquad - 2 \mathbf{f}_{\text{vec}}^\top \mathbf{\Sigma}^{-1}\mathbf{J}_{\scriptscriptstyle\mathbf{W}}^{-1} \left(\mathbf{y}_{\text{vec}} - \mathbf{m}_{{\scriptscriptstyle\mathbf{W}},\text{vec}} + \mathbf{m}_{\text{vec}}\right)\biggr]\biggr\},
\end{align*}
where in the second equality, the terms are combined under a single exponential function. We simplify the function by only considering the terms that depend on $\mathbf{f}_\text{vec}$:
\begin{align*}
    p_{\scriptscriptstyle W}(\mathbf{f}_\textup{vec} | Y, X) 
    &\propto \exp \biggl\{-\frac{1}{2}\biggl[\mathbf{f}_{\text{vec}}^\top\mathbf{K}^{-1}\mathbf{f}_{\text{vec}} - 2\mathbf{f}_{\text{vec}}^\top\mathbf{K}^{-1}\mathbf{m}_{\text{vec}} + \mathbf{f}_{\text{vec}}^\top \mathbf{\Sigma}^{-1} \mathbf{J}_{\scriptscriptstyle\mathbf{W}}^{-1} \mathbf{f}_{\text{vec}}\\
    &\qquad  \qquad\qquad- 2 \mathbf{f}_{\text{vec}}^\top \mathbf{\Sigma}^{-1} \mathbf{J}_{\scriptscriptstyle\mathbf{W}}^{-1} \big(\mathbf{y}_{\text{vec}} - \mathbf{m}_{{\scriptscriptstyle\mathbf{W}},\text{vec}} + \mathbf{m}_{\text{vec}}\big)\biggr]\biggr\}\\
    &= \exp \biggl\{-\frac{1}{2}\biggl[\mathbf{f}_{\text{vec}}^\top\big(\mathbf{K}^{-1} + \mathbf{\Sigma}^{-1}\mathbf{J}_{\scriptscriptstyle\mathbf{W}}^{-1}\big)\mathbf{f}_{\text{vec}} - 2\mathbf{f}_{\text{vec}}^\top\mathbf{K}^{-1}\mathbf{m}_{\text{vec}}\\
    &\qquad \qquad\qquad- 2 \mathbf{f}_{\text{vec}}^\top \mathbf{\Sigma}^{-1}\mathbf{J}_{\scriptscriptstyle\mathbf{W}}^{-1} \big(\mathbf{y}_{\text{vec}} - \mathbf{m}_{{\scriptscriptstyle\mathbf{W}},\text{vec}} + \mathbf{m}_{\text{vec}}\big)\biggr]\biggr\}
\end{align*}
We can then factor out $\mathbf{m}_\textup{vec}$ in the second term:
\begin{align}
    &= \exp \biggl\{-\frac{1}{2}\biggl[\mathbf{f}_{\text{vec}}^\top(\mathbf{K}^{-1} + \mathbf{\Sigma}^{-1}\mathbf{J}_{\scriptscriptstyle\mathbf{W}}^{-1})\mathbf{f}_{\text{vec}}\notag\\
    &\qquad \qquad\qquad- 2\mathbf{f}_{\text{vec}}^\top\Bigl(\mathbf{K}^{-1}\mathbf{m}_{\text{vec}} + \mathbf{\Sigma}^{-1}\mathbf{J}_{\scriptscriptstyle\mathbf{W}}^{-1} \left(\mathbf{y}_{\text{vec}} - \mathbf{m}_{{\scriptscriptstyle\mathbf{W}},\text{vec}} + \mathbf{m}_{\text{vec}}\right)\Bigr)\biggr]\biggr\}\notag\\
    &= \exp \biggl\{-\frac{1}{2}\biggl[\mathbf{f}_{\text{vec}}^\top(\mathbf{K}^{-1} + \mathbf{\Sigma}^{-1}\mathbf{J}_{\scriptscriptstyle\mathbf{W}}^{-1})\mathbf{f}_{\text{vec}}\notag\\
    &\qquad \qquad\qquad- 2\mathbf{f}_{\text{vec}}^\top\Bigl((\mathbf{K}^{-1} + \mathbf{\Sigma}^{-1}\mathbf{J}_{\scriptscriptstyle\mathbf{W}}^{-1})\mathbf{m}_{\text{vec}} + \mathbf{\Sigma}^{-1}\mathbf{J}_{\scriptscriptstyle\mathbf{W}}^{-1} \left(\mathbf{y}_{\text{vec}} - \mathbf{m}_{{\scriptscriptstyle\mathbf{W}},\text{vec}} \right)\Bigr)\biggr]\biggr\} \label{eq:before_completing_squares}
\end{align}
By completing squares, the posterior follows an $NT$-variate Gaussian distribution:
\begin{align*}
    p_{\scriptscriptstyle W}(\mathbf{f}_\textup{vec} | Y, X) &\propto \exp \biggl\{-\frac{1}{2}(\mathbf{f}_{\text{vec}}-\bm{\mu}_{\textup{vec}}^{\scriptscriptstyle \textup{RMOGP}})^\top(\bm{\Sigma}^{\scriptscriptstyle \textup{RMOGP}})^{-1}(\mathbf{f}_{\text{vec}}-\bm{\mu}_{\textup{vec}}^{\scriptscriptstyle \textup{RMOGP}})\biggr\},\\
    \bm{\Sigma}^{\scriptscriptstyle \textup{RMOGP}} &= (\mathbf{K}^{-1} + \mathbf{\Sigma}^{-1}\mathbf{J}_{\scriptscriptstyle\mathbf{W}}^{-1})^{-1} = \mathbf{K}(\mathbf{K} + \mathbf{\Sigma} \mathbf{J}_{\scriptscriptstyle\mathbf{W}})^{-1}\mathbf{\Sigma}\mathbf{J}_{\scriptscriptstyle\mathbf{W}}\\
    \bm{\mu}_{\textup{vec}}^{\scriptscriptstyle \textup{RMOGP}} &= \bm{\Sigma}^{\scriptscriptstyle \textup{RMOGP}} ((\mathbf{K}^{-1} + \mathbf{\Sigma}^{-1}\mathbf{J}_{\scriptscriptstyle\mathbf{W}}^{-1})\mathbf{m}_{\text{vec}} + \mathbf{\Sigma}^{-1}\mathbf{J}_{\scriptscriptstyle\mathbf{W}}^{-1} \left(\mathbf{y}_{\text{vec}} - \mathbf{m}_{{\scriptscriptstyle\mathbf{W}},\text{vec}} \right))\\
    &= \mathbf{m}_{\text{vec}} + \mathbf{K}(\mathbf{K} + \mathbf{\Sigma} \mathbf{J}_{\scriptscriptstyle\mathbf{W}})^{-1} \left(\mathbf{y}_{\text{vec}} - \mathbf{m}_{{\scriptscriptstyle\mathbf{W}},\text{vec}} \right),
\end{align*}
where we obtain $\bm{\Sigma}^{\scriptscriptstyle \textup{RMOGP}}$ by using that for two invertible matrices $A$ and $B$, it holds that $(A^{-1} + B^{-1})^{-1} = A(A + B)^{-1}B$. Additionally, we obtain $\bm{\mu}^{\scriptscriptstyle \textup{RMOGP}}$ since the second term in \Cref{eq:before_completing_squares} equals $-2\mathbf{f}_\textup{vec}(\mathbf{\Sigma}^\textup{RMOGP})^{-1}\bm{\mu}^{\scriptscriptstyle \textup{RMOGP}}$, and $\mathbf{\Sigma}$ and $\mathbf{J}_{\scriptscriptstyle \mathbf{W}}$ are both diagonal. Note that $\bm{\Sigma}^{\scriptscriptstyle \textup{RMOGP}}$ is positive semidefinite, since it is the inverse of a sum of two positive semidefinite matrices.
\end{proof}

We again note that this posterior is a generalisation of the posterior predictive in \citet[Appendix A.1]{altamirano2024robustGP}. We recover this posterior by selecting $T=1$, $\mathbf{\Sigma} = \sigma^2$, and $W_i = w(\mathbf{x}_i, y_i)$.

\paragraph{Predictive Distribution}

We now prove Proposition~\ref{prop:MORCGP} which states that, under the posterior defined in Lemma~\ref{lem:posterior}, the posterior predictive distribution over $\mathbf{f}_\star = \mathbf{f}(\mathbf{x}_\star)$ at $\mathbf{x}_\star$ is Gaussian:
\begin{align*}
    p_{\scriptscriptstyle W}(\mathbf{f}_\star | \mathbf{x}_\star, Y, X) &= \mathcal{N}(\mathbf{f}_\star ; \bm{\mu}_\star^{\scriptscriptstyle \textup{RMOGP}}, \bm{\Sigma}_\star^{\scriptscriptstyle \textup{RMOGP}}),\\
    \bm{\mu}_\star^{\scriptscriptstyle \textup{RMOGP}} &= \mathbf{m}_\star + \mathbf{k}_\star^\top(\mathbf{K} + \mathbf{\Sigma} \,{\mathbf{J}_{\scriptscriptstyle \mathbf{W}}})^{-1}(\mathbf{y}_{\textup{vec}} - {\mathbf{m}_{\mathbf{W},\textup{vec}}}),\\
    \bm{\Sigma}_\star^{\scriptscriptstyle \textup{RMOGP}} &= \mathbf{k}_{\star\star} - \mathbf{k}_\star^\top(\mathbf{K} + \mathbf{\Sigma} \,{\mathbf{J}_{\scriptscriptstyle \mathbf{W}}})^{-1}\mathbf{k}_\star.
\end{align*}

\begin{proof}
We first derive the predictive for $\mathbf{m}_\text{vec} = 0$ and then extend it to an arbitrary prior mean. We note that conditioning on the matrices $F$ and $Y$ is equivalent to conditioning on their vectorised forms, $\mathbf{f}_{\textrm{vec}} = \mathrm{vec}(F)$ and $\mathbf{y}_{\textrm{vec}} = \mathrm{vec}(Y)$, since they contain the same information. Given the posterior distribution over $\mathbf{f}_\textup{vec}$, the conditional distribution over $\mathbf{f}_\star$ is given by
\begin{align*}
p(\mathbf{f}_\star | \mathbf{x}_\star, \mathbf{f}_{\textrm{vec}}, \mathbf{y}_{\textrm{vec}}, X) &= \mathcal{N}(\mathbf{f}_\star; \breve{\mu}, \breve{\Sigma})\\
\breve{\mu} &= \mathbf{k}_\star^\top \mathbf{K}^{-1} \mathbf{f}_\text{vec} = a^\top \mathbf{f}_\text{vec}\\
\breve{\Sigma} &= \mathbf{k}_{\star\star} - \mathbf{k}_\star^\top \mathbf{K}^{-1} \mathbf{k}_\star,
\end{align*}
where $a = \mathbf{K}^{-1}\mathbf{k}_\star$ \citep{alvarez2012kernels}. Following Appendix A.1 of \citet{altamirano2024robustGP}, it can be shown that the density of the predictive distribution is also Gaussian, since it is obtained by integrating the product of two Gaussian densities:
\begin{align*}
    p_{\scriptscriptstyle W}(\mathbf{f}_\star | \mathbf{x}_\star, \mathbf{y}_{\textrm{vec}}, X) &= \int_{\mathbb{R}^{\scriptscriptstyle NT}} p(\mathbf{f}_\star | \mathbf{x}_\star, \mathbf{f}_{\textrm{vec}}, \mathbf{y}_{\textrm{vec}}, X) p_{\scriptscriptstyle W}(\mathbf{f}_{\textrm{vec}}|\mathbf{y}_{\textrm{vec}}, X) d\mathbf{f}_{\textrm{vec}}\\
    &= \mathcal{N}(\mathbf{f}_\star ; \bm{\mu}_\star^{\scriptscriptstyle \textup{RMOGP}}, \bm{\Sigma}_\star^{\scriptscriptstyle \textup{RMOGP}}),\\
    \bm{\mu}_\star^{\scriptscriptstyle \textup{RMOGP}} &= a^\top \bm{\mu}_{\textup{vec}}^{\scriptscriptstyle \textup{RMOGP}},\\
    \bm{\Sigma}_\star^{\scriptscriptstyle \textup{RMOGP}} &= \breve{\Sigma} + a^\top \bm{\Sigma}^{\scriptscriptstyle \textup{RMOGP}} a.
\end{align*}
For more details on the derivation, we refer the reader to \citet{altamirano2024robustGP}. By substituting for $a = \mathbf{K}^{-1}\mathbf{k}_\star$, and $\bm{\mu}_{\textup{vec}}^{\scriptscriptstyle \textup{RMOGP}}$ given in Lemma~\ref{lem:posterior} in $\bm{\mu}_{\textup{vec}}^{\scriptscriptstyle \textup{RMOGP}}$, we obtain the desired expression for the R-MOGP predictive mean:
\begin{align*}
    \bm{\mu}_\star^{\scriptscriptstyle \textup{RMOGP}} &= \mathbf{k}_\star^\top \mathbf{K}^{-1} \mathbf{K}(\mathbf{K} + \mathbf{\Sigma} \mathbf{J}_{\scriptscriptstyle\mathbf{W}})^{-1} \left(\mathbf{y}_{\text{vec}} - \mathbf{m}_{{\scriptscriptstyle\mathbf{W}},\text{vec}} \right)\\
    &= \mathbf{k}_\star^\top (\mathbf{K} + \mathbf{\Sigma} \mathbf{J}_{\scriptscriptstyle\mathbf{W}})^{-1} \left(\mathbf{y}_{\text{vec}} - \mathbf{m}_{{\scriptscriptstyle\mathbf{W}},\text{vec}} \right).
\end{align*}
Similarly, we obtain the R-MOGP predictive variance by substituting for $\breve{\Sigma} = \mathbf{k}_{\star\star} - \mathbf{k}_\star^\top \mathbf{K}^{-1} \mathbf{k}_\star$, $a = \mathbf{K}^{-1}\mathbf{k}_\star$, and $\bm{\Sigma}^{\scriptscriptstyle \textup{RMOGP}}$ given in Lemma~\ref{lem:posterior} in $\bm{\Sigma}_\star^{\scriptscriptstyle \textup{RMOGP}}$:
\begin{align*}
    \bm{\Sigma}_\star^{\scriptscriptstyle \textup{RMOGP}} &= \mathbf{k}_{\star\star} - \mathbf{k}_\star^\top \mathbf{K}^{-1} \mathbf{k}_\star + \mathbf{k}_\star^\top \mathbf{K}^{-1} (\mathbf{K}^{-1} + \mathbf{\Sigma}^{-1}\mathbf{J}_{\scriptscriptstyle\mathbf{W}}^{-1})^{-1} \mathbf{K}^{-1}\mathbf{k}_\star\\
    &= \mathbf{k}_{\star\star} - \mathbf{k}_\star^\top \mathbf{K}^{-1} \mathbf{k}_\star + \mathbf{k}_\star^\top \mathbf{K}^{-1} (\mathbf{K} - \mathbf{K} (\mathbf{K} + \mathbf{\Sigma}\mathbf{J}_{\scriptscriptstyle\mathbf{W}})^{-1}\mathbf{K}) \mathbf{K}^{-1}\mathbf{k}_\star
\end{align*}
where in the second equality we use the Woodbury matrix identity \citep[Chapter 2.1.3]{golub2013matrix} which states that for invertible matrices $A$ and $B$, where $B$ is diagonal, it holds that $(A^{-1} + B^{-1})^{-1} = A - A(A + B)^{-1}A$. We can then expand the bracketed expressions and simplify:
\begin{align*}
    \bm{\Sigma}_\star^{\scriptscriptstyle \textup{RMOGP}}&= \mathbf{k}_{\star\star} - \mathbf{k}_\star^\top \mathbf{K}^{-1} \mathbf{k}_\star + \mathbf{k}_\star^\top\mathbf{K}^{-1}\mathbf{k}_\star - \mathbf{k}_\star^\top (\mathbf{K} + \mathbf{\Sigma}\mathbf{J}_{\scriptscriptstyle\mathbf{W}})^{-1}\mathbf{k}_\star\\
    &= \mathbf{k}_{\star\star} - \mathbf{k}_\star^\top (\mathbf{K} + \mathbf{\Sigma}\mathbf{J}_{\scriptscriptstyle\mathbf{W}})^{-1}\mathbf{k}_\star,
\end{align*}
Now, if the prior mean function $\mathbf{m}_\textrm{vec}$ is non-zero, we replace $\mathbf{f}_\textrm{vec}$ by $\mathbf{f}_\textrm{vec} - \mathbf{m}_\textrm{vec}$, since for a function $\mathbf{f} \sim \mathcal{GP}(\mathbf{m}, \mathcal{K})$ it holds that $\mathbf{f}-\mathbf{m} \sim \mathcal{GP}(\mathbf{0}, \mathcal{K})$. Thus, we add back $\mathbf{m}_\textrm{vec}$ in the posterior, which results in the desired distribution.
\end{proof}

\subsection{LOO-CV Predictive Distribution} \label{app:LOO-CV}

We optimise the noise variance in each channel $\Sigma$, as well as the kernel hyperparameters $\bm{\theta}$ through the w-LOO-CV approach, as discussed in \Cref{sec:parameter_optimisation}. Specifically, we optimise
\begin{align*}
    \varphi_w(\Sigma, \bm{\theta}) = \sum^N_{i=1}\sum^{T}_{t=1} \left(\frac{w_{i,t}}{\beta_t}\right)^2 \log p_w(y_{i,t}|Y_{-(i,t)}, \Sigma, \bm{\theta}).
\end{align*}
A naive implementation would involve fitting the model $NT$ times, resulting in a computational complexity of $\mathcal{O}(N^4T^4)$. However, the LOO objective allows us to derive an analytical formulation for $p_w(y_{i,t}|Y_{-(i,t)}, \Sigma, \bm{\theta})$ which reduces the computational cost to $\mathcal{O}(N^3T^3)$ by only computing $(\mathbf{K} + \mathbf{\Sigma}\mathbf{J}_{\scriptscriptstyle\mathbf{W}})^{-1}$ once (for each unique set of parameters). This matches the computational complexity of evaluating the marginal likelihood \citep{altamirano2024robustGP, sundararajan1999predictive}. 

\begin{lemma}
Let $\textup{vec}(\mathcal{E}) \sim \mathcal{N}(0, \mathbf{\Sigma})$ where $\mathbf{\Sigma} = \Sigma \otimes I_N$ for a $T$-dimensional diagonal $\Sigma$, and $W_i = \textup{diag}(w_1(\mathbf{x}_i, \mathbf{y}_i), \dots, w_T(\mathbf{x}_i, \mathbf{y}_i))$, where $y_{i,t} \mapsto w_t(\mathbf{x}_i, \mathbf{y}_i)$ is continuously differentiable. The closed-form LOO predictive likelihood for observation $y_{i,t}$ is given as:
\begin{align*}
    p_w(y_{i,t}|Y_{-(i,t)}, \bm{\theta}, \Sigma) &= \mathcal{N}(\mu^{\scriptscriptstyle\textup{RMOGP}}_{i,t}, (\sigma^{\scriptscriptstyle\textup{RMOGP}}_{i,t})^2 + \sigma_t^2)\\
    \mu^{\scriptscriptstyle\textup{RMOGP}}_{i,t} &= z_{i,t} + m_{i,t} - \frac{[(\mathbf{K} + \mathbf{\Sigma}\mathbf{J}_{\scriptscriptstyle\mathbf{W}})^{-1}\mathbf{z}_{\textup{vec}}]_{(t-1)N+i}}{[(\mathbf{K} + \mathbf{\Sigma}\mathbf{J}_{\scriptscriptstyle\mathbf{W}})^{-1}]_{(t-1)N+i, (t-1)N+i}}\\
    (\sigma^{\scriptscriptstyle\textup{RMOGP}}_{i,t})^2 &= \frac{1}{[(\mathbf{K} + \mathbf{\Sigma}\mathbf{J}_{\scriptscriptstyle\mathbf{W}})^{-1}]_{(t-1)N+i, (t-1)N+i}} - \frac{\sigma_t^4}{2}w_t(\mathbf{x}_i, \mathbf{y}_i)^{-2},
\end{align*}
where $\mathbf{z}_{\textup{vec}} = \mathbf{y}_{\textup{vec}} - \mathbf{m}_{\mathbf{W},\textup{vec}}$.
\end{lemma}
\begin{proof}
Without loss of generality, we derive the predictive for $i=N$ and $t=T$, and this can be extended for arbitrary $i \in \{1, \dots, N\}$ and $t \in \{1, \dots, T\}$ using a permutation matrix. Let $p_w(y_{\scriptscriptstyle N,T}) = \mathcal{N}(\mu_{\scriptscriptstyle N,T}^{\scriptscriptstyle\textup{RMOGP}}, (\sigma_{\scriptscriptstyle N,T}^{\scriptscriptstyle \textup{RMOGP}})^2 + \sigma_{\scriptscriptstyle T}^2)$ be the predictive for observation $y_{\scriptscriptstyle N,T}$ given all other observations with
\begin{align*}
    \mu_{\scriptscriptstyle N,T}^{\scriptscriptstyle\textup{RMOGP}} &= m_{T}(\mathbf{x}_{N}) + \mathbf{K}_{1:NT-1,NT}^\top \big[\mathbf{K} + \mathbf{\Sigma} \mathbf{J}_{\scriptscriptstyle\mathbf{W}}\big]^{-1}_{1:NT-1,1:NT-1} \mathbf{z}_{1:NT-1}\\
    (\sigma_{\scriptscriptstyle N,T}^{\scriptscriptstyle \textup{RMOGP}})^2 &= \mathbf{K}_{(NT,NT)} - \mathbf{K}_{1:NT-1,NT}^\top \big[\mathbf{K} + \mathbf{\Sigma} \mathbf{J}_{\scriptscriptstyle\mathbf{W}}\big]^{-1}_{1:NT-1,1:NT-1} \mathbf{K}_{1:NT-1,NT},
\end{align*}
where $[\mathbf{K} + \mathbf{\Sigma} \mathbf{J}_{\scriptscriptstyle\mathbf{W}}]^{-1}_{1:NT-1,1:NT-1}$ represents the submatrix formed from the rows $\{1, \dots, NT-1\}$ and columns $\{1, \dots, NT-1\}$ of $\mathbf{K} + \mathbf{\Sigma} \mathbf{J}_{\scriptscriptstyle\mathbf{W}}$, and $\mathbf{K}_{1:NT-1,NT}$ denotes the submatrix formed from the $NT$-th column and the rows $\{1, \dots, NT-1\}$ of $\mathbf{K}$. Finally, $\mathbf{K}_{(NT,NT)} = \mathcal{K}_{TT}(\mathbf{x}_N, \mathbf{x}_N)$ represents the element from the $NT$-th row and $NT$-th column of $\mathbf{K}$. Through this notation, we can write $\mathbf{K} + \mathbf{\Sigma} \mathbf{J}_{\scriptscriptstyle\mathbf{W}}$ as block matrix
\begin{align*}
\mathbf{K} + \mathbf{\Sigma} \mathbf{J}_{\scriptscriptstyle\mathbf{W}} =
\begin{bmatrix}
    \big[\mathbf{K} + \mathbf{\Sigma} \mathbf{J}_{\scriptscriptstyle\mathbf{W}}\big]_{1:NT-1,1:NT-1} & \mathbf{K}_{1:NT-1,NT}\\
    \mathbf{K}_{1:NT-1,NT}^\top & \mathbf{K}_{(NT,NT)} + \frac{1}{2} \sigma_T^4 w_T(\mathbf{x}_N,  \mathbf{y}_N)^{-2}
\end{bmatrix}
=
\begin{bmatrix}
    A & B\\
    C & D
\end{bmatrix},
\end{align*}
where we use $A, B, C$ and $D$ for clarity in the derivation. Applying block matrix inversion leads to
\begin{align*}
    (\mathbf{K} + \mathbf{\Sigma} \mathbf{J}_{\scriptscriptstyle\mathbf{W}})^{-1} =
    \begin{bmatrix}
    A^{-1} + A^{-1}B(D - CA^{-1}B)^{-1}CA^{-1} & -A^{-1}B(D - CA^{-1}B)^{-1} \\
    -(D - CA^{-1}B)^{-1}CA^{-1} & (D - CA^{-1}B)^{-1}
    \end{bmatrix}
\end{align*}
Therefore, the $(NT,NT)$-th element of this matrix is
\begin{align*}
    \big[&(\mathbf{K} + \mathbf{\Sigma} \mathbf{J}_{\scriptscriptstyle\mathbf{W}})^{-1}\big]_{(NT,NT)} = (D - CA^{-1}B)^{-1} \\
    &= \Big(\mathbf{K}_{(NT,NT)} + \frac{1}{2}\sigma_T^4 w_T(\mathbf{x}_N,  \mathbf{y}_N)^{-2} \\
    &- \mathbf{K}_{1:NT-1,NT}^\top \big[\mathbf{K} + \mathbf{\Sigma} \mathbf{J}_{\scriptscriptstyle\mathbf{W}}\big]_{1:NT-1,1:NT-1}^{-1} \mathbf{K}_{1:NT-1,NT} \Big)^{-1}.
\end{align*}
We can then take the inverse on both sides and rearrange terms to get
\begin{multline*}
    \mathbf{K}_{(NT,NT)} - \mathbf{K}_{1:NT-1,NT}^\top \big[\mathbf{K} + \mathbf{\Sigma} \mathbf{J}_{\scriptscriptstyle\mathbf{W}}\big]_{1:NT-1,1:NT-1}^{-1} \mathbf{K}_{1:NT-1,NT} \\
    = \frac{1}{{\big[(\mathbf{K} + \mathbf{\Sigma} \mathbf{J}_{\scriptscriptstyle\mathbf{W}})^{-1}\big]}_{\scriptscriptstyle(NT,NT)}} - \frac{1}{2}\sigma_T^4 w_T(\mathbf{x}_N, \mathbf{y}_N)^{-2}.
\end{multline*}
This leads to the desired predictive variance
\begin{align*}
    (\sigma^{\scriptscriptstyle\textup{RMOGP}}_{N,T})^2 &= \frac{1}{[(\mathbf{K} + \mathbf{\Sigma}\mathbf{J}_{\scriptscriptstyle\mathbf{W}})^{-1}]_{(NT,NT)}} - \frac{\sigma_T^4}{2}w_T(\mathbf{x}_N, \mathbf{y}_N)^{-2}.
\end{align*}
Similarly, the predictive mean can be obtained by evaluating
\begin{multline*}
    (\mathbf{K} + \mathbf{\Sigma} \mathbf{J}_{\scriptscriptstyle\mathbf{W}})^{-1}\mathbf{z}_\textup{vec} \\=
    \begin{bmatrix}
    A^{-1} + A^{-1}B(D - CA^{-1}B)^{-1}CA^{-1} & -A^{-1}B(D - CA^{-1}B)^{-1} \\
    -(D - CA^{-1}B)^{-1}CA^{-1} & (D - CA^{-1}B)^{-1}
    \end{bmatrix}
    \begin{bmatrix}
        [\mathbf{z}_\textup{vec}]_{1:NT-1}\\
        z_{N,T}
    \end{bmatrix}.
\end{multline*}
This leads to the expression 
\begin{align*}
    [(\mathbf{K} + \mathbf{\Sigma} \mathbf{J}_{\scriptscriptstyle\mathbf{W}})^{-1}\mathbf{z}_\textup{vec}]_{NT} &= -(D - CA^{-1}B)^{-1}CA^{-1} [\mathbf{z}_\textup{vec}]_{1:NT-1} + (D - CA^{-1}B)^{-1} z_{N,T}\\
    &= (D - CA^{-1}B)^{-1} \big(z_{N,T} - CA^{-1} [\mathbf{z}_\textup{vec}]_{1:NT-1}\big).
\end{align*}
We can then substitute $(D - CA^{-1}B)^{-1} = [(\mathbf{K} +\mathbf{\Sigma}\mathbf{J}_{\scriptscriptstyle\mathbf{W}})^{-1}]_{(NT,NT)}$, and replace $A$ and $C$ by their values to get
\begin{multline*}
    [(\mathbf{K} + \mathbf{\Sigma} \mathbf{J}_{\scriptscriptstyle\mathbf{W}})^{-1}\mathbf{z}_\textup{vec}]_{NT} \\= [(\mathbf{K} +\mathbf{\Sigma}\mathbf{J}_{\scriptscriptstyle\mathbf{W}})^{-1}]_{(NT,NT)} \big(z_{N,T} - \mathbf{K}_{1:NT-1,NT}^\top \big[\mathbf{K} + \mathbf{\Sigma} \mathbf{J}_{\scriptscriptstyle\mathbf{W}}\big]^{-1}_{1:NT-1,1:NT-1} [\mathbf{z}_\textup{vec}]_{1:NT-1} \big).
\end{multline*}
Finally, by rearranging terms, we get
\begin{align*}
    \mathbf{K}_{1:NT-1,NT}^\top \big[\mathbf{K} + \mathbf{\Sigma} \mathbf{J}_{\scriptscriptstyle\mathbf{W}}\big]^{-1}_{1:NT-1,1:NT-1} [\mathbf{z}_\textup{vec}]_{1:NT-1} = z_{N,T} - \frac{[(\mathbf{K} + \mathbf{\Sigma} \mathbf{J}_{\scriptscriptstyle\mathbf{W}})^{-1}\mathbf{z}_\textup{vec}]_{NT}}{[(\mathbf{K} +\mathbf{\Sigma}\mathbf{J}_{\scriptscriptstyle\mathbf{W}})^{-1}]_{(NT,NT)}},
\end{align*}
which leads to the desired predictive mean
\begin{align*}
    \mu_{\scriptscriptstyle N,T}^{\scriptscriptstyle\textup{RMOGP}} = z_{N,T} + m_{T}(\mathbf{x}_{N}) - \frac{[(\mathbf{K} + \mathbf{\Sigma} \mathbf{J}_{\scriptscriptstyle\mathbf{W}})^{-1}\mathbf{z}_\textup{vec}]_{NT}}{[(\mathbf{K} +\mathbf{\Sigma}\mathbf{J}_{\scriptscriptstyle\mathbf{W}})^{-1}]_{(NT,NT)}}.
\end{align*}

\end{proof}

\subsection{Proof of Proposition~\ref{prop:robustness}} \label{app:proof_robustness}

We consider the contamination of dataset $\mathcal{D} = \{(\mathbf{x}_i, \mathbf{y}_i)\}^{N}_{i=1}$, where the observation $y_{m,s}$ with index $m \in \{1, \dots N\}$ in the channel $s \in \{1, \dots, T\}$ is replaced by an arbitrary large contaminated value $y_{m,s}^c$. This results in the contaminated dataset $\mathcal{D}^c_{m,s} = (\mathcal{D} \setminus \{(\mathbf{x}_m, \mathbf{y}_m)\}) \cup \{(\mathbf{x}_m, \mathbf{y}^c_{m,s})\}$, where $\mathbf{y}_{m,s}^c = [y_{m,1}, \dots, y_{m,s}^c, \dots, y_{m,T}]^\top$.

\paragraph{PIF for MOGP} 

First, we show that a single contaminated observation can increase the PIF arbitrarily in any output $t$ in the MOGP.

\begin{lemma} \label{lem:PIF_mogp}
Suppose $\mathbf{f} \sim \mathcal{GP}(\mathbf{m}, \mathcal{K})$, where $f_t$ denotes the $t$-th channel of $\mathbf{f}$, and suppose $\textup{vec}(\mathcal{E}) \sim \mathcal{N}(0, \mathbf{\Sigma})$. Then, for a standard MOGP, there exists a positive constant $C_1$ independent of $y_{m,s}^c$ such that
\begin{align*}
     \textup{PIF}_{\textup{MOGP}}(y_{m,s}^c, f_t, \mathcal{D}) = C_1 (y_{m,s}^c - y_{m,s})^2,
 \end{align*}
\end{lemma}
\begin{proof}
Without loss of generality, we prove the bound for $t=1$, and this can be extended for arbitrary $t \in \{1, \dots, T\}$ using a permutation matrix. Let $p(f_1(X) | \mathcal{D})$ and $p(f_1(X)|\mathcal{D}^c_{m,s})$ be the uncontaminated and contaminated marginal MOGP posterior of the first output function, where $f_1(X) = [f_{1,1}, \dots, f_{{\scriptscriptstyle N}, 1}]^\top$. This marginal posterior given the uncontaminated data follows the distribution
\begin{align*}
	p(f_1(X) | \mathcal{D}) &= \mathcal{N}(f_1(X); \mu_1, \Sigma_{11}),\\
	\mu_1 &= [\mathbf{m}_\textrm{vec}]_{\scriptscriptstyle 1:N} + [\mathbf{K}(\mathbf{K} + \mathbf{\Sigma})^{-1}(\mathbf{y}_\textrm{vec} - \mathbf{m}_\textrm{vec})]_{\scriptscriptstyle 1:N}, \\
	\Sigma_{11} &= [\mathbf{K}(\mathbf{K} + \mathbf{\Sigma})^{-1}\mathbf{\Sigma}]_{\scriptscriptstyle 1:N, 1:N}.
\end{align*}
Similarly, the marginal posterior of the contaminated data follows the distribution
\begin{align*}
	p(f_1(X)|\mathcal{D}^c_{m,s}) &= \mathcal{N}(f_1(X); \mu_1^c, \Sigma_{11}^c),\\
	\mu_1^c &= [\mathbf{m}_\textrm{vec}]_{\scriptscriptstyle 1:N} + [\mathbf{K}(\mathbf{K} + \mathbf{\Sigma})^{-1}(\mathbf{y}^c_\textrm{vec} - \mathbf{m}_\textrm{vec})]_{\scriptscriptstyle 1:N}, \\
	\Sigma_{11}^c &= \Sigma_{11},
\end{align*}
where $\mathbf{y}^{c}_{\mathrm{vec}}
  = \bigl[\, y_{1,1}, \ldots, y_{m-1,s},\; y^{c}_{m,s},\; y_{m+1,s}, \ldots, y_{N,T} \,\bigr]^\top$. Note that the covariance matrices of both uncontaminated and contaminated posteriors are equal since they only depend on the inputs, and are therefore not affected by the contamination. Since both posteriors are Gaussian distributions, the KL divergence is available in closed form such that
\begin{align*}
    \textup{PIF}_{\textup{MOGP}}(y_{m,s}^c, f_1, \mathcal{D}) \! = \! \frac{1}{2}\!\left[\!\textup{Tr}((\Sigma_{11}^c)^{-1}\Sigma_{11}) \!-\! N \! + \! (\mu_1^c - \mu_1)\!^\top (\Sigma_{11}^c)^{-1}(\mu_1^c - \mu_1) \!+\! \log \left(\!\frac{\textup{det}\Sigma_{11}^c}{\textup{det}\Sigma_{11}}\!\right)\!\right].
\end{align*}
Since $\Sigma_{11}^c = \Sigma_{11}$, it can be derived that
\begin{align*}
	\textup{Tr}((\Sigma_{11}^c)^{-1}\Sigma_{11}) - N &= \textup{Tr}(\Sigma_{11}^{-1}\Sigma_{11}) - N = \textup{Tr}(I_N) - N = N - N = 0, \\
	\log \left(\frac{\textup{det}\Sigma_{11}^c}{\textup{det}\Sigma_{11}}\right) &= \log \left(\frac{\textup{det}\Sigma_{11}}{\textup{det}\Sigma_{11}}\right) = \log(1) = 0.
\end{align*}
Therefore, the PIF is simplified to 
\begin{align*}
	\textup{PIF}_{\textup{MOGP}}(y_{m,s}^c, f_1, \mathcal{D}) = \frac{1}{2}(\mu_1^c - \mu_1)^\top (\Sigma_{11}^c)^{-1}(\mu_1^c - \mu_1).
\end{align*}
We derive the following expression for $\mu_1 - \mu_1^c$ by substituting the expressions above, and cancelling matching positive and negative terms.
\begin{align*}
	&\mu_1 - \mu_1^c \\
    &= [\mathbf{m}_\textrm{vec}]_{\scriptscriptstyle 1:N} + [\mathbf{K}(\mathbf{K} + \mathbf{\Sigma})^{-1}(\mathbf{y}_\textrm{vec} - \mathbf{m}_\textrm{vec})]_{\scriptscriptstyle 1:N} - [\mathbf{m}_\textrm{vec}]_{\scriptscriptstyle 1:N} - [\mathbf{K}(\mathbf{K} + \mathbf{\Sigma})^{-1}(\mathbf{y}^c_\textrm{vec} - \mathbf{m}_\textrm{vec})]_{\scriptscriptstyle 1:N} \\
	&= [\mathbf{K}]_{\scriptscriptstyle 1:N, 1:NT}(\mathbf{K} + \mathbf{\Sigma})^{-1}(\mathbf{y}_\textrm{vec} - \mathbf{m}_\textrm{vec}) - [\mathbf{K}]_{ \scriptscriptstyle1:N, 1:NT}(\mathbf{K} + \mathbf{\Sigma})^{-1}(\mathbf{y}^c_\textrm{vec} - \mathbf{m}_\textrm{vec}) \\
	&= [\mathbf{K}]_{\scriptscriptstyle 1:N, 1:NT}(\mathbf{K} + \mathbf{\Sigma})^{-1}(\mathbf{y}_\textrm{vec} - \mathbf{y}^c_\textrm{vec}) \\
	&= [\mathbf{K}]_{\scriptscriptstyle 1:N, 1:NT}\Omega(\mathbf{y}_\textrm{vec} - \mathbf{y}^c_\textrm{vec}),
\end{align*}
where we define $\Omega := (\mathbf{K} + \mathbf{\Sigma})^{-1}$ not depending on $\mathbf{y}$. Substituting the previously derived expression for $\mu_1 - \mu_1^c$ and using $(\Sigma_{11}^c)^{-1} = [\mathbf{K}\Omega\mathbf{\Sigma}]_{\scriptscriptstyle 1:N, 1:N}^{-1}$ yields 
\begin{align*}
	\textup{PIF}_{\textup{MOGP}}(y_{m,s}^c, f_1, \mathcal{D}) &= \frac{1}{2}([\mathbf{K}]_{\scriptscriptstyle 1:N, 1:NT}\Omega(\mathbf{y}_\textrm{vec} - \mathbf{y}^c_\textrm{vec}))^\top [\mathbf{K}\Omega\mathbf{\Sigma}]_{\scriptscriptstyle 1:N, 1:N}^{-1}([\mathbf{K}]_{\scriptscriptstyle 1:N, 1:NT}\Omega(\mathbf{y}_\textrm{vec} - \mathbf{y}^c_\textrm{vec})) \\
	&= \frac{1}{2}(\mathbf{y}_\textrm{vec} - \mathbf{y}^c_\textrm{vec})^\top \Omega [\mathbf{K}]_{\scriptscriptstyle 1:NT, 1:N} [\mathbf{K}\Omega\mathbf{\Sigma}]_{\scriptscriptstyle 1:N, 1:N}^{-1}[\mathbf{K}]_{\scriptscriptstyle 1:N, 1:NT}\Omega(\mathbf{y}_\textrm{vec} - \mathbf{y}^c_\textrm{vec}) \\
	&= \frac{1}{2}(\mathbf{y}_\textrm{vec} - \mathbf{y}^c_\textrm{vec})^\top \Omega [\mathbf{K}]_{\scriptscriptstyle 1:NT, 1:N} \Psi [\mathbf{K}]_{\scriptscriptstyle 1:N, 1:NT}\Omega(\mathbf{y}_\textrm{vec} - \mathbf{y}^c_\textrm{vec}),
\end{align*}
where in the second equality we expand the transpose, and in the third equality we define $\Psi = [\mathbf{K}\Omega\mathbf{\Sigma}]_{\scriptscriptstyle 1:N, 1:N}^{-1}$, which is independent of the contamination. Since $\mathbf{y}_\textup{vec}$ and $\mathbf{y}^c_\textup{vec}$ are equal except for the $((s-1)N + m)$-th element, we obtain the desired expression
\begin{align*}
	\textup{PIF}_{\textup{MOGP}}(y_{m,s}^c, f_1, \mathcal{D}) &= \frac{1}{2}\Bigl[\Omega [\mathbf{K}]_{\scriptscriptstyle 1:NT, 1:N} \Psi [\mathbf{K}]_{\scriptscriptstyle 1:N, 1:NT}\Omega\Bigr]_{((s-1)N + m), ((s-1)N + m)}(y_{m,s} - y_{m,s}^c)^2 \\
	&= C_1 (y_{m,s} - y_{m,s}^c)^2,
\end{align*}
where $C_1 = \frac{1}{2}[\Omega [\mathbf{K}]_{\scriptscriptstyle 1:NT, 1:N} \Psi [\mathbf{K}]_{\scriptscriptstyle 1:N, 1:NT}\Omega]_{((s-1)N + m), ((s-1)N + m)}$ does not depend on $y_{m,s}^c$.
\end{proof}
Consequently, MOGPs are not robust since $\textup{PIF}_{\textup{MOGP}}(y_{m,s}^c, f_t, \mathcal{D}) \rightarrow \infty$ as $|y_{m,s}^c - y_{m,s}| \rightarrow \infty$.
\paragraph{PIF for Independent Outputs in the MOGP}
We now consider the case where $s \neq t$, and the covariance matrix has a block-diagonal structure such that $s$ and $t$ belong to different blocks. In other words, the outputs $s$ and $t$ are independent of each other, and contamination in $s$ therefore has no effect on the posterior influence function for $t$. More formally, let the set of outputs be partitioned into $Q$ disjoint index blocks $\{G_1, \dots, G_Q\}$ with $\bigcup_{q=1}^Q G_q=\{1,\dots,T\}$. We say the MOGP output covariance matrix is \textit{block-diagonal} w.r.t. this partition if for any $a \neq b$ and any $t \in G_a, s\in G_b$ the covariance matrix satisfies
\begin{align*}
    \mathcal{K}_{s,t}(\mathbf{x},\mathbf{x}) = \textup{Cov}(f_s(\mathbf{x}), f_t(\mathbf{x})) = 0
\end{align*}

\begin{lemma}
Let the MOGP be as in Lemma~\ref{lem:PIF_mogp}. Assume the prior/joint covariance has a block-diagonal structure with respect to the partition $G_1, \dots, G_Q$ defined above. If $s \neq t$ and $s \in G_a, t \in G_b$ for some $a \neq b$, it follows that $C_1 = 0$, and consequently
\begin{align*}
    \textup{PIF}_{\textup{MOGP}}(y_{m,s}^c, f_t, \mathcal{D}) = 0.
\end{align*}
\end{lemma}
\begin{proof}
Again, without loss of generality, we prove the result for $t=1$. Additionally, assume that the first block---which includes output $t=1$---contains $R < s$ outputs, while the remaining $T-R$ outputs belong to the second (and final) block. We can then express the matrix $\mathbf{K} + \mathbf{\Sigma}$ as
\begin{align*}
    \mathbf{K} + \mathbf{\Sigma} = \!\!
    \begin{bmatrix}
        [\mathbf{K}+\mathbf{\Sigma}]_{\scriptscriptstyle 1:NR, 1:NR} & [\mathbf{K}]_{\scriptscriptstyle 1:NR, NR+1:NT}\\
        [\mathbf{K}]_{\scriptscriptstyle NR+1:NT, 1:NR} & [\mathbf{K}+\mathbf{\Sigma}]_{\scriptscriptstyle NR+1:NT, NR+1:NT} \!
    \end{bmatrix} \!\! = \!\!
    \begin{bmatrix}
        [\mathbf{K}+\mathbf{\Sigma}]_{\scriptscriptstyle 1:NR, 1:NR} & \mathbf{0}\\
        \hspace{-0.7cm}\mathbf{0} & \hspace{-1.22cm}[\mathbf{K}+\mathbf{\Sigma}]_{\scriptscriptstyle NR+1:NT, NR+1:NT} \!
    \end{bmatrix}
\end{align*}
$[\mathbf{\Sigma}]_{\scriptscriptstyle NR+1:NT, 1:NR} = \mathbf{0}$ since $\mathbf{\Sigma}$ is a diagonal matrix, and $[\mathbf{K}]_{\scriptscriptstyle NR+1:NT, 1:NR} = \mathbf{0}$ since $\mathbf{K}$ is block-diagonal. It is notationally cumbersome but straightforward to show, using block matrix inversion, that
\begin{align*}
    \Omega := (\mathbf{K} + \mathbf{\Sigma})^{-1} =
    \begin{bmatrix}
        \Omega_{ \scriptscriptstyle 1:NR, 1:NR} & \mathbf{0} \\
        \mathbf{0} & \Omega_{\scriptscriptstyle NR+1:NT, NR+1:NT}
    \end{bmatrix}
\end{align*}
for some block matrices $\Omega_{ \scriptscriptstyle 1:NR, 1:NR} \in \mathbb{R}^{\scriptscriptstyle NR \times NR}$ and $\Omega_{\scriptscriptstyle NR+1:NT, NR+1:NT} \in \mathbb{R}^{\scriptscriptstyle N(T-R) \times N(T-R)}$. Additionally, the block matrix $[\mathbf{K}]_{\scriptscriptstyle 1:N, 1:NT}$ can be written such that
\begin{align*}
    [\mathbf{K}]_{\scriptscriptstyle 1:N, 1:NT} = 
    \begin{bmatrix}
        [\mathbf{K}]_{\scriptscriptstyle 1:N, 1:NR} & [\mathbf{K}]_{\scriptscriptstyle 1:N, NR+1:NT}
    \end{bmatrix} = 
    \begin{bmatrix}
        [\mathbf{K}]_{\scriptscriptstyle 1:N, 1:NR} & \mathbf{0}
    \end{bmatrix},
\end{align*}
where $[\mathbf{K}]_{\scriptscriptstyle 1:N, NR+1:NT} = \mathbf{0}$ since $\mathbf{K}$ is block-diagonal. Multiplying $[\mathbf{K}]_{\scriptscriptstyle 1:N, 1:NT}$ with $\Omega$, it can be shown that
\begin{align*}
    [\mathbf{K}]_{\scriptscriptstyle 1:N, 1:NT} \Omega = 
    \begin{bmatrix}
        [\mathbf{K}]_{\scriptscriptstyle 1:N, 1:NR} \Omega_{ \scriptscriptstyle 1:NR, 1:NR} & \mathbf{0}
    \end{bmatrix}.
\end{align*}
This results in the following expression:
\begin{align*}
    \Omega [\mathbf{K}]_{\scriptscriptstyle 1:NT, 1:N} \Psi [\mathbf{K}]_{\scriptscriptstyle 1:N, 1:NT}\Omega = 
    \begin{bmatrix}
        \Omega_{ \scriptscriptstyle 1:NR, 1:NR}[\mathbf{K}]_{\scriptscriptstyle 1:NR, 1:N} \Psi [\mathbf{K}]_{\scriptscriptstyle 1:N, 1:NR} \Omega_{ \scriptscriptstyle 1:NR, 1:NR} & \mathbf{0}\\
        \mathbf{0} & \mathbf{0}
    \end{bmatrix}.
\end{align*}
We now note that, since $R < s$, the $((s-1)N + m)$-th diagonal element of this matrix equals zero, i.e. $C_1 = 0$. Therefore, $\textup{PIF}_{\textup{MOGP}}(y_{m,s}^c, f_1, \mathcal{D}) = 0$.
\end{proof}

\textbf{PIF for R-MOGP}

\textbf{Proposition~\ref{prop:robustness}} \textit{Suppose $\mathbf{f} \sim \mathcal{GP}(\mathbf{m}, \mathcal{K})$, $\textup{vec}(\mathcal{E}) \sim \mathcal{N}(0, \mathbf{\Sigma})$, and $W_i = \textup{diag}(w_1(\mathbf{x}_i, \mathbf{y}_i), \dots, w_T(\mathbf{x}_i, \mathbf{y}_i))$, where $w_{t}$ is defined as in \eqref{eq:weight}. Then, there exists a  constant $C_2>0$ independent of $y_{m,s}^c$ for which}
    \begin{align*}
        \max_{1\leq t \leq T}\sup_{y_{m,s}^c \in \mathbb{R}}
        \textup{PIF}(y_{m,s}^c, f_t, \mathcal{D}) \leq C_2.
    \end{align*}
\begin{proof}
To prove robustness of R-MOGPs, we make use of Proposition 3.2 in \citet{altamirano2024robustGP}, which states that the PIF over the \textit{Gaussian} posterior over \textit{all observations} is bounded under mild conditions for the weights. In our multi-output setting, we vectorise the outputs across all channels into a single observations vector. As a result, we can adapt this proposition to the multi-output setting:

\begin{proposition}[\citealp{altamirano2024robustGP}]
Suppose $\mathbf{f} \sim \mathcal{GP}(\mathbf{m}, \mathcal{K})$ and $\textup{vec}(\mathcal{E}) \sim \mathcal{N}(\mathbf{0}, \mathbf{\Sigma})$, and assume that for all $t \in \{1, \ldots, T\}$, $\sup_{\mathbf{x}, \mathbf{y}} w_t(\mathbf{x}, \mathbf{y}) < \infty$ and $\sup_{\mathbf{x}, \mathbf{y}} |y_t \cdot w_t(\mathbf{x}, \mathbf{y})^2| < \infty$. Then there exists a constant $C_2 > 0$ independent of $y_{m,s}^c$ such that the PIF for the \textbf{joint distribution} of R-MOGPs satisfies
\begin{align*}
        \textup{PIF}_{\textup{RMOGP}}(y_{m,s}^c, \mathcal{D}) \leq C_2.
\end{align*}
\end{proposition}

To demonstrate the robustness of the \textit{marginal distributions} of R-MOGPs, it remains to show that the \textit{marginal} PIFs are bounded and that the assumptions on the weighting functions are satisfied. To show the former, we make use of the chain rule for the KL divergence \citep{cover1999elements}. This property states that the divergence between joint distributions is at least as large as the divergence between their marginal distributions. Formally, let $p(x,y)$ and $q(x,y)$ be joint probability densities over random variables $\tilde{x}$ and $\tilde{y}$. Then the KL divergence between these joint distributions decomposes as:
\begin{align*}
    \textup{KL}(p(x,y) \| q(x,y)) &= \textup{KL}(p(x) \| q(x)) + \mathbb{E}_{\tilde{x}\sim p(\tilde{x})}\left[\textup{KL}(p(y|\tilde{x}) \,\|\, q(y|\tilde{x}))\right] \\
    &\geq \textup{KL}(p(x) \| q(x)),
\end{align*}
since the divergence term is non-negative, i.e. $\textup{KL}(p(y|\tilde{x}) \,\|\, q(y|\tilde{x})) \geq 0$. In other words, the KL divergence of the marginal distributions is less than or equal to the KL divergence of the joint distributions. Applying this result to R-MOGPs yields
\begin{align*}
    \textup{PIF}_{\textup{RMOGP}}(y_{m,s}^c, f_t, \mathcal{D}) \leq \textup{PIF}_{\textup{RMOGP}}(y_{m,s}^c, \mathcal{D}) \leq C_2.
\end{align*}
Thus, the marginal PIFs in R-MOGPs are bounded, which ensures their robustness. We now verify that the assumptions on the weights are satisfied. Specifically, we show that
\begin{align*}
    \sup_{\mathbf{x}, \mathbf{y}} w_t(\mathbf{x}, \mathbf{y}) < \infty \qquad \textup{and} \qquad \sup_{\mathbf{x}, \mathbf{y}} |y_t \cdot w_t(\mathbf{x}, \mathbf{y})^2| < \infty.
\end{align*}
The first condition holds trivially since $w_t(\mathbf{x}, \mathbf{y}) \leq \beta_t = \frac{\sigma_t}{\sqrt{2}} < \infty$. To establish the second condition, we refer to the result in Appendix B.3 of \citet{laplante2025robust} adapted to the multi-output setting, which states that $\sup_{\mathbf{x}, \mathbf{y}} |y_t \cdot w_t(\mathbf{x}, \mathbf{y})^2| < \infty$ provided that $|\gamma_t(\mathbf{x}, \mathbf{y})| < \infty$ and $|c_t(\mathbf{x}, \mathbf{y})| < \infty$. Since $\gamma_t(\mathbf{x}, \mathbf{y})$ represents the mean of a Gaussian distribution, it is bounded. Finally, $c_t(\mathbf{x}, \mathbf{y})$ is also bounded since it is the quantile absolute deviation of a set of finite elements. Thus, the second condition also holds.
\end{proof}

\subsection{Sparse Variational R-MOGPs} \label{app:sparse-variational}

Let the MOGP be modelled using the ICM with $R$ underlying latent functions such that $B = AA^\top$ for $A \in \mathbb{R}^{T \times R}$, and each latent function is approximated independently with a set of $Q$ inducing points per latent function. We then define the matrix $U \in \mathbb{R}^{Q \times R}$ for the function evaluation over all inducing points with corresponding inducing location matrix for each latent function $Z_r \in \mathbb{R}^{Q \times d}$ for $r \in \{1, \dots, R\}$. The sparse variational R-MOGP predictive is then given by:
\begin{align*}
    &p_W(f_\star | \mathbf{x}_\star, Y, X) = \mathcal{N}(f_\star; \bm{\mu}_\star^S, \bm{\Sigma}_\star^S), \\
    &\bm{\mu}_\star^S = \mathbf{m}_\star + \mathbf{k}_{U\star}^\top \mathbf{P}_{U}^{-1}\mathbf{K}_{FU}^\top \bm{\nu} \\
    &\bm{\Sigma}_\star^S = \mathbf{k}_{\star\star} - \mathbf{k}_{U\star}^\top \mathbf{K}_{UU}^{-1}\mathbf{k}_{U\star} + \mathbf{k}_{U\star}^\top \mathbf{P}_{U}^{-1}\mathbf{k}_{U\star},
\end{align*}
where $\bm{\nu} = \mathbf{\Sigma}^{-1}\mathbf{J}_{\scriptscriptstyle\mathbf{W}}^{-1} (\mathbf{y}_{\text{vec}} - \mathbf{m}_{{\scriptscriptstyle\mathbf{W}},\text{vec}})$, $\mathbf{P}_U = \mathbf{K}_{UU} + \mathbf{K}_{FU}^\top \mathbf{\Sigma}^{-1} \mathbf{J}_{\scriptscriptstyle\mathbf{W}}^{-1}\mathbf{K}_{FU}$, $\mathbf{K}_{UU} = \textup{diag}(\kappa(Z_1, Z_1), \dots, \kappa(Z_R, Z_R)) \in \mathbb{R}^{QR \times QR}$ is a diagonal block matrix, and $\mathbf{K}_{FU} \in \mathbb{R}^{NT \times QR}$ is a block-matrix such that the $t,r$-th block of $\mathbf{K}_{FU}$ equals $a_{t,r} \kappa(X, Z_r)$ for $t \in \{1, \dots, T\}$ and $r \in \{1, \dots, R\}$. In addition, $\mathbf{k}_{U\star} \in \mathbb{R}^{QR \times T}$ is the block-matrix such that $ [\mathbf{k}_{U\star}]_{r,t} = a_{t,r} \kappa(Z_r, \mathbf{x}_\star)$. The corresponding ELBO is given as:
\begin{align*}
    \textup{ELBO}(Z, \Sigma, \bm{\theta}) = - \frac{1}{2} &\biggl(\textup{Tr}\left(\mathbf{\Sigma}^{-1/2}\mathbf{J}_{\scriptscriptstyle\mathbf{W}}^{-1/2} (\mathbf{K} - \mathbf{K}_{FU} \mathbf{K}_{UU}^{-1} \mathbf{K}_{FU}^\top) \mathbf{J}_{\scriptscriptstyle\mathbf{W}}^{-1/2}\mathbf{\Sigma}^{-1/2}\right) + C(\Sigma, \bm{\theta})\biggr) \\
    &+ \frac{1}{2} \bm{\nu}^\top \mathbf{K}_{FU} \mathbf{P}_U^{-1}\mathbf{K}_{FU}^\top \bm{\nu} + \frac{1}{2} \log \left(\frac{\textup{det}(\mathbf{K}_{UU})}{\textup{det}(\mathbf{P}_U)}\right).
\end{align*}
where $C(\Sigma, \bm{\theta})$ is a function that depends only on $\Sigma$ and $\bm{\theta}$.

\begin{proof}

Without loss of generality, we derive the predictive for $\mathbf{m}_\textup{vec} = 0$, and then extend it to an arbitrary prior mean. The ICM allows us to model the output as a linear combination of $R$ latent functions $v_r(\mathbf{x})$ for $r \in \{1, \dots, R\}$ such that
\begin{align*}
    f_t(\mathbf{x}) = a_{t1} v_1(\mathbf{x}) + \dots + a_{tR}v_R(\mathbf{x}),
\end{align*}
where each latent function is modelled as an independent GP such that $v_r \sim \mathcal{GP}(0, \kappa)$. The inducing point evaluations in $U$ are then given as $[U]_{qr} = v_r(\mathbf{z}_q^{(r)})$, where $\mathbf{z}_q^{(r)}$ is an inducing location such that $Z_r = [z_1^{(r)}, \dots, z_Q^{(r)}]^\top$. We define $\mathbf{u}_\textup{vec} = \textup{vec}(U) \in \mathbb{R}^{RQ}$ by stacking its columns. The variational distribution is then given as:
\begin{align*}
    q(\mathbf{f}_\textup{vec}, \mathbf{u}_\textup{vec}) = p(\mathbf{f}_\textup{vec} | \mathbf{u}_\textup{vec}) q(\mathbf{u}_\textup{vec}),
\end{align*}
where $q(\mathbf{u}_\textup{vec}) = \mathcal{N}(\mathbf{u}_\textup{vec}; \bm{\mu}_U, \mathbf{\Sigma}_U)$. We now seek to approximate the exact posterior $p(\mathbf{f}_\textup{vec}, \mathbf{u}_\textup{vec} | \mathbf{y}_\textup{vec})$ by this variational distribution by minimising the KL divergence between $q(\mathbf{f}_\textup{vec}, \mathbf{u}_\textup{vec})$ and $p(\mathbf{f}_\textup{vec}, \mathbf{u}_\textup{vec} | \mathbf{y}_\textup{vec})$. This is equivalent to maximising ELBO:
\begin{align*}
    \textup{ELBO}(Z, \Sigma, \bm{\theta}) = \int \log \frac{\Psi(\mathbf{y}_\textup{vec}, \mathbf{u}_\textup{vec}) p(\mathbf{u}_\textup{vec})}{q(\mathbf{u}_\textup{vec})}q(\mathbf{u}_\textup{vec}) d\mathbf{u}_\textup{vec},
\end{align*}
where
\begin{align*}
    \Psi(\mathbf{y}_\textup{vec}, \mathbf{u}_\textup{vec}) = \exp{\left(\int \log (p(\mathbf{y}_\textup{vec} | \mathbf{f}_\textup{vec})) p(\mathbf{f}_\textup{vec} | \mathbf{u}_\textup{vec}) d\mathbf{f}_\textup{vec}\right)}.
\end{align*}
It is straightforward to show that the optimal variational distribution is given by $q(\mathbf{u}_\textup{vec}) \propto \Psi(\mathbf{y}_\textup{vec}, \mathbf{u}_\textup{vec}) p(\mathbf{u}_\textup{vec})$ such that
\begin{align*}
    q(\mathbf{u}_\textup{vec}) = \frac{\Psi(\mathbf{y}_\textup{vec}, \mathbf{u}_\textup{vec}) p(\mathbf{u}_\textup{vec})}{\int \Psi(\mathbf{y}_\textup{vec}, \mathbf{u}_\textup{vec}) p(\mathbf{u}_\textup{vec}) d\mathbf{u}_\textup{vec}}.
\end{align*}
This results in the following ELBO:
\begin{align*}
    \textup{ELBO}(Z, \Sigma, \bm{\theta}) = \log \left(\int \Psi(\mathbf{y}_\textup{vec}, \mathbf{u}_\textup{vec}) p(\mathbf{u}_\textup{vec}) d\mathbf{u}_\textup{vec}\right).
\end{align*}
For the sparse variational R-MOGP, we replace the standard likelihood $p(\mathbf{y}_\textup{vec} | \mathbf{f}_\textup{vec})$ with $\exp(-N \mathcal{L}_W(F, Y, X))/((2\pi)^{NT/2} \textup{det}(\mathbf{\Sigma})^{1/2})$ in $\Psi(\mathbf{y}_\textup{vec}, \mathbf{u}_\textup{vec})$. We divide by the constant $(2\pi)^{NT/2} \textup{det}(\mathbf{\Sigma})^{1/2}$ as it is independent of $\mathbf{y}_\textup{vec}$ and is used to regularise optimisation with respect to $\Sigma$.
This results in the following optimal robust variational distribution:
\begin{align*}
    q_W(\mathbf{u}_\textup{vec}) = \frac{\Psi_W(\mathbf{y}_\textup{vec}, \mathbf{u}_\textup{vec}) p(\mathbf{u}_\textup{vec})}{\int \Psi_W(\mathbf{y}_\textup{vec}, \mathbf{u}_\textup{vec}) p(\mathbf{u}_\textup{vec}) d\mathbf{u}_\textup{vec}},
\end{align*}
where
\begin{align*}
    \log \Psi_W(\mathbf{y}_\textup{vec}, \mathbf{u}_\textup{vec}) &= \int \log \left(\exp(-N \mathcal{L}_W(F, Y, X))\right) p(\mathbf{f}_\textup{vec} | \mathbf{u}_\textup{vec}) d\mathbf{f}_\textup{vec} \\
    &= \int \log \left(\exp\left(- \frac{1}{2} \! \left[\mathbf{f}_{\text{vec}}^\top \mathbf{\Sigma}^{-1}\mathbf{J}_{\scriptscriptstyle\mathbf{W}}^{-1}\mathbf{f}_{\text{vec}} - 2 \mathbf{f}_{\text{vec}}^\top \bm{\nu} + C(\Sigma, \bm{\theta})\right]\right)\right) p(\mathbf{f}_\textup{vec} | \mathbf{u}_\textup{vec}) d\mathbf{f}_\textup{vec} \\
    &= \int - \frac{1}{2} \left[\mathbf{f}_{\text{vec}}^\top \mathbf{\Sigma}^{-1}\mathbf{J}_{\scriptscriptstyle\mathbf{W}}^{-1}\mathbf{f}_{\text{vec}} - 2 \mathbf{f}_{\text{vec}}^\top \bm{\nu} + C(\Sigma, \bm{\theta})\right]p(\mathbf{f}_\textup{vec} | \mathbf{u}_\textup{vec}) d\mathbf{f}_\textup{vec},
\end{align*}
where $\bm{\nu} = \mathbf{\Sigma}^{-1}\mathbf{J}_{\scriptscriptstyle\mathbf{W}}^{-1} (\mathbf{y}_{\text{vec}} - \mathbf{m}_{{\scriptscriptstyle\mathbf{W}},\text{vec}})$ and $C(\Sigma, \bm{\theta}) = 2\sum_{i=1}^N c_i + NT\log (2 \pi) + \log \textup{det}(\mathbf{\Sigma})$, where $c_i$ is given in Appendix~\ref{app:quadratic_loss}. The conditional distribution of $\mathbf{f}_\textup{vec}$ given $\mathbf{u}_\textup{vec}$ is given as
\begin{align*}
    p(\mathbf{f}_\textup{vec} | \mathbf{u}_\textup{vec}) &= \mathcal{N}(\mathbf{f}_\textup{vec}; \bm{\mu}_{\mathbf{f}_\textup{vec} | \mathbf{u}_\textup{vec}}, \bm{\Sigma}_{\mathbf{f}_\textup{vec} | \mathbf{u}_\textup{vec}}),\\
    \bm{\mu}_{\mathbf{f}_\textup{vec} | \mathbf{u}_\textup{vec}} &= \mathbf{K}_{FU}\mathbf{K}_{UU}^{-1}\mathbf{u}_\textup{vec}, \\
    \bm{\Sigma}_{\mathbf{f}_\textup{vec} | \mathbf{u}_\textup{vec}} &= \mathbf{K} - \mathbf{K}_{FU} \mathbf{K}_{UU}^{-1} \mathbf{K}_{FU}^\top,
\end{align*}
where $\mathbf{K}_{UU} = \textup{diag}(\kappa(Z_1, Z_1), \dots, \kappa(Z_R, Z_R)) \in \mathbb{R}^{QR \times QR}$ is a diagonal block matrix, and $\mathbf{K}_{FU} \in \mathbb{R}^{NT \times QR}$ is a block matrix such that the $t,r$-th block of $\mathbf{K}_{FU}$ equals $a_{t,r} \kappa(X, Z_r)$. We can then write the above expression as:
\begin{align*}
    \log &\Psi_W(\mathbf{y}_\textup{vec}, \mathbf{u}_\textup{vec}) \\
    &= - \frac{1}{2} \int\mathbf{f}_{\text{vec}}^\top \mathbf{\Sigma}^{-1}\mathbf{J}_{\scriptscriptstyle\mathbf{W}}^{-1}\mathbf{f}_{\text{vec}}p(\mathbf{f}_\textup{vec} | \mathbf{u}_\textup{vec}) d\mathbf{f}_\textup{vec} + \int\mathbf{f}_{\text{vec}}^\top \bm{\nu} p(\mathbf{f}_\textup{vec} | \mathbf{u}_\textup{vec}) d\mathbf{f}_\textup{vec} -\frac{1}{2} C(\Sigma, \bm{\theta})\\
    &= - \frac{1}{2} \int\mathbf{f}_{\text{vec}}^\top \mathbf{\Sigma}^{-1}\mathbf{J}_{\scriptscriptstyle\mathbf{W}}^{-1}\mathbf{f}_{\text{vec}}p(\mathbf{f}_\textup{vec} | \mathbf{u}_\textup{vec}) d\mathbf{f}_\textup{vec} + \mathbb{E}_{\mathbf{f}_\textup{vec} \sim p(\mathbf{f}_\textup{vec} | \mathbf{u}_\textup{vec})}\left[\mathbf{f}_\textup{vec}^\top \bm{\nu}\right] -\frac{1}{2} C(\Sigma, \bm{\theta}).
\end{align*}
We now recall that the inner product between two vectors equals the trace of their outer product such that
\begin{align*}
    \mathbf{f}_{\text{vec}}^\top \mathbf{\Sigma}^{-1}\mathbf{J}_{\scriptscriptstyle\mathbf{W}}^{-1}\mathbf{f}_{\text{vec}} = (\mathbf{\Sigma}^{-1/2}\mathbf{J}_{\scriptscriptstyle\mathbf{W}}^{-1/2}
    \mathbf{f}_\textup{vec})^\top (\mathbf{\Sigma}^{-1/2}\mathbf{J}_{\scriptscriptstyle\mathbf{W}}^{-1/2}
    \mathbf{f}_\textup{vec}) = \textup{Tr}(\mathbf{\Sigma}^{-1/2}\mathbf{J}_{\scriptscriptstyle\mathbf{W}}^{-1/2} \mathbf{f}_\textup{vec}\mathbf{f}_\textup{vec}^\top\mathbf{J}_{\scriptscriptstyle\mathbf{W}}^{-1/2}\mathbf{\Sigma}^{-1/2}).
\end{align*}
Plugging this into above expression, we get:
\begin{align*}
    &\log \Psi_W(\mathbf{y}_\textup{vec}, \mathbf{u}_\textup{vec}) \\
    &= \! - \frac{1}{2} \! \int \! \textup{Tr}(\mathbf{\Sigma}^{-1/2}\mathbf{J}_{\scriptscriptstyle\mathbf{W}}^{-1/2} \mathbf{f}_\textup{vec}\mathbf{f}_\textup{vec}^\top\mathbf{J}_{\scriptscriptstyle\mathbf{W}}^{-1/2}\mathbf{\Sigma}^{-1/2})p(\mathbf{f}_\textup{vec} | \mathbf{u}_\textup{vec}) d\mathbf{f}_\textup{vec} \! + \! \mathbb{E}_{\mathbf{f}_\textup{vec} \sim p(\mathbf{f}_\textup{vec} | \mathbf{u}_\textup{vec})}\left[\mathbf{f}_\textup{vec}^\top \bm{\nu}\right] \!-\! \frac{1}{2} C(\Sigma, \bm{\theta}) \\
    &= \! - \frac{1}{2} \mathbb{E}_{\mathbf{f}_\textup{vec} \sim p(\mathbf{f}_\textup{vec} | \mathbf{u}_\textup{vec})} \!\left[\textup{Tr}(\mathbf{\Sigma}^{-1/2}\mathbf{J}_{\scriptscriptstyle\mathbf{W}}^{-1/2} \mathbf{f}_\textup{vec}\mathbf{f}_\textup{vec}^\top\mathbf{J}_{\scriptscriptstyle\mathbf{W}}^{-1/2}\mathbf{\Sigma}^{-1/2})\right] + \mathbb{E}_{\mathbf{f}_\textup{vec} \sim p(\mathbf{f}_\textup{vec} | \mathbf{u}_\textup{vec})}\left[\mathbf{f}_\textup{vec}^\top \bm{\nu}\right] \!-\! \frac{1}{2} C(\Sigma, \bm{\theta}).
\end{align*}
We can use that the expectation of the trace equals the trace of the expectation to get that
\begin{align*}
    &\log \Psi_W(\mathbf{y}_\textup{vec}, \mathbf{u}_\textup{vec})\\
    &= \! - \frac{1}{2} \textup{Tr}\!\left(\mathbb{E}_{\mathbf{f}_\textup{vec} \sim p(\mathbf{f}_\textup{vec} | \mathbf{u}_\textup{vec})} \left[\mathbf{\Sigma}^{-1/2}\mathbf{J}_{\scriptscriptstyle\mathbf{W}}^{-1/2} \mathbf{f}_\textup{vec}\mathbf{f}_\textup{vec}^\top\mathbf{J}_{\scriptscriptstyle\mathbf{W}}^{-1/2}\mathbf{\Sigma}^{-1/2}\right]\!\right) \!+\! \mathbb{E}_{\mathbf{f}_\textup{vec} \sim p(\mathbf{f}_\textup{vec} | \mathbf{u}_\textup{vec})}\left[\mathbf{f}_\textup{vec}^\top \bm{\nu}\right] \!-\! \frac{1}{2} C(\Sigma, \bm{\theta})\\
    &= \! - \frac{1}{2} \textup{Tr}\!\left(\mathbf{\Sigma}^{-1/2}\mathbf{J}_{\scriptscriptstyle\mathbf{W}}^{-1/2} \mathbb{E}_{\mathbf{f}_\textup{vec} \sim p(\mathbf{f}_\textup{vec} | \mathbf{u}_\textup{vec})} \!\left[\mathbf{f}_\textup{vec}\mathbf{f}_\textup{vec}^\top\right]\mathbf{J}_{\scriptscriptstyle\mathbf{W}}^{-1/2}\mathbf{\Sigma}^{-1/2}\!\right) \!+\! \mathbb{E}_{\mathbf{f}_\textup{vec} \sim p(\mathbf{f}_\textup{vec} | \mathbf{u}_\textup{vec})}\left[\mathbf{f}_\textup{vec}^\top \bm{\nu}\right] \!-\! \frac{1}{2} C(\Sigma, \bm{\theta}).
\end{align*}
We now recall that
\begin{align*}
    \mathbb{E}_{\mathbf{f}_\textup{vec} \sim p(\mathbf{f}_\textup{vec} | \mathbf{u}_\textup{vec})} [\mathbf{f}_\textup{vec}\mathbf{f}_\textup{vec}^\top] = \mathbb{V}_{\mathbf{f}_\textup{vec} \sim p(\mathbf{f}_\textup{vec} | \mathbf{u}_\textup{vec})} [\mathbf{f}_\textup{vec}] + \mathbb{E}_{\mathbf{f}_\textup{vec} \sim p(\mathbf{f}_\textup{vec} | \mathbf{u}_\textup{vec})}[\mathbf{f}_\textup{vec}] \mathbb{E}_{\mathbf{f}_\textup{vec} \sim p(\mathbf{f}_\textup{vec} | \mathbf{u}_\textup{vec})}[\mathbf{f}_\textup{vec}]^\top
\end{align*}
to get the following expression:
\begin{align*}
    \log \Psi_W(\mathbf{y}_\textup{vec}, \mathbf{u}_\textup{vec}) \!=\! - \frac{1}{2} \textup{Tr}&\left(\mathbf{\Sigma}^{-1/2}\mathbf{J}_{\scriptscriptstyle\mathbf{W}}^{-1/2} \left(\mathbb{V} [\mathbf{f}_\textup{vec}| \mathbf{u}_\textup{vec}] + \mathbb{E}[\mathbf{f}_\textup{vec}| \mathbf{u}_\textup{vec}] \mathbb{E}[\mathbf{f}_\textup{vec}| \mathbf{u}_\textup{vec}]^\top\right)\mathbf{J}_{\scriptscriptstyle\mathbf{W}}^{-1/2}\mathbf{\Sigma}^{-1/2}\right) \\
    &+ \mathbb{E}\left[\mathbf{f}_\textup{vec}| \mathbf{u}_\textup{vec}\right]^\top \bm{\nu} -\frac{1}{2} C(\Sigma, \bm{\theta}).
\end{align*}
Now note that $\mathbb{V}_{\mathbf{f}_\textup{vec} \sim p(\mathbf{f}_\textup{vec} | \mathbf{u}_\textup{vec})} [\mathbf{f}_\textup{vec}]$ and $\mathbb{E}_{\mathbf{f}_\textup{vec} \sim p(\mathbf{f}_\textup{vec} | \mathbf{u}_\textup{vec})}\left[\mathbf{f}_\textup{vec}\right]$ are known since $p(\mathbf{f}_\textup{vec} | \mathbf{u}_\textup{vec}) = \mathcal{N}(\mathbf{f}_\textup{vec}; \bm{\mu}_{\mathbf{f}_\textup{vec} | \mathbf{u}_\textup{vec}}, \bm{\Sigma}_{\mathbf{f}_\textup{vec} | \mathbf{u}_\textup{vec}})$. We can substitute these terms to get:
\begin{align*}
    &\log \Psi_W(\mathbf{y}_\textup{vec}, \mathbf{u}_\textup{vec}) \\
    &= - \frac{1}{2} \textup{Tr}\left(\mathbf{\Sigma}^{-1/2}\mathbf{J}_{\scriptscriptstyle\mathbf{W}}^{-1/2} \left(\bm{\Sigma}_{\mathbf{f}_\textup{vec} | \mathbf{u}_\textup{vec}} + \bm{\mu}_{\mathbf{f}_\textup{vec} | \mathbf{u}_\textup{vec}} \bm{\mu}_{\mathbf{f}_\textup{vec} | \mathbf{u}_\textup{vec}}^\top\right)\mathbf{J}_{\scriptscriptstyle\mathbf{W}}^{-1/2}\mathbf{\Sigma}^{-1/2}\right) + \bm{\mu}_{\mathbf{f}_\textup{vec} | \mathbf{u}_\textup{vec}}^\top \bm{\nu} -\frac{1}{2} C(\Sigma, \bm{\theta}) \\
    &= - \frac{1}{2} \textup{Tr}\left(\mathbf{\Sigma}^{-1/2}\mathbf{J}_{\scriptscriptstyle\mathbf{W}}^{-1/2} \bm{\Sigma}_{\mathbf{f}_\textup{vec} | \mathbf{u}_\textup{vec}} \mathbf{J}_{\scriptscriptstyle\mathbf{W}}^{-1/2}\mathbf{\Sigma}^{-1/2}\right) \\
    & \qquad \qquad \qquad \qquad - \frac{1}{2} \textup{Tr}\left(\mathbf{\Sigma}^{-1/2}\mathbf{J}_{\scriptscriptstyle\mathbf{W}}^{-1/2} \bm{\mu}_{\mathbf{f}_\textup{vec} | \mathbf{u}_\textup{vec}} \bm{\mu}_{\mathbf{f}_\textup{vec} | \mathbf{u}_\textup{vec}}^\top\mathbf{J}_{\scriptscriptstyle\mathbf{W}}^{-1/2}\mathbf{\Sigma}^{-1/2}\right) + \bm{\mu}_{\mathbf{f}_\textup{vec} | \mathbf{u}_\textup{vec}}^\top \bm{\nu} -\frac{1}{2} C(\Sigma, \bm{\theta}).
\end{align*}
We can then again use that trace of the outer product of a vector equals the inner product such that:
\begin{align*}
    \log \Psi_W(\mathbf{y}_\textup{vec}, \mathbf{u}_\textup{vec}) = - \frac{1}{2} \textup{Tr}&\left(\mathbf{\Sigma}^{-1/2}\mathbf{J}_{\scriptscriptstyle\mathbf{W}}^{-1/2} \bm{\Sigma}_{\mathbf{f}_\textup{vec} | \mathbf{u}_\textup{vec}} \mathbf{J}_{\scriptscriptstyle\mathbf{W}}^{-1/2}\mathbf{\Sigma}^{-1/2}\right) \\
    &- \frac{1}{2} \bm{\mu}_{\mathbf{f}_\textup{vec} | \mathbf{u}_\textup{vec}}^\top \mathbf{J}_{\scriptscriptstyle\mathbf{W}}^{-1} \mathbf{\Sigma}^{-1} \bm{\mu}_{\mathbf{f}_\textup{vec} | \mathbf{u}_\textup{vec}} + \bm{\mu}_{\mathbf{f}_\textup{vec} | \mathbf{u}_\textup{vec}}^\top \bm{\nu} -\frac{1}{2} C(\Sigma, \bm{\theta}).
\end{align*}
To complete the proof, we plug $\Psi_W$ into the optimal variational distribution:
\begin{align*}
    &q_W(\mathbf{u}_\textup{vec}) \\
    &\propto \Psi_W(\mathbf{y}_\textup{vec}, \mathbf{u}_\textup{vec}) p(\mathbf{u}_\textup{vec}) \\
    &\propto \exp \left( - \frac{1}{2} \bm{\mu}_{\mathbf{f}_\textup{vec} | \mathbf{u}_\textup{vec}}^\top \mathbf{J}_{\scriptscriptstyle\mathbf{W}}^{-1} \mathbf{\Sigma}^{-1} \bm{\mu}_{\mathbf{f}_\textup{vec} | \mathbf{u}_\textup{vec}} + \bm{\mu}_{\mathbf{f}_\textup{vec} | \mathbf{u}_\textup{vec}}^\top \bm{\nu} - \frac{1}{2} \mathbf{u}_\textup{vec}^\top \mathbf{K}_{UU}^{-1}\mathbf{u}_\textup{vec} \right) \\
    &= \exp \left( - \frac{1}{2} \mathbf{u}_\textup{vec}^\top \mathbf{K}_{UU}^{-1} \mathbf{K}_{FU}^\top \mathbf{J}_{\scriptscriptstyle\mathbf{W}}^{-1} \mathbf{\Sigma}^{-1} \mathbf{K}_{FU} \mathbf{K}_{UU}^{-1}\mathbf{u}_\textup{vec} + \mathbf{u}_\textup{vec}^\top \mathbf{K}_{UU}^{-1} \mathbf{K}_{FU}^\top \bm{\nu} - \frac{1}{2} \mathbf{u}_\textup{vec}^\top \mathbf{K}_{UU}^{-1}\mathbf{u}_\textup{vec} \right) \\
    &= \exp \left( - \frac{1}{2} \mathbf{u}_\textup{vec}^\top \left( \mathbf{K}_{UU}^{-1} \mathbf{K}_{FU}^\top \mathbf{J}_{\scriptscriptstyle\mathbf{W}}^{-1} \mathbf{\Sigma}^{-1} \mathbf{K}_{FU} \mathbf{K}_{UU}^{-1} + \mathbf{K}_{UU}^{-1}\right) \mathbf{u}_\textup{vec} + \mathbf{u}_\textup{vec}^\top \mathbf{K}_{UU}^{-1} \mathbf{K}_{FU}^\top \bm{\nu} \right).
\end{align*}
By completing squares, we get:
\begin{align*}
    q_W(\mathbf{u}_\textup{vec}) &\propto \exp \left( - \frac{1}{2}(\mathbf{u}_\textup{vec} - \bm{\mu}_U)^\top \mathbf{\Sigma}_U^{-1} (\mathbf{u}_\textup{vec} - \bm{\mu}_U) \right),\\
    \mathbf{\Sigma}_U &= (\mathbf{K}_{UU}^{-1} \mathbf{K}_{FU}^\top \mathbf{J}_{\scriptscriptstyle\mathbf{W}}^{-1} \mathbf{\Sigma}^{-1} \mathbf{K}_{FU} \mathbf{K}_{UU}^{-1} + \mathbf{K}_{UU}^{-1})^{-1} \\
    &= \mathbf{K}_{UU}(\mathbf{K}_{UU} + \mathbf{K}_{FU}^\top\mathbf{\Sigma}^{-1}\mathbf{J}_{\scriptscriptstyle\mathbf{W}}^{-1}\mathbf{K}_{FU})^{-1}\mathbf{K}_{UU}\\
    &= \mathbf{K}_{UU}\mathbf{P}_{U}^{-1}\mathbf{K}_{UU},\\
    \bm{\mu}_U &= \mathbf{\Sigma}_U \mathbf{K}_{UU}^{-1} \mathbf{K}_{FU}^\top \bm{\nu} \\
    &= \mathbf{K}_{UU}\mathbf{P}_{U}^{-1}\mathbf{K}_{UU} \mathbf{K}_{UU}^{-1} \mathbf{K}_{FU}^\top \bm{\nu}\\
    &= \mathbf{K}_{UU}\mathbf{P}_{U}^{-1}\mathbf{K}_{FU}^\top \mathbf{\Sigma}^{-1}\mathbf{J}_{\scriptscriptstyle\mathbf{W}}^{-1} (\mathbf{y}_{\text{vec}} - \mathbf{m}_{{\scriptscriptstyle\mathbf{W}},\text{vec}}),
\end{align*}
where we define $\mathbf{P}_{U} = \mathbf{K}_{UU} + \mathbf{K}_{FU}^\top\mathbf{\Sigma}^{-1}\mathbf{J}_{\scriptscriptstyle\mathbf{W}}^{-1}\mathbf{K}_{FU}$. Finally, we obtain the posterior predictive distribution over $f_\star$ as follows:
\begin{align*}
    p_W(f_\star | \mathbf{x}_\star, Y, X) = \int p_W(f_\star | \mathbf{x}_\star, \mathbf{u}_\textup{vec}) q_W(\mathbf{u}_\textup{vec}) d\mathbf{u}_\textup{vec}.
\end{align*}
Here, we can use the integral obtained in Appendix~\ref{app:proof_MORCGP} to get the following:
\begin{align*}
    p_W(f_\star | \mathbf{x}_\star, Y, X) &= \mathcal{N}(f_\star; \bm{\mu}_\star^S, \bm{\Sigma}_\star^S), \\
    \bm{\mu}_\star^S &= \mathbf{k}_{U\star}^\top \mathbf{K}_{UU}^{-1} \bm{\mu}_U, \\
    &= \mathbf{k}_{U\star}^\top \mathbf{K}_{UU}^{-1} \mathbf{K}_{UU}\mathbf{P}_{U}^{-1}\mathbf{K}_{FU}^\top \bm{\nu} \\
    &= \mathbf{k}_{U\star}^\top \mathbf{P}_{U}^{-1}\mathbf{K}_{FU}^\top \bm{\nu} \\
    \bm{\Sigma}_\star^S &= \mathbf{k}_{\star\star} - \mathbf{k}_{U\star}^\top \mathbf{K}_{UU}^{-1} (\mathbf{K}_{UU} - \mathbf{\Sigma}_U)\mathbf{K}_{UU}^{-1}\mathbf{k}_{U\star}\\
    &= \mathbf{k}_{\star\star} - \mathbf{k}_{U\star}^\top \mathbf{K}_{UU}^{-1} (\mathbf{K}_{UU} - \mathbf{K}_{UU}\mathbf{P}_{U}^{-1}\mathbf{K}_{UU})\mathbf{K}_{UU}^{-1}\mathbf{k}_{U\star} \\
    &= \mathbf{k}_{\star\star} - \mathbf{k}_{U\star}^\top \mathbf{K}_{UU}^{-1}\mathbf{k}_{U\star} + \mathbf{k}_{U\star}^\top \mathbf{P}_{U}^{-1}\mathbf{k}_{U\star}.
\end{align*}
The ELBO is then obtained by substituting $\Psi_W$:
\begin{align*}
    &\textup{ELBO}(Z, \Sigma, \bm{\theta}) \\
    &= \log \biggl(\int \exp \biggl(- \frac{1}{2} \textup{Tr}\left(\mathbf{\Sigma}^{-1/2}\mathbf{J}_{\scriptscriptstyle\mathbf{W}}^{-1/2} \bm{\Sigma}_{\mathbf{f}_\textup{vec} | \mathbf{u}_\textup{vec}} \mathbf{J}_{\scriptscriptstyle\mathbf{W}}^{-1/2}\mathbf{\Sigma}^{-1/2}\right) - \frac{1}{2} \bm{\mu}_{\mathbf{f}_\textup{vec} | \mathbf{u}_\textup{vec}}^\top \mathbf{J}_{\scriptscriptstyle\mathbf{W}}^{-1} \mathbf{\Sigma}^{-1} \bm{\mu}_{\mathbf{f}_\textup{vec} | \mathbf{u}_\textup{vec}} \\
    &+ \bm{\mu}_{\mathbf{f}_\textup{vec} | \mathbf{u}_\textup{vec}}^\top \bm{\nu} -\frac{1}{2} C(\Sigma, \bm{\theta}) - \frac{1}{2} \mathbf{u}_\textup{vec}^\top \mathbf{K}_{UU}^{-1}\mathbf{u}_\textup{vec} \biggr) (2\pi)^{-QR/2} \textup{det}(\mathbf{K}_{UU})^{-1/2}  d\mathbf{u}_\textup{vec}\biggr).
\end{align*}
We can simplify as follows:
\begin{align*}
    &\textup{ELBO}(Z, \Sigma, \bm{\theta}) \\
    &= - \frac{1}{2} \biggl(\textup{Tr}\left(\mathbf{\Sigma}^{-1/2}\mathbf{J}_{\scriptscriptstyle\mathbf{W}}^{-1/2} \bm{\Sigma}_{\mathbf{f}_\textup{vec} | \mathbf{u}_\textup{vec}} \mathbf{J}_{\scriptscriptstyle\mathbf{W}}^{-1/2}\mathbf{\Sigma}^{-1/2}\right) + C(\Sigma, \bm{\theta})\biggr) -\frac{QR}{2} \log(2\pi) - \frac{1}{2} \log (\textup{det}(\mathbf{K}_{UU})) \\
    &+ \log \left(\int \exp \left( - \frac{1}{2} \bm{\mu}_{\mathbf{f}_\textup{vec} | \mathbf{u}_\textup{vec}}^\top \mathbf{J}_{\scriptscriptstyle\mathbf{W}}^{-1} \mathbf{\Sigma}^{-1} \bm{\mu}_{\mathbf{f}_\textup{vec} | \mathbf{u}_\textup{vec}} + \bm{\mu}_{\mathbf{f}_\textup{vec} | \mathbf{u}_\textup{vec}}^\top \bm{\nu} - \frac{1}{2} \mathbf{u}_\textup{vec}^\top \mathbf{K}_{UU}^{-1}\mathbf{u}_\textup{vec} \right) d\mathbf{u}_\textup{vec}\right).
\end{align*}
We can integrate over $\mathbf{u}_\textup{vec}$ in the latter term to get:
\begin{align*}
    &\log \left(\int \exp \left( - \frac{1}{2} \bm{\mu}_{\mathbf{f}_\textup{vec} | \mathbf{u}_\textup{vec}}^\top \mathbf{J}_{\scriptscriptstyle\mathbf{W}}^{-1} \mathbf{\Sigma}^{-1} \bm{\mu}_{\mathbf{f}_\textup{vec} | \mathbf{u}_\textup{vec}} + \bm{\mu}_{\mathbf{f}_\textup{vec} | \mathbf{u}_\textup{vec}}^\top \bm{\nu} - \frac{1}{2} \mathbf{u}_\textup{vec}^\top \mathbf{K}_{UU}^{-1}\mathbf{u}_\textup{vec} \right) d\mathbf{u}_\textup{vec}\right) \\
    &= \log\!\left(\int \exp \left(\! - \frac{1}{2} \! (\mathbf{u}_\textup{vec}^\top \mathbf{K}_{UU}^{-1} \mathbf{K}_{FU}^\top \mathbf{J}_{\scriptscriptstyle\mathbf{W}}^{-1}\mathbf{\Sigma}^{-1}\mathbf{K}_{FU} \mathbf{K}_{UU}^{-1} + \mathbf{K}_{UU}^{-1}) \mathbf{u}_\textup{vec} + \mathbf{u}_\textup{vec}^\top \mathbf{K}_{UU}^{-1} \mathbf{K}_{FU}^\top \bm{\nu} \! \right) d\mathbf{u}_\textup{vec}\!\right) \\
    &= \log \left( \frac{(2 \pi)^{QR/2} \exp (\frac{1}{2} \bm{\nu}^\top \mathbf{K}_{FU} (\mathbf{K}_{UU} + \mathbf{K}_{FU}^\top \mathbf{\Sigma}^{-1} \mathbf{J}_{\scriptscriptstyle\mathbf{W}}^{-1} \mathbf{K}_{FU})^{-1}\mathbf{K}_{FU}^\top \bm{\nu})}{\sqrt{\textup{det}((\mathbf{K}_{UU}(\mathbf{K}_{UU} + \mathbf{K}_{FU}^\top\mathbf{\Sigma}^{-1}\mathbf{J}_{\scriptscriptstyle\mathbf{W}}^{-1}\mathbf{K}_{FU})^{-1}\mathbf{K}_{UU})^{-1})}} \right)\\
    &= \frac{QR}{2}\log(2\pi) + \frac{1}{2} \bm{\nu}^\top \mathbf{K}_{FU} (\mathbf{K}_{UU} + \mathbf{K}_{FU}^\top \mathbf{\Sigma}^{-1} \mathbf{J}_{\scriptscriptstyle\mathbf{W}}^{-1} \mathbf{K}_{FU})^{-1}\mathbf{K}_{FU}^\top \bm{\nu} \\
    & \qquad \qquad + \frac{1}{2} \log \left(\frac{1}{\textup{det}(\mathbf{K}_{UU}^{-1}(\mathbf{K}_{UU} + \mathbf{K}_{FU}^\top\mathbf{\Sigma}^{-1}\mathbf{J}_{\scriptscriptstyle\mathbf{W}}^{-1}\mathbf{K}_{FU})\mathbf{K}_{UU}^{-1})}\right)\\
    &= \frac{QR}{2}\log(2\pi) + \frac{1}{2} \bm{\nu}^\top \mathbf{K}_{FU} \mathbf{P}_{U}^{-1}\mathbf{K}_{FU}^\top \bm{\nu} + \frac{1}{2} \log \left(\frac{\textup{det}(\mathbf{K}_{UU})^2}{\textup{det}(\mathbf{P}_{U})}\right)
\end{align*}
This results in the final expression for the ELBO:
\begin{align*}
    \textup{ELBO}(Z, \Sigma, \bm{\theta}) = - \frac{1}{2} \biggl(\textup{Tr}&\left(\mathbf{\Sigma}^{-1/2}\mathbf{J}_{\scriptscriptstyle\mathbf{W}}^{-1/2} (\mathbf{K} - \mathbf{K}_{FU} \mathbf{K}_{UU}^{-1} \mathbf{K}_{FU}^\top) \mathbf{J}_{\scriptscriptstyle\mathbf{W}}^{-1/2}\mathbf{\Sigma}^{-1/2}\right) + C(\Sigma, \bm{\theta})\biggr) \\
    &+ \frac{1}{2} \bm{\nu}^\top \mathbf{K}_{FU} \mathbf{P}_{U}^{-1}\mathbf{K}_{FU}^\top \bm{\nu} + \frac{1}{2} \log \left(\frac{\textup{det}(\mathbf{K}_{UU})}{\textup{det}(\mathbf{P}_{U})}\right).
\end{align*}

\end{proof}

\section{Additional Experiments} \label{app:additional_experiments}

The code to reproduce experiments is available at \url{https://github.com/joshuarooijakkers/rmogp}.

\subsection{Implementation Details} \label{app:optimisation_details}

For R-MOGPs, we use the w-LOO-CV objective as given in \Cref{sec:parameter_optimisation}:
\begin{align*}
    \varphi_w(\Sigma, \bm{\theta}) = \sum^N_{i=1}\sum^{T}_{t=1} \left(\frac{w_{i,t}}{\beta_t}\right)^2 \log p_w(y_{i,t}|Y_{-(i,t)}, \Sigma, \bm{\theta}).
\end{align*}
We note, however, that computing the weights $w_{i,t}$ requires the hyperparameters ${C_{-t,-t}^{(i)}} = [\mathcal{K}(\mathbf{x}_i, \mathbf{x}_i) + \Sigma]_{-t, -t}$ and  $C_{-t,t}^{(i)} = [\mathcal{K}(\mathbf{x}_i, \mathbf{x}_i)]_{-t, t}$, which we aim to learn. In principle, one could estimate the full objective jointly, but we found that doing so destabilises the optimisation. To address this, we adopt a common two-step procedure from the literature \citep{huber2011robust, rousseeuw1986robust}: first, the weights are (robustly) estimated and then held fixed during the optimisation of the hyperparameters. Once the optimal parameters are selected, the weights are re-estimated for use during inference. Under the ICM, the hyperparameters ${C_{-t,-t}^{(i)}}$ and $C_{-t,t}^{(i)}$ are components of the covariance matrix $C$ across outputs, and therefore we use the Fast Minimum Covariance Determinant (FastMCD) algorithm \citep{rousseeuw1984least, rousseeuw1999fast} for weight estimation in the first step. This estimator robustly estimates the location and scatter of multivariate data by finding a subset of points whose covariance determinant is minimal.

In summary, to optimise the model parameters, we first estimate a robust covariance matrix using the FastMCD algorithm. Next, we compute the weights as described in \Cref{sec:weight-choice} and keep them fixed throughout the optimisation. For a given set of hyperparameters, we evaluate the LOO-CV objective for each $i \in \{1, \dots, N\}, t\in \{1, \dots, T\}$:
\begin{align*}
    p_w(y_{i,t}|Y_{-(i,t)}, \bm{\theta}, \Sigma) &= \mathcal{N}(\mu^{\scriptscriptstyle\textup{RMOGP}}_{i,t}, (\sigma^{\scriptscriptstyle\textup{RMOGP}}_{i,t})^2 + \sigma_t^2)\\
    \mu^{\scriptscriptstyle\textup{RMOGP}}_{i,t} &= z_{i,t} + m_{i,t} - \frac{[(\mathbf{K} + \mathbf{\Sigma}\mathbf{J}_{\scriptscriptstyle\mathbf{W}})^{-1}\mathbf{z}_{\textup{vec}}]_{(t-1)N+i}}{[(\mathbf{K} + \mathbf{\Sigma}\mathbf{J}_{\scriptscriptstyle\mathbf{W}})^{-1}]_{(t-1)N+i, (t-1)N+i}}\\
    (\sigma^{\scriptscriptstyle\textup{RMOGP}}_{i,t})^2 &= \frac{1}{[(\mathbf{K} + \mathbf{\Sigma}\mathbf{J}_{\scriptscriptstyle\mathbf{W}})^{-1}]_{(t-1)N+i, (t-1)N+i}} - \frac{\sigma_t^4}{2}w_t(\mathbf{x}_i, \mathbf{y}_i)^{-2}.
\end{align*}
We then compute the weighted LOO-CV objective as a sum over all $i,t$ using weights $(w_{i,t}/\beta_t)^2$. The optimisation proceeds by iteratively updating the hyperparameters via the L-BFGS algorithm \citep{byrd1995limited} until convergence. Finally, the weights are re-estimated using the optimised covariance matrix.

\subsection{Performance Metrics} \label{app:issues_RCGPs}

To evaluate and compare model performance, we compute the root mean squared error (RMSE) across all observations and outputs:
\begin{align*}
    \textup{RMSE} = \sqrt{\frac{1}{N} \sum_{i=1}^N \frac{1}{T_i}\sum_{t=1}^{T_i} (y_{i,t} - \hat{y}_{i,t})^2},
\end{align*}
where $T_i \in \{1, \dots, T\}$ denotes the number of output variables for observation $i \in \{1, \dots, N\}$, ${y}_{i,t}$ is the true value, and $\hat{y}_{i,t}$ is the corresponding model prediction. The second performance metric we consider is the negative log predictive density (NLPD), which is given as
\begin{align*}
    \textup{NLPD} = -\frac{1}{N} \sum_{i=1}^N \frac{1}{T_i}\sum_{t=1}^{T_i} \log \mathcal{N} (y_{i,t} ; \hat{y}_{i,t}, \hat{\sigma}^2_{i,t}),
\end{align*}
where $\hat{\sigma}^2_{i,t}$ is the model's variance of its prediction.

\subsection{Closed-form Conditional Expectations} \label{app:conditional_expectation}

The closed-form conditional expectations presented in \Cref{sec:weight-choice} hold for general multi-output kernels. Most multi-output kernel constructions yield a covariance matrix $\mathbf{K} \in \mathbb{R}^{NT \times NT}$. For a given input $\mathbf{x}_i$, the sliced matrix $A^{(i)}$ is obtained by extracting entries $(i+kN,\, i+lN)$ for $k,l \in \{0,\dots,T-1\}$ from $\mathbf{K}$. This is equivalent to evaluating the cross-output kernel at the same input, i.e.\ $\mathcal{K}(\mathbf{x}_i,\mathbf{x}_i)$. The corresponding covariance matrix is therefore $C^{(i)} = A^{(i)} + \Sigma$. We now illustrate this construction for three common multi-output kernel choices.

\paragraph{Intrinsic Coregionalisation Model (ICM)} Under the ICM, the multi-output covariance is $\mathbf{K} = B \otimes \kappa(X, X)$, where $\kappa(X, X) \in \mathbb{R}^{N \times N}$. Applying the slicing procedure yields $[A^{(i)}]_{t,t'} = B_{t,t'} \kappa(\mathbf{x}_i, \mathbf{x}_i)$. Since $\kappa(\mathbf{x}_i,\mathbf{x}_i)=1$ for normalised kernels, it follows that $A^{(i)} = B$ such that $C^{(i)} = B + \Sigma$. 

\paragraph{Linear Model of Coregionalisation (LMC)} For the LMC, $\mathbf{K} = \sum_{g=1}^{G} B_g \otimes \kappa_g(X, X)$ for $g \in \{1, \dots, G\}$. The sliced matrix becomes $[A^{(i)}]_{t,t'} = \sum_{g=1}^G [B_g]_{t, t'} \kappa_g(\mathbf{x}_i, \mathbf{x}_i) = \sum_{g=1}^G [B_g]_{t, t'}$, again using $\kappa_g(\mathbf{x}_i,\mathbf{x}_i)=1$. Hence, $C^{(i)} = \sum_{g=1}^{G}B_g + \Sigma$.

\paragraph{Process Convolution (PC)} The process convolution kernel is defined as $[\mathcal{K}(\mathbf{x}, \mathbf{x}')]_{t,t'} = \int G_t(\mathbf{x} - \mathbf{z}) G_{t'}(\mathbf{x}' - \mathbf{z}) d\mathbf{z}$ where $G_t$ and $G_{t'}$ are smoothing kernels. For Gaussian MOGPs, $G_t(\mathbf{x} - \mathbf{z}) = v_t \exp(-\frac{1}{2} \lambda_t (\mathbf{x} - \mathbf{z})^2)$. The integral admits a closed form:

\begin{align*}
    [\mathcal{K}(\mathbf{x}, \mathbf{x}')]_{t,t'} = \frac{v_t v_{t'} \sqrt{2\pi}}{\sqrt{\lambda_t + \lambda_{t'}}} \exp\left( -\frac{\lambda_t \lambda_{t'}}{2(\lambda_t + \lambda_{t'})}(\mathbf{x} - \mathbf{x}')^2 \right).
\end{align*}
Evaluating at $\mathbf{x} = \mathbf{x}'=\mathbf{x}_i$ gives $[A^{(i)}]_{t,t'} = v_t v_{t'} \sqrt{2\pi}/\sqrt{\lambda_t + \lambda_{t'}}$.

To illustrate that R-MOGP is also robust when using non-ICM multi-output kernels, we replicate the experiment from \Cref{fig:toy_example}, replacing the ICM kernel with the LMC and PC kernels. The resulting predictive distributions, shown in \Cref{fig:LMC_PC}, demonstrate that R-MOGP is likewise able to robustly impute the missing values under these alternative multi-output kernel specifications.

\begin{figure}[htbp]
    \centering
    \begin{subfigure}[b]{0.48\textwidth}
        \centering
        \includegraphics[width=\textwidth]{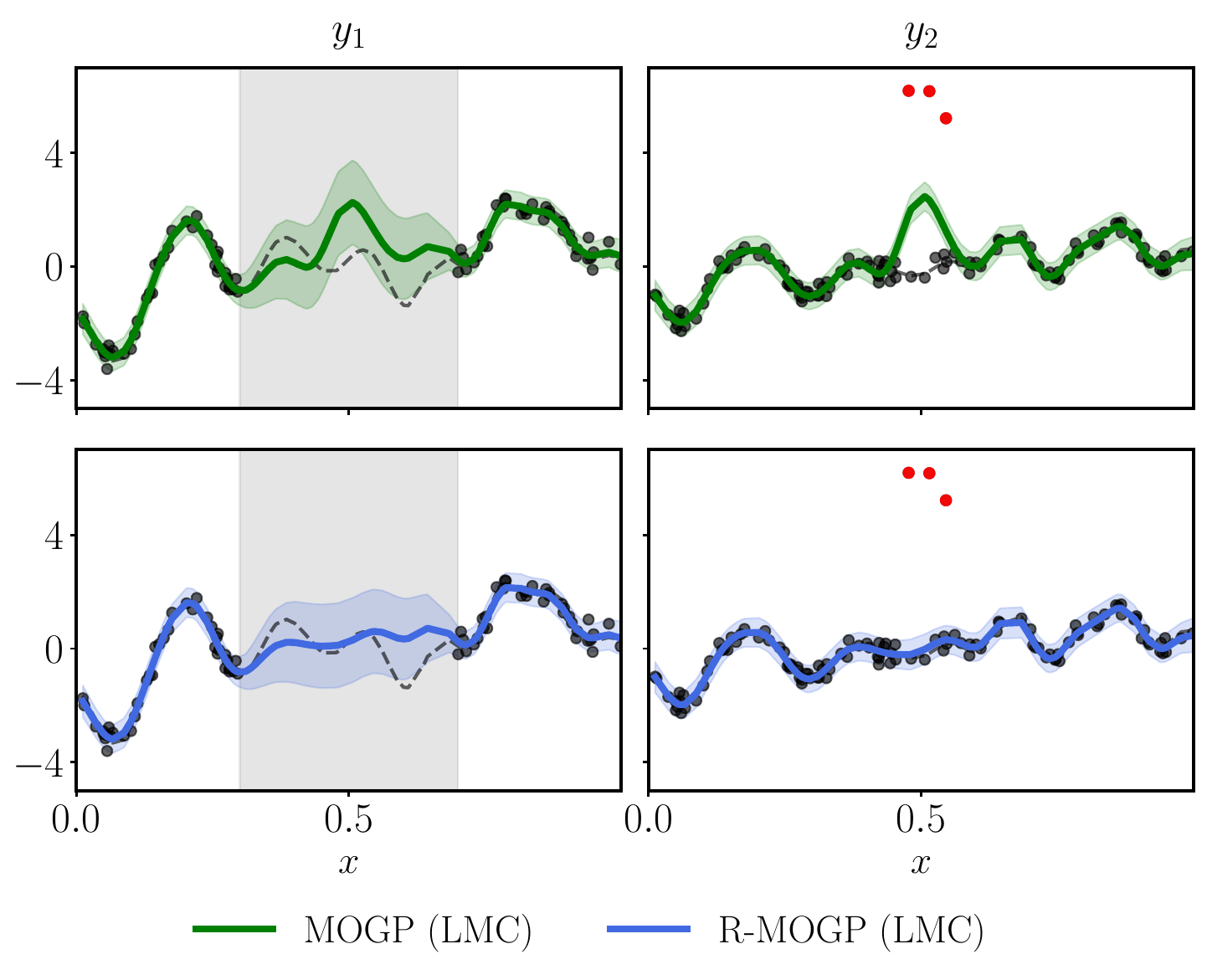}
        \caption{Linear Model of Coregionalisation (LMC).}
        \label{fig:LMC}
    \end{subfigure}
    \hfill
    \begin{subfigure}[b]{0.48\textwidth}
        \centering
        \includegraphics[width=\textwidth]{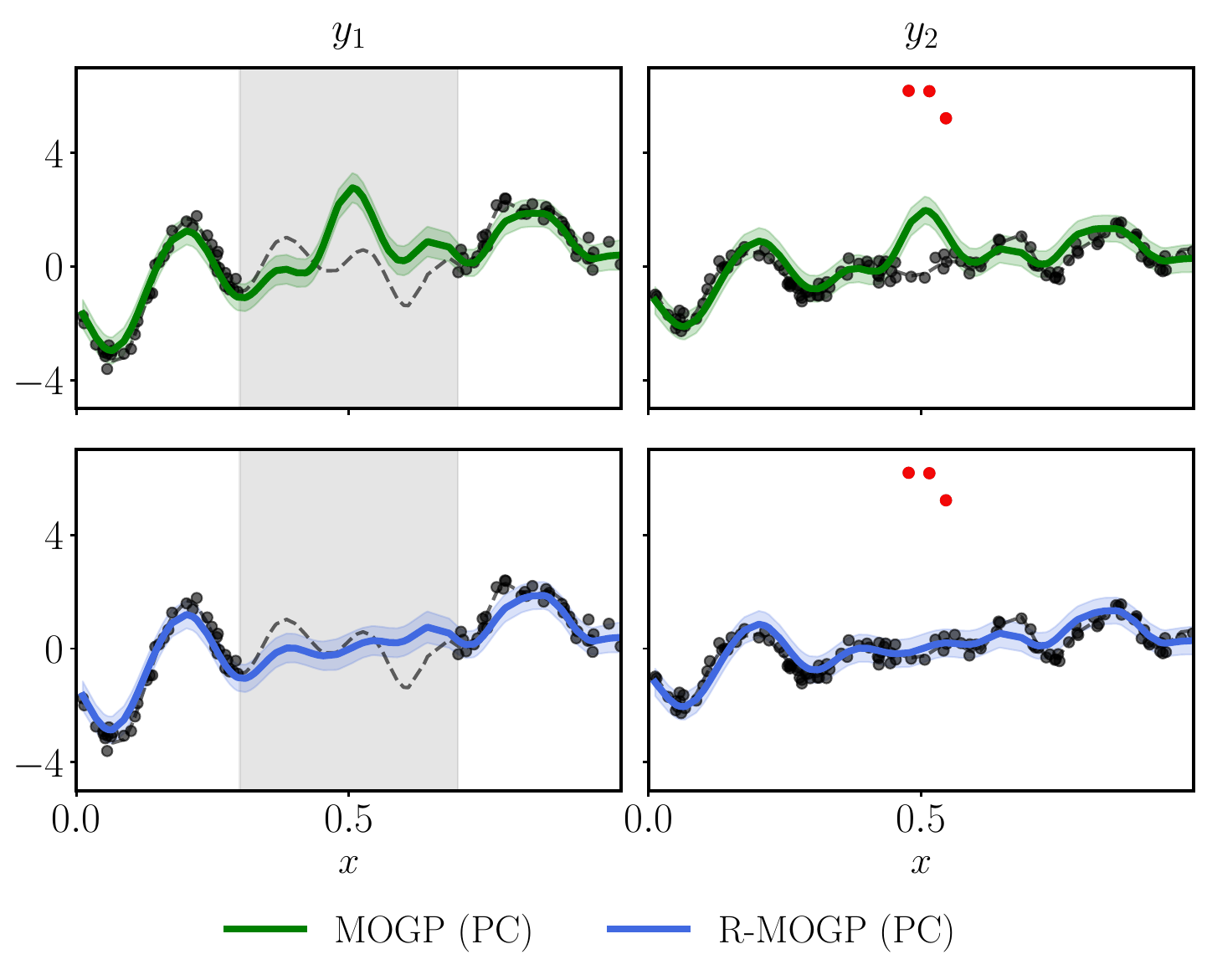}
        \caption{Process Convolution (PC).}
        \label{fig:PC}
    \end{subfigure}
    \caption{Synthetic imputation experiment using R-MOGPs based on the Linear Model of Coregionalisation (LMC) and Process Convolution (PC) frameworks.}
    \label{fig:LMC_PC}
\end{figure}

\subsection{Synthetic Multi-task Imputation Problem in Figure~\ref{fig:toy_example}} \label{app:toy_example}

We first uniformly at random sample $N = 120$ input locations ($d=1$) over the interval [0,1]. The true latent functions are generated using an ICM with coregionalisation parameters $B_{11} = 2$, $B_{22} = 1$, and $B_{12} = 1.25$, and a squared exponential kernel with $\ell = 0.05$. Observations are obtained by adding Gaussian noise $\mathcal{N}(0, 0.05)$. Next, we contaminate $\epsilon_2 = 2.5\%$ of the observations in $y_2$ around $x=0.5$ by replacing them with samples from $\mathcal{N}(6, 0.5)$. Finally, we remove data in $y_1$ where $0.3 < x < 0.7$.

In addition to the standard MOGP and our proposed R-MOGP, we include the $t$-MOGP and L-MOGP for a comprehensive comparison of robust methods (see Appendix~\ref{app:energy_efficiency} for implementation details). All methods use prior mean $m_t(\mathbf{x}) = \frac{1}{N} \sum^N_{i=1} y_{i,t}$. 

As illustrated in Figure~\ref{fig:toy_example_all}, the predictive distributions demonstrate that R-MOGP, $t$-MOGP and L-MOGP effectively mitigate the influence of outliers. In contrast, the standard MOGP is significantly biased by the contaminated data, leading to poor predictive performance in both the corrupted regions and the missing data gap of $y_1$.
\begin{figure*}[htbp]
  \centering
  \includegraphics[width=\textwidth]{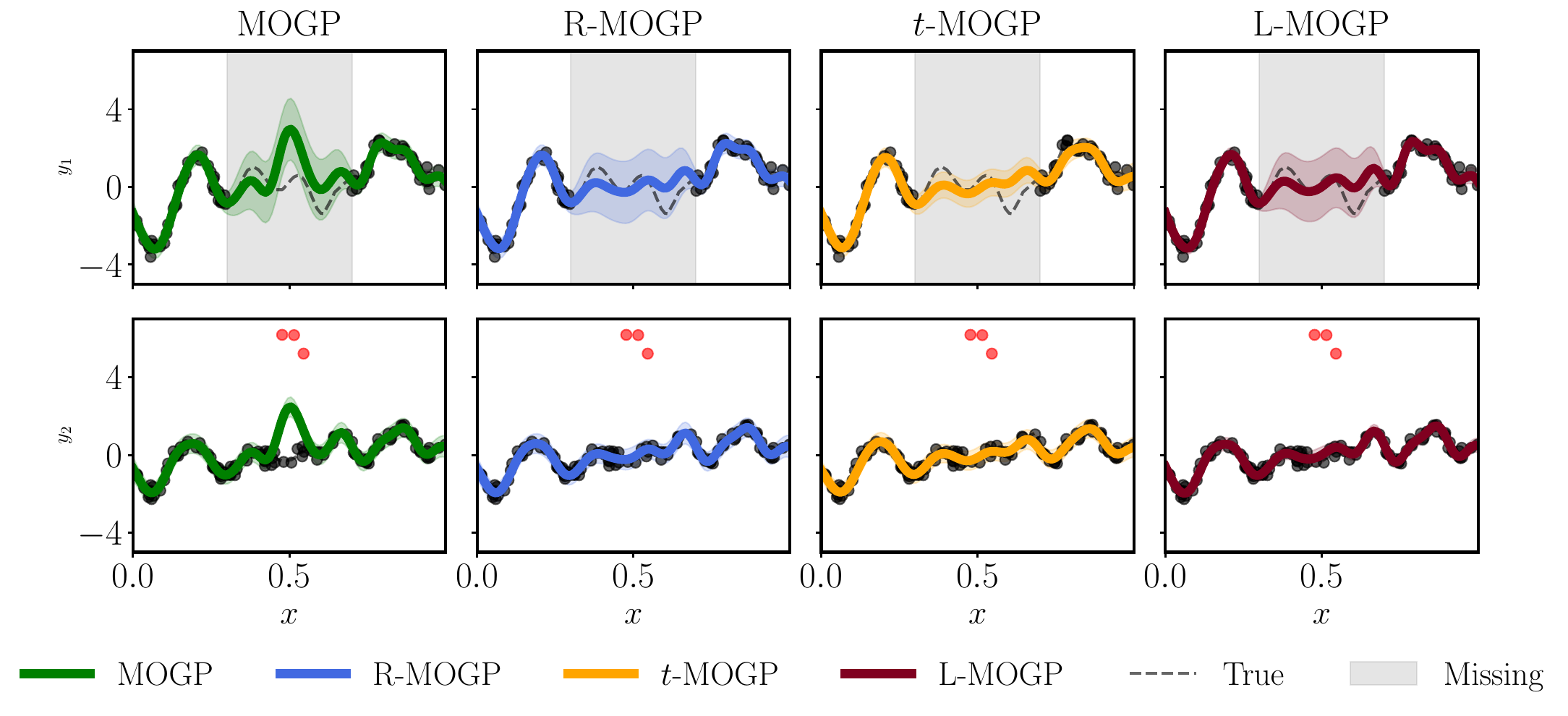}
  \caption{Synthetic imputation experiment for MOGP, R-MOGP, $t$-MOGP, and L-MOGP.}
  
  \label{fig:toy_example_all}
\end{figure*}

Next, we assess the reliability of uncertainty quantification for R-MOGP and MOGP using coverage plots for the same experiment, both with and without outliers, given in \Cref{fig:coverage_final}. We increase the sample size to $N=1000$ to obtain more stable coverage estimates. In the absence of outliers, the two methods achieve very similar coverage. In the presence of outliers, R-MOGP maintains coverage comparable to that in the uncorrupted setting, whereas standard MOGP exhibits substantially lower coverage.

\begin{figure*}[htbp]
  \centering
  \includegraphics[width=0.6\textwidth]{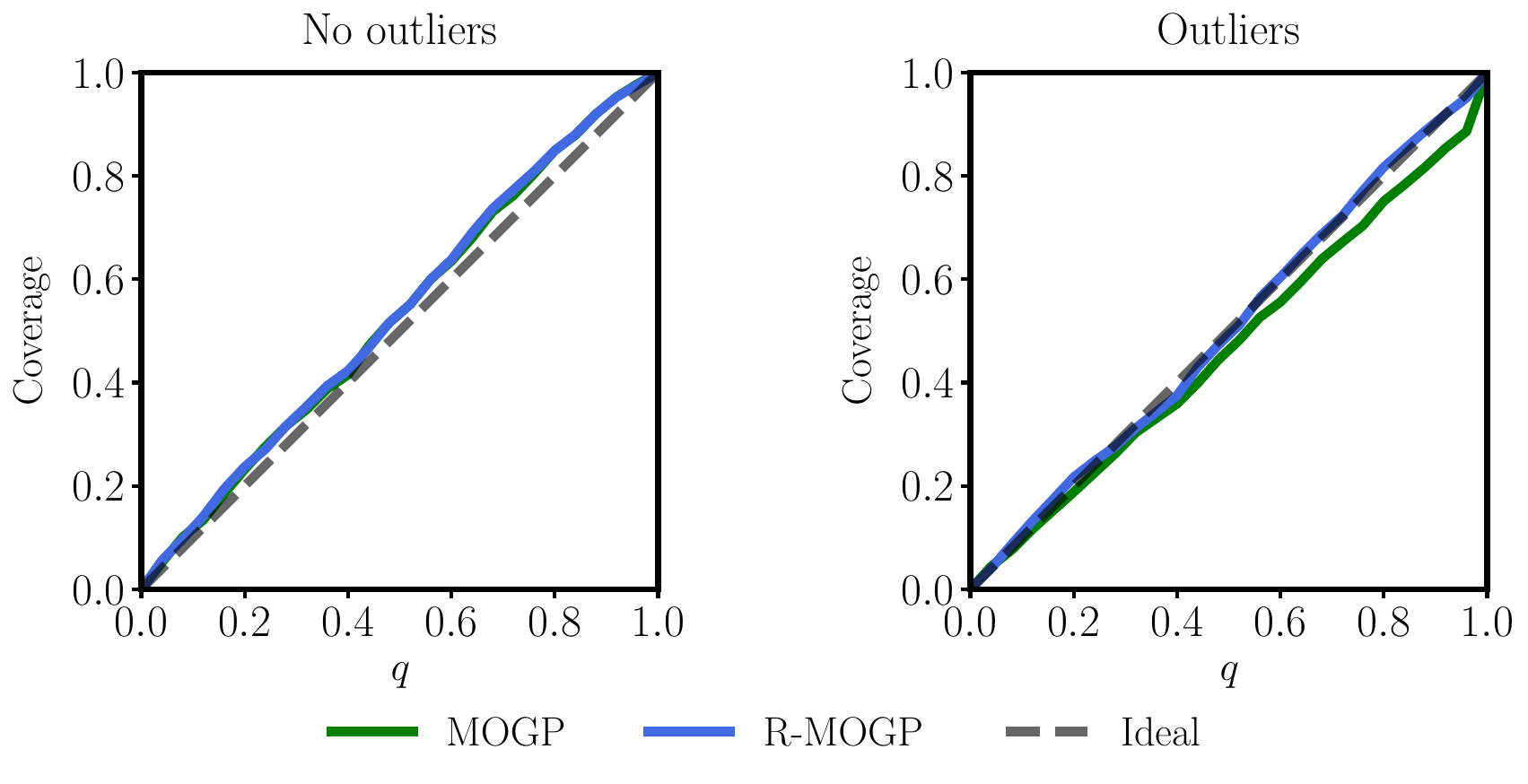}
  \caption{Coverage plots for MOGP and R-MOGP in the synthetic imputation experiment with $N = 1000$, with and without outliers. Coverage is the proportion of uncontaminated observations that fall within predictive intervals with nominal coverage $q$.
}
  
  \label{fig:coverage_final}
\end{figure*} 

\subsection{Synthetic MOGP Problem from Section~\ref{sec:experiments}} \label{app:synthetic_MOGP_table}

We first uniformly sample $N = 100$ input locations ($d=1$) over the interval $[-5,5]$. Observations are then generated across $T = 3$ outputs from an ICM with a squared exponential kernel of lengthscale $\ell = 1$, and the coregionalisation matrix
\begin{align*}
    B =
    \begin{bmatrix}
        1.0 & 0.9 & 0.7 \\
        0.9 & 1.0 & 0.8 \\
        0.7 & 0.8 & 1.0 \\
    \end{bmatrix}.
\end{align*}
Gaussian noise $\varepsilon_{i,t} \sim \mathcal{N}(0, 0.1^2)$ is added to all outputs $t \in \{1, \dots, 3\}$. This results in the dataset we use in the well-specified setting. In the outlier-contaminated setting, we introduce $\epsilon_1 = 10\%$ contamination by randomly selecting observations from the first output ($t = 1$) and perturbing them with outliers. Specifically, for each selected observation, we add a random draw from $U[2, 3]$ if $y_{i,t} < m_t(\mathbf{x}_i)$, and subtract it otherwise.

For the R-MOGP ($w_{\textup{RMOGP}}$), we use the closed-form expression in Proposition~\ref{prop:MORCGP} with the weights proposed in \Cref{sec:weight-choice}. For the R-MOGP ($w_{\textup{RCGP}}$), we use the closed-form expression in Proposition~\ref{prop:MORCGP} where we use the weight equation defined in \eqref{eq:weight}. However, we now use that $\gamma_t(\mathbf{x}, \mathbf{y}) = m_t(\mathbf{x})$, and $c_t(\mathbf{x}, \mathbf{y})$ is chosen as the $(1-\epsilon_t)$-th quantile of $\{|y_{i,t} - m_t(\mathbf{x}_i)|\}^N_{i=1}$. For all methods, the prior mean is taken as $m_t(\mathbf{x}) = 0$. Both R-MOGP variants use the optimisation objective described in \Cref{sec:parameter_optimisation}, each with their respective weights. We repeat the experiment 20 times, each with a unique random seed.

\subsection{Source of Robustness in R-MOGP}

The R-MOGP posterior in Proposition~\ref{prop:MORCGP} differs from the standard MOGP posterior in \eqref{eq:mogp_predictive} in two respects: the identity matrix $I_{NT}$ is replaced by the weight-dependent matrix $\mathbf{J}_{\mathbf{W}}$, and the prior mean $\mathbf{m}_\textup{vec}$ is replaced by its weighted counterpart $\mathbf{m}_{\mathbf{W}, \textup{vec}}$. The matrix modification acts as a noise term, treating outliers as noisier observations, while the mean modification acts as a shrinkage term, trading a small increase in bias for a reduction in variance \citep{altamirano2024robustGP}. In \Cref{fig:outlier_source}, we isolate the effects of these two terms across varying values of $\epsilon_2$. The noise term is essential to robustness: removing it causes R-MOGP to lose robustness. By contrast, R-MOGP is relatively insensitive to changes in the shrinkage term. The results also show that poor calibration of $\epsilon_2$ can compromise robustness.

\begin{figure*}[htbp]
  \centering
  \includegraphics[width=\textwidth]{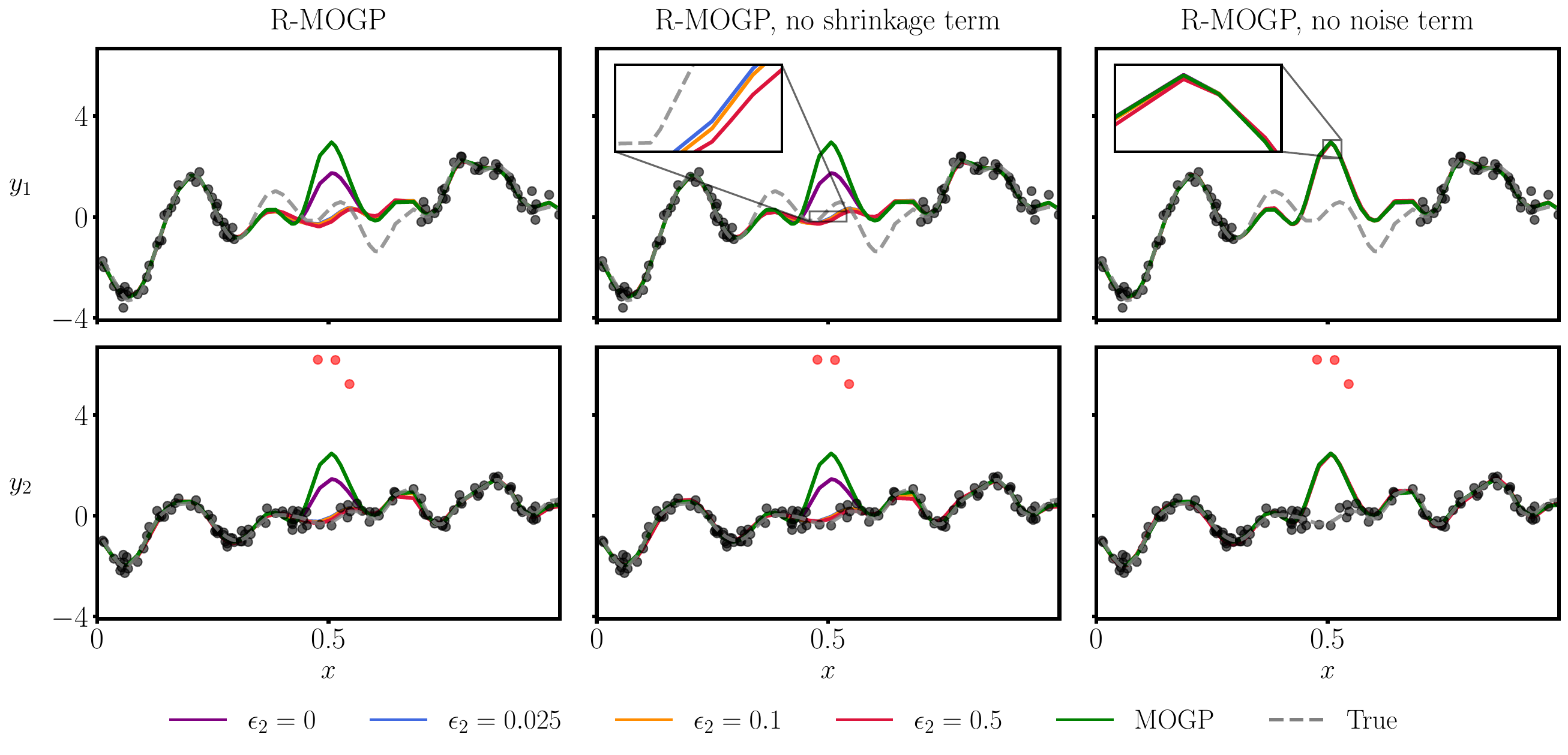}
  \caption{Posterior mean of R-MOGP (left), R-MOGP without shrinkage term (middle) and R-MOGP without noise term (right) for varying values of $\epsilon_2$ on the synthetic imputation experiment.}
  
  \label{fig:outlier_source}
\end{figure*} 

\subsection{Effect of Weights and Objectives on Robustness} \label{app:effect_robustness}

In R-MOGP, robustness is achieved through both the choice of weight (\Cref{sec:weight-choice}) and the robust hyperparameter optimisation objective (\Cref{sec:parameter_optimisation}). We assess the impact of these two factors by evaluating models under different weight schemes and optimisation objectives. To this end, we use the same synthetic dataset as in Appendix~\ref{app:synthetic_MOGP_table}. We compare MOGP, R-MOGP ($w_{\textup{RCGP}}$), and R-MOGP ($w_{\textup{RMOGP}}$) using various weightings in the LOO-CV objective: no weights (i.e. all equal to 1), linear weights, and squared weights. For MOGP, we adopt the initial R-MOGP weights. Results are presented in Table~\ref{tab:effect_robustness}.

When using the unweighted objective, R-MOGP gains little over MOGP. This is expected: the unweighted LOO objective is essentially equivalent to maximising the standard log marginal likelihood, which is well known to be non-robust to outliers. Under such an objective, the hyperparameters are strongly influenced by the outliers, such that our proposed likelihood cannot express its intended robustness. Conversely, simply applying (squared) weighted optimisation to a standard MOGP does not yield robustness either: for prediction, the MOGP closed form is not robust in presence of model misspecification (as proven in Lemma~\ref{lem:PIF_mogp}). Therefore, robustness arises only from the combination of (i) the robust R-MOGP formulation and (ii) a robust hyperparameter objective. Finally, in the absence of outliers, all methods perform equally well regardless of weighting. This is expected, as all observations receive equal weights, making the objectives effectively equivalent.

\begin{table}[ht]
\caption{Comparison of MOGP and R-MOGP ($w_\textup{RCGP}$) and R-MOGP ($w_\textup{RMOGP}$) under different LOO-CV weighting schemes (unweighted, weighted, squared weights). The results show that R-MOGP ($w_\textup{RMOGP}$) combined with a squared weighted hyperparameter optimisation objective is most robust to outliers.} \label{tab:effect_robustness}
\resizebox{\textwidth}{!}{%
\begin{tabular}{clcccccc}
\hline
\textbf{Outliers} & \textbf{Method} & \multicolumn{2}{c}{\textbf{Unweighted}} & \multicolumn{2}{c}{\textbf{Weighted}} & \multicolumn{2}{c}{\textbf{Squared}} \\
 &  & RMSE & NLPD & RMSE & NLPD & RMSE & NLPD \\ \hline
 & MOGP & $0.20 \pm 0.03$ & $-0.37 \pm 0.04$ & $0.18 \pm 0.03$ & $-0.48 \pm 0.05$ & $0.18 \pm 0.03$ & $-0.52 \pm 0.07$ \\
\cmark & R-MOGP ($w_\textup{RCGP}$) & $0.19 \pm 0.03$ & $-0.38 \pm 0.04$ & $0.19 \pm 0.04$ & $-0.40 \pm 0.04$ & $0.18 \pm 0.04$ & $-0.44 \pm 0.07$ \\
 & R-MOGP ($w_\textup{RMOGP}$) & $0.18 \pm 0.04$ & $-0.38 \pm 0.04$ & $0.13 \pm 0.03$ & $-0.53 \pm 0.06$ & $0.11 \pm 0.01$ & $-0.67 \pm 0.08$ \\ \hline
 & MOGP & $0.093 \pm 0.003$ & $-0.95 \pm 0.04$ & $0.093 \pm 0.004$ & $-0.95 \pm 0.04$ & $0.093 \pm 0.004$ & $-0.95 \pm 0.04$ \\
\xmark & R-MOGP ($w_\textup{RCGP}$) & $0.093 \pm 0.003$ & $-0.95 \pm 0.03$ & $0.093 \pm 0.004$ & $-0.94 \pm 0.04$ & $0.093 \pm 0.004$ & $-0.94 \pm 0.04$ \\
 & R-MOGP ($w_\textup{RMOGP}$) & $0.093 \pm 0.004$ & $-0.95 \pm 0.04$ & $0.093 \pm 0.004$ & $-0.95 \pm 0.04$ & $0.093 \pm 0.004$ & $-0.95 \pm 0.04$ \\ \hline
\end{tabular}}
\end{table}

\subsection{Choice of the Weight Function} \label{app:choice_weight_function} 

In \Cref{sec:weight-choice} and \Cref{tab:synthetic_imputation}, we discussed how the naive RCGP weight used in \citet{altamirano2024robustGP} is suboptimal, as it is sensitive to prior mean misspecification. To illustrate this and demonstrate the improvement afforded by our proposed weight, we consider the setup in \Cref{fig:choice_weight_function}. We first uniformly generate input values over the interval $[0,1]$ and then sample $N=120$ data points across $T=2$ outputs from an ICM with $B_{11}=2$, $B_{22}=1$, $B_{12}=1.25$, and a squared exponential kernel with lengthscale $\ell=0.1$. Gaussian noise $\mathcal{N}(0,0.05)$ is added to the observations. These hyperparameters are used for predictions across all models. For the first output ($t=1$), we introduce outliers corresponding to $\epsilon_1 = 10\%$ of $N$ at around $x = 0.5$, and the corresponding observations are replaced by samples from $\mathcal{N}(2.5,0.5)$.

We generate predictions from three models: the standard MOGP, the R-MOGP with the RCGP weight $w_{\textup{RCGP}}$ from \citet{altamirano2024robustGP}, and the R-MOGP with our proposed weight $w_{\textup{RMOGP}}$ (see \Cref{sec:weight-choice}). Note that R-MOGP ($w_{\textup{RCGP}}$) uses the same multi-output $w_\textup{RCGP}$ as detailed in Appendix~\ref{app:synthetic_MOGP_table}. Across all models, we use the prior mean $m_t(\mathbf{x}) = \frac{1}{N}\sum^N_{i=1}y_{i,t}$.

\begin{figure*}[htbp]
  \centering
  \includegraphics[width=0.9\textwidth]{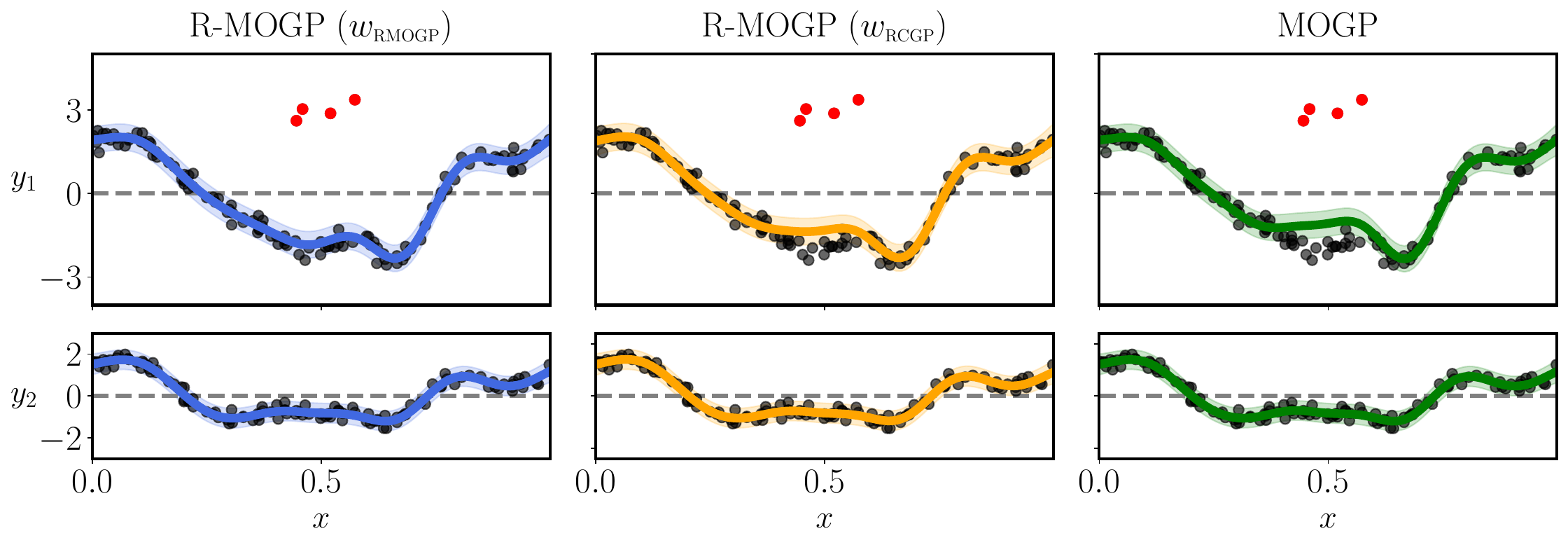}
  \caption{Effect of the weight function on the robustness of R-MOGP. R-MOGP~($w_{\scriptscriptstyle \textup{RMOGP}}$) uses the weight function described in this paper, while the R-MOGP ($w_{\scriptscriptstyle \textup{RCGP}}$) use the original RCGP weight with $\gamma_t(\mathbf{x}, \mathbf{y}) = m_t(\mathbf{x}) = \frac{1}{N}\sum^N_{i=1}y_{i,t}$. Among all methods, only R-MOGP~($w_{\scriptscriptstyle \textup{RMOGP}}$) is robust to the contamination.}
  
  \label{fig:choice_weight_function}
\end{figure*} 

The plotted predictives show that both the standard MOGP and the R-MOGP with $w_{\textup{RCGP}}$ lack robustness: the former due to its inability to downweight observations, and the latter because outliers and data are at the same distance from the chosen prior, so both receive the same weight. In contrast, the R-MOGP with the proposed weight $w_{\textup{RMOGP}}$ shows robustness by effectively sharing information across outputs, since the centering function is taken to be the conditional expectation given the observations in the other output channel.

\subsection{Performance Comparison for Multivariate Outliers}

In many settings, outliers occur jointly, affecting multiple observed values for a given input (see, e.g. \citealp{leys2018detecting}). In this section, we analyse how the performance of both MOGP and R-MOGP changes as the number of jointly outlying observations in the observed vector increases. To ensure joint contamination, we introduce correlated outliers such that the Mahalanobis distance between the contaminated and uncontaminated points is large (set-up described below). This approach reflects realistic conditions in which outliers tend to be correlated \citep{rousseeuw2011robust}, and the Mahalanobis distance ensures that the joint deviation is large relative to the covariance structure.

Specifically, we first uniformly at random generate $N = 80$ input values over the interval $[0,1]$. We then sample $T=10$ outputs from an ICM with coregionalisation matrix $B$, where $B_{ij} = 1 - 0.1 |i-j|$, and a squared exponential kernel with $\ell = 0.1$. Gaussian noise $\mathcal{N}(0, 0.05)$ is then added to each data point. These parameters are used for predictions in both models.

To generate the outliers, we uniformly at random select 10\% of the input values for contamination. For a specified number of outputs $S \in \{1, \dots, T\}$ to be jointly outlying, we add perturbations with large Mahalanobis distance, following \citet{mardia2024multivariate}. Specifically, for a given target Mahalanobis distance $\textup{MD}_{\textup{target}}$, and a base vector $\mathbf{v} \in \mathbb{R}^S$, we compute the Mahalanobis norm $q = \mathbf{v}^\top B^{-1} \mathbf{v}$. We then determine the scaling factor $\alpha = \textup{MD}_{\textup{target}} / \sqrt{q}$, which yields the final outlier vector $\alpha \mathbf{v}$ with desired Mahalanobis distance $\textup{MD}_{\textup{target}}$. This perturbation is multiplied by either -1 or 1 with equal probability, before being added to the first $S$ elements of the selected output vector. For each input, $\textup{MD}_{\textup{target}}$ is sampled from the univariate uniform distribution $U[10, 15]$, and the base vector is drawn from an $S$-variate uniform distribution $U[0.5, 1.5]$ to ensure correlation among the outliers. We set the prior mean as $m_t(\mathbf{x}) = \frac{1}{N}\sum^N_{i=1}y_{i,t}$ for $t \in \{1, \dots, T\}$. The experiment is repeated for $S \in \{1, \dots, T\}$, with 20 repetitions per $S$, each using a unique seed.

The results are displayed in \Cref{fig:increasing_outliers}. As the dimensionality of the outliers increases, both RMSE and NLPD grow much faster for MOGP than for R-MOGP. Thus, R-MOGP remains more robust even when outliers occur in higher dimensions.

\begin{figure}[htbp]
    \centering
    \begin{subfigure}[b]{0.45\textwidth}
        \centering
        \includegraphics[width=\textwidth]{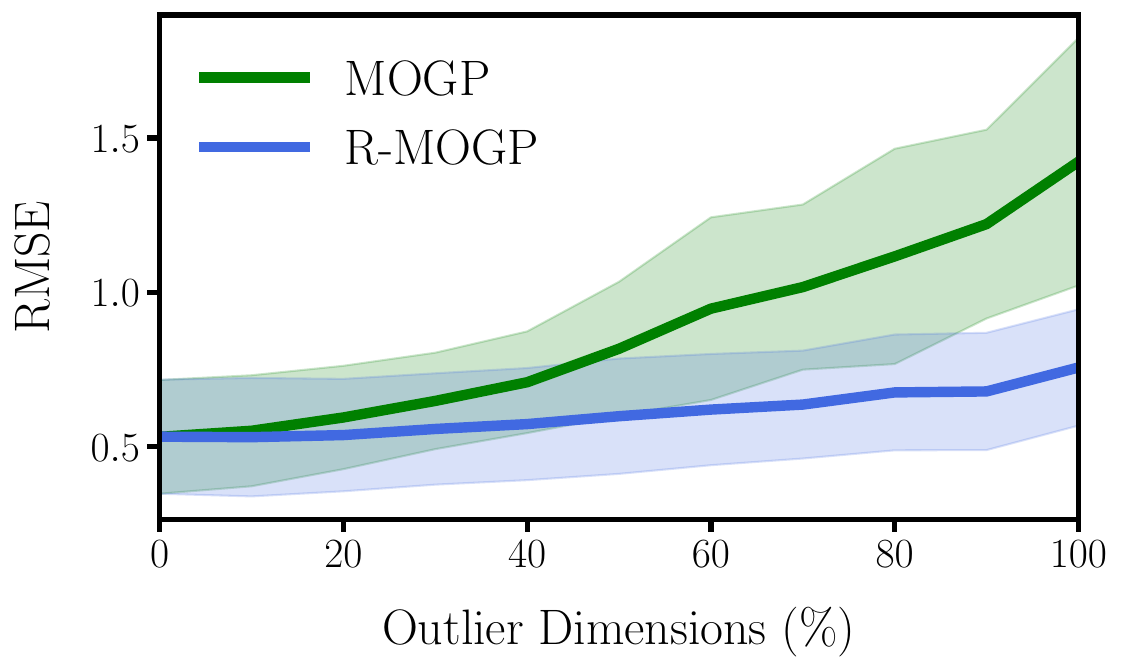}
        \caption{Mean RMSE and one standard deviation across 20 random seeds versus the percentage of dimensions in which outliers occur.}
        \label{fig:RMSE_increasing_outliers}
    \end{subfigure}
    \hfill
    \begin{subfigure}[b]{0.45\textwidth}
        \centering
        \includegraphics[width=\textwidth]{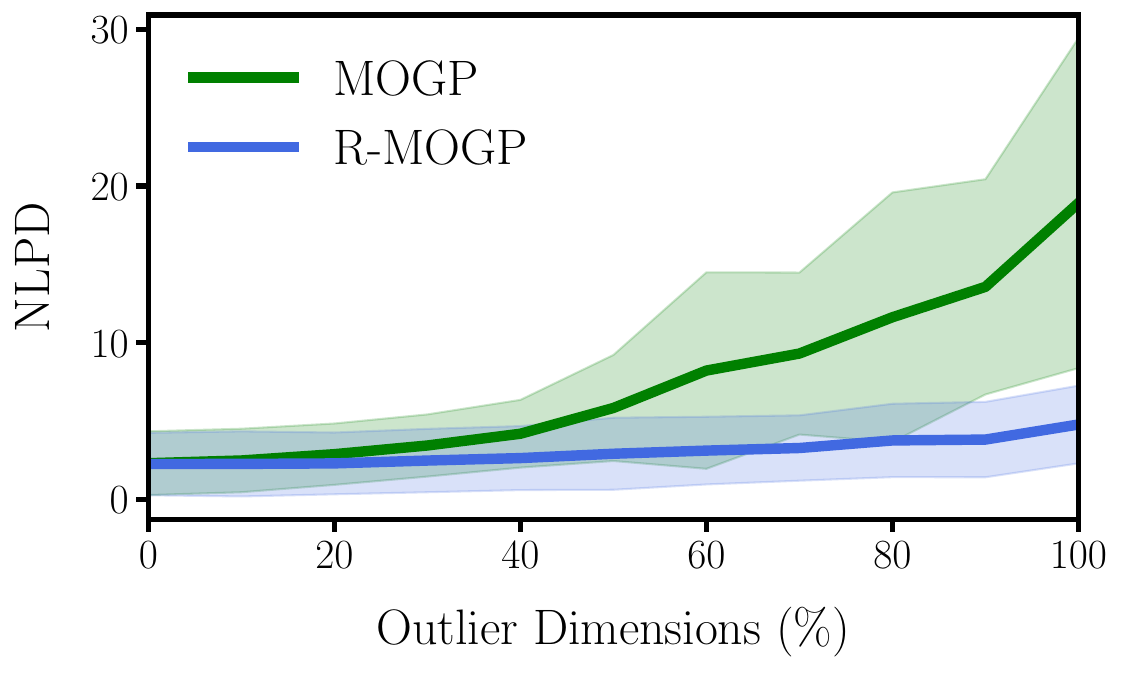}
        \caption{Mean NLPD and one standard deviation across 20 random seeds versus the percentage of dimensions in which outliers occur.}
        \label{fig:NLPD_increasing_outliers}
    \end{subfigure}
    \caption{{Performance of \textcolor{Green}{{MOGP}} and \textcolor{RoyalBlue}{{R-MOGP}} as the number of outliers in the observed vector increases.} We generate $N=80$ data points from an ICM across $T=10$ outputs. We then contaminate $\epsilon_s = 10\%$ of the observations with multivariate outliers for $s \in \{1, \dots, S\}$ with $S \leq T$ increasing incrementally. The prior mean is $m_t(\mathbf{x}) = \frac{1}{N}\sum_{i=1}^{N} y_{i,t}$. As the percentage of outliers grows, the performance of the MOGP drops substantially faster than that of the R-MOGP.}
    \label{fig:increasing_outliers}
\end{figure}

\subsection{Energy Efficiency Prediction} \label{app:energy_efficiency}

\paragraph{Dataset}

We obtain data from the \textsc{Energy Efficiency} dataset from the \href{https://archive.ics.uci.edu/}{UCI repository}, which can be downloaded at \url{https://archive.ics.uci.edu/dataset/242/energy+efficiency}. The dataset was generated using building-energy simulation in ECOTECT for 12 prototype residential building shapes, systematically varied design parameters, and was designed to assess the energy efficiency of buildings based on various architectural characteristics. Each data point represents a building and consists of $d=8$ covariates: relative compactness, surface area, wall area, roof area, overall height, orientation, glazing area and glazing area distribution. The goal is to predict $T=2$ outputs: heating load ($t=1$) and cooling load ($t=2$). 

We standardise both the input and output variables by subtracting their own mean and then dividing by their own standard deviation. The dataset comprises $N=768$ data points, which we randomly split into a training set (75\%) and a test set (25\%) for model evaluation. This comprises the setup for the scenario with no outliers. 

\paragraph{Outlier Scenarios}

We uniformly at random introduce outliers into $\epsilon_1 = 10\%$ of the heating load observations using the following contamination strategies, adapted from \citet{altamirano2024robustGP}:
\begin{itemize}[label=\scriptsize$\bullet$]
\item \textit{Uniform outliers}: A specified proportion of data points (after standardisation) is uniformly at random chosen for contamination. This subset is then split evenly: for half the points, we add noise sampled from $z \sim U(6, 9)$; for the other half, we subtract it. Here, $U$ represents the uniform distribution.
\item \textit{Asymmetric outliers}: We uniformly sample a fixed proportion of the dataset (after standardisation), as with uniform outliers. However, instead of splitting, we contaminate all selected points by adding $z \sim U(6, 9)$, introducing a directional bias in the outliers.
\item \textit{Student-$t$ outliers}: A specified proportion of data points (after standardisation) is uniformly at random chosen for contamination. We then add noise sampled from $z \sim t_3(0, 5)$ to these data points. Here, $t_\nu(\mu, \sigma)$ represents the Student-$t$ distribution with location parameter $\mu$, scale parameter $\sigma$, and $\nu$ degrees of freedom.
\end{itemize}

\paragraph{Implementation}

The experiments were run on a 2025 HP ZBook Power G11 Mobile Workstation PC with an Intel(R) Core(TM) Ultra 7 155H CPU with 16GB memory.

We compare the MOGP and R-MOGP models against the $t$-MOGP and L-MOGP, which extend the MOGP by using a Student-$t$ and Laplace likelihood, respectively. This likelihood breaks conjugacy, so the posterior and predictive distributions are no longer available in closed form and must instead be approximated. Our implementation is in \texttt{GPflow} \citep{matthews2017gpflow}, which uses the variational inference method of \citet{opper2009variational} to estimate the posterior, predictive distribution, and hyperparameters. However, the package only provides a built-in implementation for scale parameters shared across outputs, which we adopt. To ensure fairness, we also fix $\sigma_t = \sigma$ across the MOGP and R-MOGP models.
For all outlier scenarios (including the no-outlier case), we set $\nu = 3$, the default in \texttt{GPflow}. We found that larger values of $\nu$ consistently degrade performance in outlier-based scenarios. 

For a fair comparison, parameter selection for both MOGP and R-MOGP is performed using the (w)-LOO-CV objective. The methods are implemented in \texttt{TensorFlow} \citep{abadi2016tensorflow}, the same underlying framework \texttt{GPflow} uses. We use a prior mean function with $m_t(\mathbf{x}) = \frac{1}{N}\sum_{i=1}^{N} y_{i,t}$ for $t \in \{1,2\}$. Experiments are repeated 10 times with different train-test splits. The computation times illustrated in \Cref{fig:UCI_runtime_fe} are based on these 40 runs (10 per outlier scenario).

Finally, we compare R-MOGP with several baseline machine learning methods to assess its robustness in a broader context. We include a standard MLP with three hidden layers of 64 neurons each and two regression heads corresponding to the $T=2$ outputs. Since this MLP does not provide predictive uncertainty estimates, we also consider a Bayesian neural network with the same architecture, trained using variational inference. We further include XGBoost, training a separate model for each output. These comparisons should be interpreted in light of the different modelling assumptions and trade-offs associated with each method. Gaussian processes provide a continuous probabilistic model, while methods such as XGBoost do not directly provide predictive distributions. Moreover, the performance of each baseline depends on methodological and hyperparameter choices, such as tree depth and boosting parameters in XGBoost.

\begin{table}[htbp]
\centering
\caption{RMSE and NLPD for MOGP, R-MOGP, and baseline machine learning methods on the \textsc{Energy Efficiency} dataset under different outlier settings.}
\label{tab:comparison_other_ML}
\begin{tabular}{llcc}
\toprule
Outliers & Method & RMSE & NLPD \\
\midrule
None & MOGP    & 0.12 (0.01)  & $-0.92$ (0.08) \\
     & R-MOGP  & 0.12 (0.01)  & $-0.88$ (0.07) \\
     & MLP     & 0.10 (0.004) & -- \\
     & BNN     & 0.21 (0.01)  & 0.18 (0.13) \\
     & XGBoost & 0.083 (0.01) & -- \\
\midrule
Uniform & MOGP    & 2.27 (0.29) & 1.26 (0.06) \\
        & R-MOGP  & 0.17 (0.02) & $-0.25$ (0.05) \\
        & MLP     & 1.14 (0.13) & -- \\
        & BNN     & 0.36 (0.09) & 1.29 (0.25) \\
        & XGBoost & 0.97 (0.04) & -- \\
\midrule
Asymmetric & MOGP    & 2.18 (0.18) & 1.24 (0.06) \\
           & R-MOGP  & 0.16 (0.02) & $-0.31$ (0.07) \\
           & MLP     & 1.11 (0.10) & -- \\
           & BNN     & 0.68 (0.02) & 1.98 (0.12) \\
           & XGBoost & 1.06 (0.07) & -- \\
\midrule
Student-$t$ & MOGP    & 0.19 (0.02) & $-0.03$ (0.10) \\
            & R-MOGP  & 0.14 (0.01) & $-0.59$ (0.04) \\
            & MLP     & 0.23 (0.02) & -- \\
            & BNN     & 0.24 (0.02) & 0.70 (0.11) \\
            & XGBoost & 0.19 (0.03) & -- \\
\bottomrule
\end{tabular}
\end{table}

The results in \Cref{tab:comparison_other_ML} show that, in the absence of outliers, the MLP and XGBoost achieve lower RMSE than both MOGP and R-MOGP. Under all three outlier settings, however, R-MOGP achieves the lowest RMSE among the methods considered and the lowest NLPD among those providing predictive uncertainty estimates.

\subsection{Navitoclax Dose-Response Modelling} \label{app:cancer_research}

We obtain data for the experimental drug Navitoclax from the \href{https://www.cancerrxgene.org/}{Genomics of Drug Sensitivity in Cancer Database} \citep{yang2012genomics}. Specifically, we used raw bulk data from the GDSC2 screening campaign (conducted around 2016–2017), available at \url{https://www.cancerrxgene.org/downloads/bulk_download}. Each record in the dataset contains the cell line identifier, drug name, dose, and corresponding viability measurement. For our analysis, we selected the cell lines P12-ICHIKAWA and ARH-77.

To ensure comparability across experiments, the viability scores were computed by normalising treated wells to plate-specific positive and negative controls using the protocol described in \citet{Vis01052016} and implementation in the R package \texttt{gdscIC50} \citep{gdscIC50}. No batch correction was applied, and viability scores outside the interval $[0,1]$ were retained without truncation. Technical replicates were averaged to yield a single \texttt{(cell line, drug, dose, viability)} observation for each unique combination.

For the MOGP, hyperparameters are selected by maximising the marginal likelihood, whereas for the R-MOGP, we use the w-LOO-CV objective described in \Cref{sec:parameter_optimisation}. In addition, we include the $t$-MOGP and L-MOGP for a comprehensive comparison of robust methods (see Appendix~\ref{app:energy_efficiency} for implementation details). In all models, we set the prior mean to $m_t(\mathbf{x}) = 0.5$ for $t \in \{1,2\}$, reflecting a neutral prior belief halfway between cells being alive (1.0) and dead (0.0). The plots for each method are shown in Figure~\ref{fig:cancer_research_all} and show that the MOGP predictions are clearly distorted by the outlier, while R-MOGP, $t$-MOGP and L-MOGP remain robust.

\begin{figure*}[htbp]
  \centering
  \includegraphics[width=0.9\textwidth]{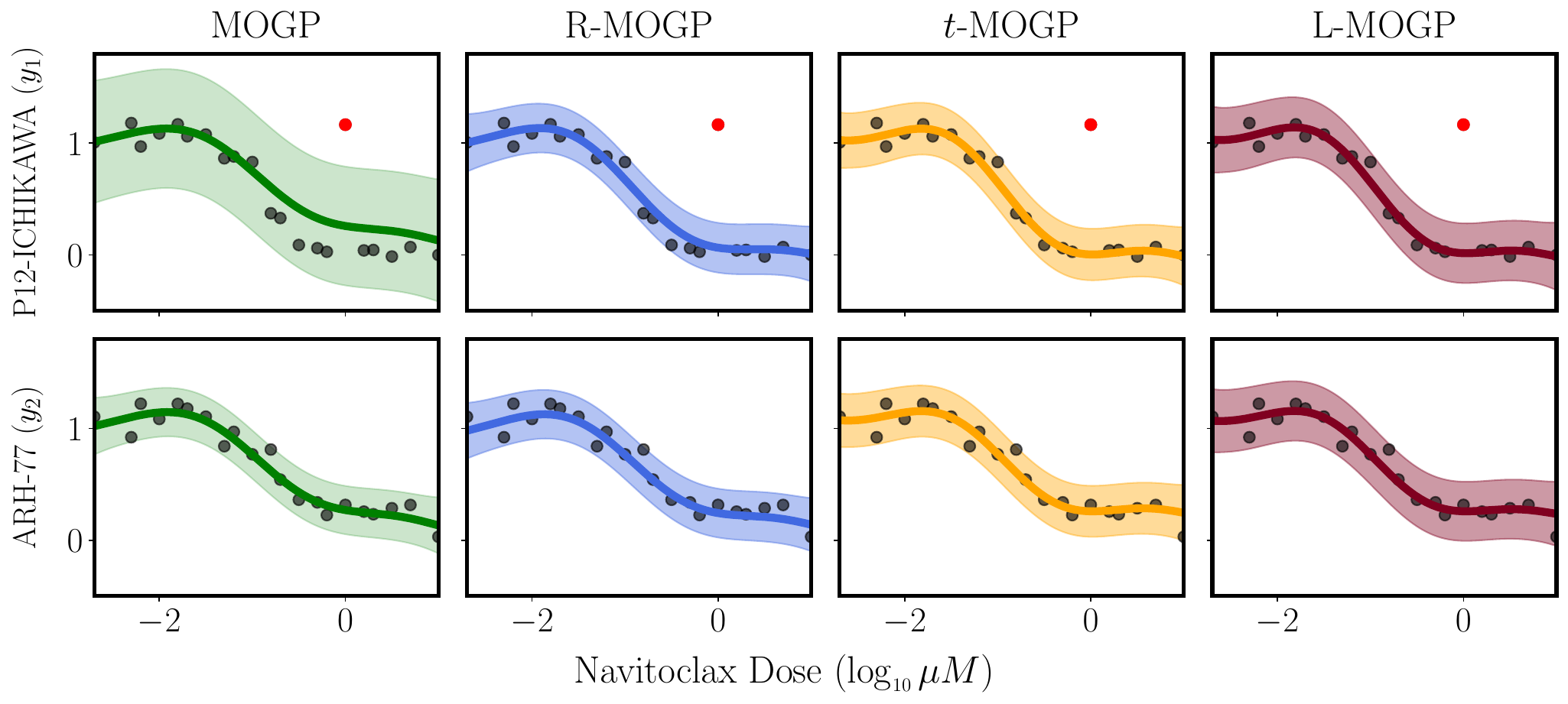}
  \caption{Navitoclax dose-response analysis with outlier in separate plots for MOGP, R-MOGP, $t$-MOGP and L-MOGP.}
  
  \label{fig:cancer_research_all}
\end{figure*} 

In addition, we quantitatively evaluate all methods using the in-sample RMSE and NLPD, reported in \Cref{tab:LOOCV_navitoclax}. The results confirm that MOGP breaks down in the presence of the outlier, whereas R-MOGP, $t$-MOGP, and L-MOGP approach the performance of MOGP without the outlier. Among these robust alternatives, only R-MOGP achieves a runtime of the same order as the conjugate MOGP.

\begin{table}[h]
    \centering
    \caption{Predictive performance and runtime of MOGP variants on the Navitoclax experiment.}
    \label{tab:LOOCV_navitoclax}
    \begin{tabular}{lccc}
        \toprule
        Method & RMSE & NLPD & Runtime (s) \\
        \midrule
        MOGP (no outliers) & 0.066 (0.012) & $-1.23$ (0.11) & 0.048 \\
        MOGP               & 0.170 (0.020) & $-0.21$ (0.05) & 0.049 \\
        R-MOGP             & 0.081 (0.011) & $-1.00$ (0.07) & 0.040 \\
        $t$-MOGP             & 0.067 (0.013) & $-1.06$ (0.06) & 3.820 \\
        L-MOGP             & 0.068 (0.013) & $-0.65$ (0.73) & 1.060 \\
        \bottomrule
    \end{tabular}
\end{table}

\subsection{Electricity Load Forecasting} \label{app:electricity}

We use data from \href{https://data.open-power-system-data.org/time_series/}{Open Power System Data} \citep{wiese2019open}, consisting of country-level electricity load time series from January 2010 to September 2020. The original hourly observations are aggregated to daily frequency, yielding 3,795 data points. Our sample covers $T=5$ Northern European countries: Denmark, Estonia, Finland, Norway, and Sweden.

We standardise both the input and output variables by subtracting their own mean and dividing by their own standard deviation. The dataset is randomly split into a training set (90\%) and a test set (10\%), yielding $N = 3{,}415$ data points. From this training set, we draw random subsamples of sizes [600, 1000, 1400, 1800, 2200, 2600, 3000] data points without replacement. For each subset, we train both SVMOGP and R-SVMOGP models.

To evaluate robustness, we contaminate each subset by replacing a fraction $\epsilon_1 = 0.05$ of the observations---randomly selected from the 40th to 60th percentile---with samples drawn from $\mathcal{N}(9, 0.2)$ for the Danish electricity load ($t=1$).

Both SVMOGP and R-SVMOGP are trained with $R = 1$ latent process and $Q = \sqrt{N}$ inducing points per output. Model parameters are optimised using the L-BFGS algorithm, and are both implemented in \texttt{TensorFlow}. We then report total runtimes, including optimisation and prediction.

\subsection{Robustness to Financial Anomalies} \label{app:finance}

As an additional experiment, we analyse the to-be-announced (TBA) market. 
In this market, buyers and sellers trade mortgage-backed securities (MBS)---a collection of thousands of similar home loans---without specifying the exact underlying pool of mortgages until shortly before a defined settlement date. 
In the TBA market, transaction-level data are often highly correlated between securities as MBS tend to be constructed from similar pools of mortgages and are affected by the same economic factors \citep{boyarchenko2019understanding}.
While the modelling problems in the TBA market are ideally suited to MOGPs \citep{de2020gaussian}, outliers are extremely common, particularly near the settlement date \citep{vickery2013tba}. 
Some key reasons for this are that certain trades are priced much lower than the going rate because sellers can deliver slightly weaker collateral at the same price. 
Conversely, other trades are priced unusually high when specific collateral is scarce and commands a temporary premium. 

We use transaction-level Fannie Mae Conventional Loan-backed MBS trading data, obtained from \href{https://www.finra.org/finra-data/fixed-income/tba/trade}{FINRA}. It consists of $N = 139$ transactions observed on September 8 and 9, 2025, for $T=3$ correlated uniform MBS coupons: \href{https://www.finra.org/finra-data/fixed-income/trade-history?symbol=UMBS3519250&bondType=TBA}{5\% MBS}, \href{https://www.finra.org/finra-data/fixed-income/trade-history?symbol=UMBS3519260&bondType=TBA}{5.5\% MBS}, and \href{https://www.finra.org/finra-data/fixed-income/trade-history?symbol=UMBS3519270&bondType=TBA}{6\% MBS}. We discretise the series at a 1-minute granularity and, when multiple trades occur within the same minute, we retain the last trade observed in that minute. In addition, we standardise both the input and output variables by subtracting their own mean and then dividing by their own standard deviation. We also set $\epsilon_t = 5/N$, expecting 5 contaminated data points per output.
We fit both the MOGP and R-MOGP with a prior mean of $m_t(\mathbf{x}) = \frac{1}{N}\sum_{i=1}^{N} y_{i,t}$ for $t \in \{1,2,3\}$. The MOGP uses marginal likelihood for optimisation, and the R-MOGP uses the w-LOO-CV objective defined in \Cref{sec:parameter_optimisation}. To reflect realistic trading conditions, we fix the noise variance in all models at a quarter tick, following common practice in liquid markets to avoid overestimating noise and to ensure stable, realistic predictions \citep{gonzalvez2019financial}.

\begingroup
\makeatletter
\setlength{\@dblfptop}{0pt}
\setlength{\@dblfpbot}{0pt plus 1fil}
\setlength{\@fptop}{0pt}
\setlength{\@fpbot}{0pt plus 1fil}
\makeatother

\begin{figure*}[t]
  \centering
  \includegraphics[width=\textwidth]{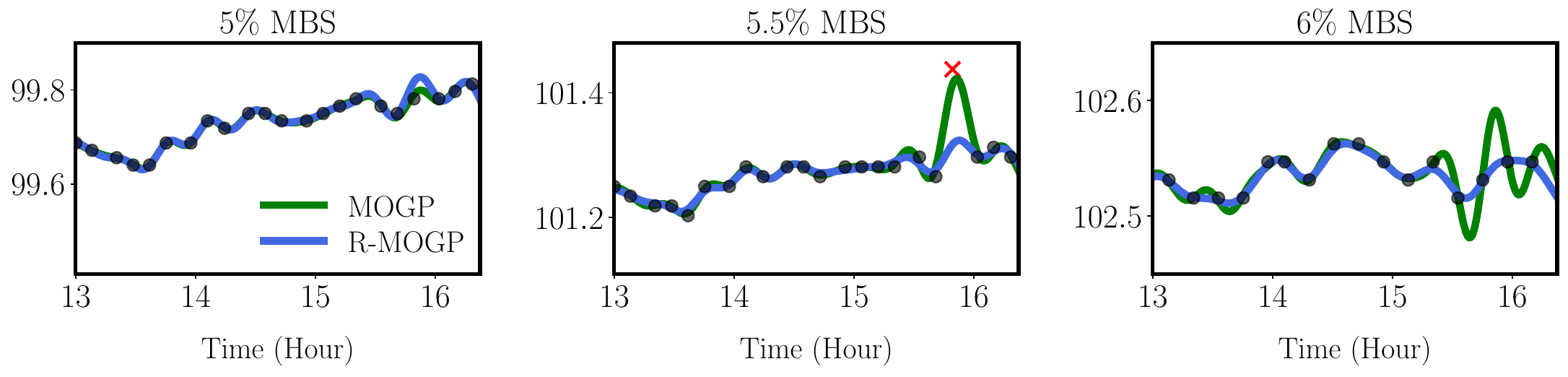}
  \caption{Robustness of MOGP and R-MOGP in MBS trades. The outlier at 4 pm in the 5.5\% MBS distorts the MOGP predictions for that output and propagates errors to the 6\% MBS. In contrast, R-MOGP is robust.}

  \label{fig:FNCL}
  \vspace{-0.15cm}
\end{figure*} 

As can be seen in \Cref{fig:FNCL}, the data shows a clear anomalously high trade being visible in the MBS with 5.5\% on September 8 around 4 pm. This outlier influences not only the predictions for the MBS with 5.5\%, but also for the MBS with 6\%, demonstrating the sensitivity of MOGP to such outliers across different output channels. 
In contrast, R-MOGP remains stable and unaffected in all three output channels.


\clearpage
\endgroup

\end{document}